\pdfoutput=1
\documentclass[11pt]{article}
\usepackage[textwidth=400pt,centering]{geometry}
\usepackage{fullpage}
\usepackage{longtable}

\usepackage{booktabs}

\usepackage[dvipdfmx]{graphicx}
\usepackage[utf8]{inputenc} 
\usepackage[T1]{fontenc}    
\usepackage{amsfonts}       
\usepackage{csquotes}       
\usepackage{nicefrac}       
\usepackage{xcolor}    
\usepackage{times}  
\usepackage{helvet} 
\usepackage{courier}  
\usepackage{caption}
\usepackage{subcaption}
\usepackage[sort]{natbib}
\usepackage[nottoc,notlot,notlof]{tocbibind}

\usepackage{epigraph} 
\usepackage{framed}

\usepackage{pifont}
\usepackage{bbm}
\usepackage{bm}
\usepackage{multirow}
\usepackage{braket}
\usepackage{physics}
\usepackage{enumitem}
\definecolor{darkblue}{rgb}{0,0.08,0.8}
\usepackage{verbatim}

\usepackage{amsmath,amssymb,amsthm}
\usepackage{thmtools,thm-restate}

\theoremstyle{plain}
\newtheorem{theorem}{Theorem}

\newtheorem{lemma}[theorem]{Lemma}

\theoremstyle{definition}

\theoremstyle{remark}

\newcommand\numberthis{\addtocounter{equation}{1}\tag{\theequation}}
\usepackage{calc}  
\usepackage{array} 
\usepackage{algorithm}
\usepackage[algo2e,ruled]{algorithm2e}
\usepackage{algpseudocode}
\usepackage[hidelinks]{hyperref}

\DeclareMathOperator*{\argmax}{arg\,max}
\DeclareMathOperator*{\argmin}{arg\,min}

\newcommand{\I}{\mathbbm{1}}
\newcommand{\dx}{\mathrm{d}}

\newcommand{\KL}{\mathrm{KL}}

\newcommand{\eps}{\epsilon}

\newcommand{\hbmu}{\hat{\bmu}}

\newcommand{\tm}{\tilde{m}}

\newcommand{\bv}{\bm{v}}

\newcommand{\eC}{\mathcal{C}}
\newcommand{\etB}{\tilde{\mathcal{B}}}

\newcommand{\uw}{\underline{w}}
\newcommand{\ug}{\underline{g}}
\newcommand{\ud}{\underline{d}}
\newcommand{\uz}{\underline{z}}
\newcommand{\sig}{\sigma}
\newcommand{\rG}{\mathrm{G}}
\newcommand{\rB}{\mathrm{B}}

\newcommand{\eA}{\mathcal{A}}

\newcommand{\eB}{\mathcal{B}}

\newcommand{\eM}{\mathcal{M}}

\newcommand{\bw}{\bm{w}}

\newcommand{\bmx}{\bm{x}}

\newcommand{\bmu}{\boldsymbol{\mu}}

\newcommand{\hmu}{\hat{\mu}}
\newcommand{\tmu}{\tilde{\mu}}

\newcommand{\cP}{\mathcal{P}}

\title{Thompson Exploration with Best Challenger Rule \\ in Best Arm Identification}

\author{Jongyeong Lee$^{1,2}$ \and Junya Honda$^{3,2}$ \and Masashi Sugiyama$^{2,1}$}
\date{
$^1$ The University of Tokyo 
$^2$ RIKEN AIP
$^3$ Kyoto University
}

\begin{document}
\maketitle

\begin{abstract}
This paper studies the fixed-confidence best arm identification (BAI) problem in the bandit framework in the canonical single-parameter exponential models.
For this problem, many policies have been proposed, but most of them require solving an optimization problem at every round and/or are forced to explore an arm at least a certain number of times except those restricted to the Gaussian model.
To address these limitations, we propose a novel policy that combines Thompson sampling with a computationally efficient approach known as the best challenger rule.
While Thompson sampling was originally considered for maximizing the cumulative reward, we demonstrate that it can be used to naturally explore arms in BAI without forcing it.
We show that our policy is asymptotically optimal for any two-armed bandit problems and achieves near optimality for general $K$-armed bandit problems for $K\geq 3$.
Nevertheless, in numerical experiments, our policy shows competitive performance compared to asymptotically optimal policies in terms of sample complexity while requiring less computation cost. 
In addition, we highlight the advantages of our policy by comparing it to the concept of $\beta$-optimality, a relaxed notion of asymptotic optimality commonly considered in the analysis of a class of policies including the proposed one.
\end{abstract}

\begin{framed}
\begin{displayquote}
    This document serves as a corrigendum to \citet{pmlr-v222-lee24a}, addressing a technical flaw in the original proof of Theorem 2.
    The issue has been corrected without affecting the validity of the main results reported in the published version. We are grateful to Ruo-Chun Tzeng for bringing this matter to our attention.
\end{displayquote}
\end{framed}

\section{Introduction}\label{sec: bai_intro}
As a formulation of reinforcement learning, multi-armed bandit (MAB) problems exemplify a trade-off between exploration and exploitation of knowledge.
In traditional stochastic MAB problems, an agent plays an arm and observes a reward from the unknown but fixed distribution associated with the played arm.
Although a large number of studies on MAB have been designed to maximize the cumulative rewards~\citep{agrawal2012analysis, slivkins2019introduction}, one might be interested only in the quality of a final decision rather than the performance of the overall plays.
For example, one can consider the development of a new drug, where the researchers would aim to identify the most effective treatment from a set of alternatives before testing it on a large group of patients.
When exploration and evaluation phases are separated in this way, it is known that a policy designed to maximize the cumulative rewards performs poorly~\citep{bubeck2011pure}.
Such a setting is called pure exploration and several specialized policies have been proposed for this setting~\citep{bubeck2009pure, gabillon2012best, chen2014combinatorial}.
In this paper, we consider the most standard fundamental formulation of the pure exploration problem, \emph{best arm identification} (BAI), where the agent aims to identify the optimal arm that yields the largest mean reward~\citep{maron1997racing, even2006action}.

Two problem settings, the fixed-budget setting and the fixed-confidence setting, have been mainly considered in the BAI problems.
In the fixed-budget setting, an agent aims to maximize the probability of successfully identifying the optimal arm within a fixed number of trials~\citep{gabillon2012best, komiyama2022minimax}.
On the other hand in the fixed-confidence setting, the agent aims to minimize the number of trials while ensuring that the probability of misidentifying the best arm is less than a fixed threshold~\citep{kalyanakrishnan2012pac, kuroki2020polynomial}. 

In the fixed-confidence setting, \citet{garivier2016optimal} provided a tight lower bound on the expected number of trials, which is also called the sample complexity, for canonical single-parameter exponential family (SPEF) bandit models including the Bernoulli distributions and Gaussian distributions with known variances.
This bound represents the expected number of trials required to achieve a given level of confidence in identifying the best arm.
Along with this lower bound on the sample complexity, they also proposed the Track-and-Stop (TaS) policy that tracks the optimal sampling proportion of arm plays and showed its asymptotic optimality.
However, this policy requires solving a computationally expensive optimization at every round to obtain the optimal sampling proportion.

To address this limitation, several computationally efficient policies have been proposed that solve the optimization problem through a single gradient ascent in the online fashion~\citep{menard2019gradient, wang2021fast}. 
However, most of these policies rely on forced exploration, where an arm is played a certain number of times to ensure that the empirical mean converges to its true value.
While one can naturally specify the number of needed explorations for simple cases such as Bernoulli or Gaussian models, this becomes heavily nontrivial for general models where the variance of rewards may not be bounded.
Recognizing the need for a more natural approach to exploration, \citet{menard2019gradient} emphasized the importance of finding policies that allow for exploration without the need for forced exploration. 
More recently, \citet{barrier2022non} proposed a sampling policy that naturally encourages exploration by employing an upper confidence bound.
However, their algorithm is specifically designed for Gaussian bandits with known variance and exhibits slower convergence of the empirical mean compared to approaches that employ the forced exploration steps.
As a result, their policy requires a larger number of samples in numerical experiments.

The BAI problems have also been considered in the Bayesian setting.
\citet{russo2016simple} proposed top-two sampling rules which are adapted to solve the BAI problem.
Generally in this approach, the leader (e.g., the currently best arm) is played with a fixed probability $\beta$, and the challenger (e.g., an arm selected by some randomized rule) is played with a probability of $1-\beta$, where $\beta$ is a predetermined hyperparameter.
This approach allows for different configurations of the leader and the challenger in each round~\citep{qin2017improving, shang2020fixed}, for which more comprehensive examples can be found in \citet{jourdan2022top}.
A relaxed notion of optimality, $\beta$-optimality, has been commonly considered for top-two sampling rules.
In other words, the sample complexity bounds of these $\beta$-optimal policies do not match the lower bound in general at the cost of their computational efficiency.

\paragraph{Contribution}
In this paper, we present a simple approach that combines a heuristic policy, a variant of the Best Challenger (BC) rule\footnote{The BC rule considered in \citet{garivier2016optimal} and \citet{menard2019gradient} can be seen as a variant of top-two sampling since it also plays either the leader or the challenger at every round. However, the key distinction lies in the deterministic nature of BC, which is solely determined by historical information and does not involve any randomness introduced by a hyperparameter $\beta$.
In this paper, the BC rule refers to a policy without hyperparameter $\beta$, while top-two sampling refers to that with $\beta$.} introduced by \citet{menard2019gradient}, with Thompson sampling (TS), a Bayesian policy originally introduced for cumulative reward minimization.
Although it is known that a policy designed to maximize the cumulative rewards performs poorly when the exploration and evaluation phases are separated~\citep{bubeck2011pure}, we show that TS can still be used for the exploration part to solve the BAI problem.
Our policy addresses the limitations of existing approaches, which often involve solving computationally expensive optimization problems~\citep{garivier2016optimal} and/or require the forced exploration steps~\citep{menard2019gradient,wang2021fast}.
Therefore, our policy allows for a more computationally efficient and practical solution to the BAI problem.

It is important to note that our proposed policy does not achieve asymptotic optimality in all scenarios, similar to the $\beta$-optimal policies.
Nevertheless, we prove that our policy achieves asymptotic optimality for any two-armed bandit problems, which distinguishes it from $\beta$-optimal policies. 
This unique characteristic of our policy offers its own advantages and strengths compared to ($\beta$-)optimal policies.
The contributions of this paper are summarized as follows:
\begin{itemize}
    \item We propose a computationally efficient policy for BAI problems in the SPEF bandits without the need for solving optimization problems, forcing explorations, and using additional hyperparameter $\beta$.
    \item We derive a sample complexity bound of the proposed policy for general $K$-armed SPEF bandits, which achieves the lower bound asymptotically for $K=2$ and is numerically tighter than that of $\beta$-optimal policies for many instances for general $K$.
    \item We experimentally demonstrate the effectiveness of using TS as an exploration mechanism, which serves as a substitute for the forced exploration steps in the BAI problems.
\end{itemize}

\paragraph{Organization}
The rest of this paper is organized as follows.
In Section~\ref{sec: bai_pre}, we formulate the BAI problems for the SPEF bandits and introduce the asymptotic optimality and TS.
Next, in Section~\ref{sec: bai_BCTE}, we propose a simple policy called Best Challenger with Thompson Exploration (BC-TE), which is based on a variant of the best challenger policies described in previous works~\citep{garivier2016optimal, menard2019gradient}.
The sample complexity analysis of BC-TE is presented in Section~\ref{sec: bai_rslt}, where we also compare its result with the asymptotic optimality and $\beta$-optimality.
Furthermore, in Section~\ref{sec: bai_exp}, we provide simulation results that demonstrate the effectiveness of BC-TE, showing competitive performance in terms of the sample complexity and superior computational efficiency compared to other asymptotically ($\beta$-)optimal policies.

\vspace{-0.2em}
\section{Preliminaries}\label{sec: bai_pre}
In this section, we formulate the BAI problem for the model of SPEF and the asymptotic lower bound on the sample complexity.
Then we introduce the stopping rule considered in \citet{garivier2016optimal}.
\vspace{-0.2em}
\subsection{Notation and SPEF bandits}\label{sec: bai_notation}
We consider the $K$-armed bandit model where each arm belongs to a canonical SPEF with a form
\begin{equation}\label{eq: bai_SPEF_form}
    \cP = \left\{ (\nu_{\theta_i})_{i=1}^K : \frac{\dx \nu_{\theta_i} }{\dx \xi}(x) = \exp (\theta_i x - A(\theta_i)), \theta_i \in \Theta, \forall i \in [K] \right\},
\end{equation}
where $\Theta \subset \mathbb{R}$ denotes the parameter space, $\xi$ is some reference measure on $\mathbb{R}$, $A:\Theta \to \mathbb{R}$ is a convex and twice differentiable function, and $[K] := \{1, \ldots, K \}$.
For this model, we can write the expected reward of an arm as $\mu(\theta) = A'(\theta)$ and the KL divergence between two distributions as follows~\citep{cappe2013kullback}:
\vspace{-0.2em}
\begin{equation*}
    \KL(\nu_{\theta_1}, \nu_{\theta_2}) = \mu(\theta_1) (\theta_1 - \theta_2) + A(\theta_2) - A(\theta_1),
\end{equation*}
which induces a divergence function $d$ on $A'(\theta)$ defined by $d(\mu(\theta), \mu(\theta')) =  \KL(\nu_{\theta}, \nu_{\theta'})$.
Following the notation used in \citet{garivier2016optimal}, a bandit instance $\nu = (\nu_{\theta_1}, \ldots, \nu_{\theta_K})$ is identified with the means $\bmu = (\mu_1, \ldots, \mu_K)$.
We denote a set of SPEF bandit models with a unique optimal arm by $\mathcal{S}$.
Therefore, for any $\bmu \in \mathcal{S}$, $\argmax_{i \in [K]} \mu_i$ is a singleton and we assume that $\mu(\theta_1) > \mu(\theta_2) \geq \cdots \geq\mu(\theta_K)$ without loss of generality.
Then, we denote the current maximum likelihood estimate of $\bmu$ at round $t$ by $\hbmu(t) = (\hmu_{1}(t), \ldots, \hmu_{K}(t))$ for $\hmu_i(t) = \frac{1}{N_i(t)}\sum_{s=1}^t x_{i,N_i(s)}$, where $N_i(t)$ denotes the number of rounds the arm $i$ is played until round $t$ and $x_{i,n}$ denotes the $n$-th observation from the arm $i\in [K]$.
By abuse of notation, we sometimes denote $\hmu_i(t)$ by $\hmu_{i, N_i(t)} $ to specify the number of plays of the arm $i$.

In the fixed-confidence setting, a policy is said to be $\delta$ probably approximately correct ($\delta$-PAC) when it satisfies $\mathbb{P}[i(\tau_\delta) \ne 1 \lor \tau_\delta = \infty]\leq \delta$.
Here, $\tau_\delta$ is the number of trials until the sampling procedure stops for a given risk parameter $\delta$, and $i(t)$ denotes the chosen arm at round $t\in\mathbb{N}$.
Thus, the agent aims to build a $\delta$-PAC policy while minimizing the sample complexity $\mathbb{E}_{\bmu}[\tau_\delta]$.

\subsection{Asymptotic lower bound on the sample complexity}
\citet{garivier2016optimal} showed that any $\delta$-PAC policy satisfies for any $\delta \in (0,1)$ and $\bmu\in \mathcal{S}$
\begin{equation}\label{eq: bai_LB}
  \mathbb{E}_{\bmu}[\tau_\delta] \geq T^*(\bmu)\log\left( \frac{1}{2.4 \delta} \right),
\end{equation}
where
\begin{equation}\label{eq: bai_tstar_f}
    T^*(\bmu) := \left(\sup_{\bw \in \Sigma_K} \min_{i\ne 1} f_i(\bw; \bmu) \right)^{-1}.
\end{equation}
Here, the function $f_i$ is defined as
\begin{align*}
    f_i: \Sigma_K \times \mathcal{S} &\rightarrow \mathbb{R}_{+}  \\
    (\bw; \bmu)  &\mapsto w_1 d(\mu_1, \mu_{1,i}^{\bw}) + w_i d(\mu_i, \mu_{1,i}^{\bw}), \numberthis{\label{eq: bai_def_fa}}
\end{align*}
where $\mu_{1,i}^{\bw} = \frac{w_1}{w_1+w_i}\mu_1 +\frac{w_i}{w_1+w_i}\mu_i$ is a weighted mean and $\Sigma_K = \{ \bw \in [0,1]^K : \sum_{i=1}^K w_i = 1 \}$ denotes the probability simplex.
We define $f_i(x; \cdot)=-\infty$ for $x \not\in \Sigma_K$ and $i \in [K]$ for simplicity.
Through the derivation of (\ref{eq: bai_LB}), \citet{garivier2016optimal} also showed that the maximizer $\bw^* = \bw^*(\bmu) := \argmax_{\bw \in \Sigma_K} \min_{i\ne 1} f_i(\bw; \bmu)$ indicates the optimal sampling proportion of arm plays, that is, it is necessary to play arms to bring $\bw^t := \left( \frac{N_1(t)}{t}, \ldots, \frac{N_K(t)}{t} \right) $ closer to $\bw^*$ for matching the lower bound.
The convergence of $\bw^t$ towards $\bw^*$ is widely recognized as a crucial factor for achieving optimal performance in the BAI problem~\citep{menard2019gradient,wang2021fast}.

Along with the lower bound in (\ref{eq: bai_LB}), a policy is said to be asymptotically optimal if it satisfies
\begin{equation*}
    \limsup_{\delta \rightarrow 0} \frac{\mathbb{E}_{\bmu}[\tau_\delta]}{\log(1/\delta)} \leq T^*(\bmu).
\end{equation*}
\citet{garivier2016optimal} proposed the Track-and-Stop (TaS) policy, which tracks the optimal proportions $\bw^*$ at every round, and showed its asymptotic optimality.
Since the true mean reward $\bmu$ is unknown in practice, the TaS policy tracks the plug-in estimates $\bw^*(\hbmu(t))$.
This means that the TaS policy essentially requires solving the minimax optimization problem at every round to find $\bw^*(\hbmu(t))$.
Although some computational burden can be alleviated by using the solution from the previous round as an initial solution, the TaS policy remains computationally expensive due to the presence of the inverse function of the KL divergence.

On the other hand, a relaxed optimality notion, $\beta$-optimality, has been considered in top-two sampling rules, where the leader is played with a predefined probability $\beta \in (0,1)$~\citep{russo2016simple, qin2017improving, shang2020fixed, jourdan2022top}.
Here, a policy is said to be asymptotically $\beta$-optimal if it satisfies
\begin{equation*}
    \lim_{t\to \infty}w_1^t \rightarrow \beta \text{ and } \limsup_{\delta \rightarrow 0} \frac{\mathbb{E}_{\bmu}[\tau_\delta]}{\log(1/\delta)} \leq T^\beta(\bmu),
\end{equation*}
where
\begin{equation}\label{eq: bai_beta_opt}
    T^\beta(\bmu) := \left(\sup_{\bw \in \Sigma_K, w_1=\beta} \min_{i\ne 1} f_i(\bw; \bmu) \right)^{-1}.
\end{equation}
From its definition, $T^*(\bmu) = \min_{\beta \in [0,1]} T^\beta (\bmu)$ holds.
Thus, the $\beta$-optimality does not necessarily imply the optimality in the sense of (\ref{eq: bai_LB}) unless $\beta$ is equal to $w_1^*(\bmu)$.
Still, $\beta=1/2$ is usually employed since $T^*(\bmu) \leq T^{1/2}(\bmu) \leq 2 T^*(\bmu)$ holds, that is, $T^{1/2}(\bmu)$ is at most two times larger than that of optimal policies~\citep[see][Lemma 3]{russo2016simple}.

\subsection{Stopping rule}\label{sec: bai_stopping}
One important question is when an agent should terminate the sampling procedure, which is usually related to a statistical test.
\citet{garivier2016optimal} considered the generalized likelihood ratio statistic that has a closed-form expression for the exponential family.
Based on this statistic, they proposed Chernoff's stopping rule which is written as
\begin{equation}\label{eq: bai_chernoff_stopping}
    \tau_\delta = \inf \left\{ t \in \mathbb{N} : \max_{a\in [K]} \min_{b:\hmu_{a}(t) \geq \hmu_{b}(t)} t f_{a,b}(\bw^t; \hat{\bmu}(t)) > \beta(t, \delta) \right\},
\end{equation}
where $f_{a,b}(\bw;\bmu) := w_a d(\mu_a, \mu_{a,b}^{\bw}) + w_b d(\mu_b, \mu_{a,b}^{\bw})$ for $\mu_a \geq \mu_b$ and $\beta(t, \delta)$ is a threshold to be tuned appropriately.
Therefore, several thresholds $\beta(t, \delta)$ have been proposed~\citep{garivier2016optimal, menard2019gradient, jedra2020optimal, kaufmann2021mixture}.
In this paper, we simply utilize the deviational threshold 
$ \beta(t, \delta) = \log \left( \frac{Ct^\alpha}{\delta} \right)$
for $\alpha>1$ and some constants $C=C(\alpha, K)$ since it was shown that using Chernoff's stopping rule with this threshold ensures the $\delta$-PAC of any policies for the SPEF~\citep[see][Propostion 12]{garivier2016optimal}.

\subsection{Thompson sampling with the Jeffreys prior}
In the regret minimization problem, Thompson sampling has been shown to be asymptotically optimal for various reward models~\citep{kaufmann2012thompson, honda2014optimality, riou2020bandit, Lee2023}.
For the SPEF bandits, TS with the Jeffreys prior was shown to be asymptotically optimal~\citep{KordaTS}.
The Jeffreys prior is a noninformative prior that is invariant under any reparameterization~\citep{robert2009rejoinder}, which is written for the model in (\ref{eq: bai_SPEF_form}) by
\begin{equation*}
    \pi_{\mathrm{j}}(\theta) \propto \sqrt{|I(\theta)|} = \sqrt{|A''(\theta)|},
\end{equation*}
for the Fisher information $I(\theta)$.

Under the Jeffreys prior, the posterior on $\theta$ after $n$ observations is given by
\begin{equation}\label{eq: bai_post}
    \pi(\theta| x_{1}, \ldots, x_n) \propto \sqrt{|A''(\theta)|} \exp\left( \theta \sum_{m=1}^n x_m - n A(\theta)  \right).
\end{equation}
For more details on the Jeffreys prior, we recommend referring to \citet{robert2009rejoinder} and \citet{ghosh2011objective}, as well as the reference therein.
Additionally, one can find more specific configurations on Thompson sampling with the Jeffreys prior for SPEF bandits in \citet{KordaTS}.

\section{Best Challenger with Thompson Exploration}\label{sec: bai_BCTE}
In this section, we aim to build a $\delta$-PAC policy that does not rely on the forced exploration steps.
To achieve this, we utilize TS with the Jeffreys prior as a tool to encourage the exploration of arms in a natural manner.

\subsection{The use of the best challenger rule}
Here, we first introduce the intuition behind the best challenger rule.

For the sake of simplicity, we define a concave objective function $g(\bw; \bmu) := \min_{i\ne1} f_i(\bw; \bmu)$ for $x \in \Sigma_K$ and $g(x; \cdot) =-\infty$ for $x\not\in \Sigma_K$.
Then, (\ref{eq: bai_tstar_f}) can be rewritten as
\begin{equation*}
     \left( T^*(\bmu)\right)^{-1} =\sup_{\bw \in \Sigma_K} g(\bw; \bmu)  = g(\bw^*; \bmu).
\end{equation*}
As discussed in Section~\ref{sec: bai_notation}, one can achieve the asymptotic optimality by moving the empirical proportion $\bw^t$ closer to the optimal proportion $\bw^*$.
Since the optimal proportion $\bw^*$ is a point that maximizes $g$, moving $\bw^t$ in the direction of increasing $g$ is a reasonable idea to reduce the gap between $\bw^t$ and $\bw^*$.
As $\bw^*$ is a solution to a convex optimization problem, a natural approach is to apply a gradient method to iteratively update $\bw^t$, which would bring $\bw^t$ to $\bw^*$ without explicitly solving complex optimization problems.
Although $g$ is not differentiable, it can be expected that playing arms to track a subgradient of $g$ would achieve the lower bound since $g$ is concave.\footnote{In the strict sense, we should use the term subgradient to minimize the convex function $-g$ or supergradient to maximize the concave function $g$. However, we use the term subgradient for $g$ since the term subgradient is more popular, and the use of $-g$ needlessly degrades the readability.}

Here, we say that $\bv$ is a subgradient of the concave function $g$ at the point $(\bw; \bmu)$ if
\begin{equation*}
    \forall \bw^\prime \in \Sigma_K, \ g(\bw^\prime; \bmu) \leq g(\bw; \bmu) + \bv^\top(\bw^\prime-\bw).
\end{equation*}
The subdifferential $\partial g(\bw; \bmu)$ is the set of all such subgradients.
The following lemma shows that the subgradients of the objective function $g$ are expressed as the sum of all-ones vector $\mathbf{1}$ and convex combinations of the gradients $\nabla_{\bw} f (\bw; \bmu)$ of $f$ with respect to $\bw$.
The proofs of all lemmas and theorems are given in the supplementary material.
\begin{restatable}{restatelemma}{subgfind}\label{lem: subgfind}
The subdifferential $\partial g$ of $g$ with respect to $\bw \in \mathrm{Int}\,\Sigma_K$ for given $\bmu \in \mathcal{S}$ is such that
\begin{equation*}
   \partial g(\bw; \bmu) = \Bigg\{ \sum_{i\in \mathcal{J}(\bw;\bmu)} \lambda_i \nabla_{\bw} f_i (\bw; \bmu)+r \mathbf{1} : \sum_{i\in \mathcal{J}(\bw;\bmu)} \lambda_i =1,  \lambda_i \geq 0, r \in \mathbb{R}  \Bigg\}, 
\end{equation*}
where $\mathcal{J}(\bw;\bmu) := \argmin_{i\ne 1}f_i(\bw; \bmu)$ denotes the set of challengers, $f_i$ is defined in (\ref{eq: bai_def_fa}), and $\mathrm{Int}\,\Sigma_K$ denotes the interior of the probability simplex.
\end{restatable} 
By letting $r=0$ and $\lambda_i = 1/|\mathcal{J}(\bw;\bmu)|$ for any $i\in [K]$ in Lemma~\ref{lem: subgfind}, we can obtain a subgradient $\bv$ for $\bmu \in \mathcal{S}$ satisfying
\begin{equation*}
    v_i(\bw; \bmu) = \begin{cases}
        0 & \text{ if } i \notin \{1\} \cup \mathcal{J}(\bw;\bmu), \\
        \frac{1}{| \mathcal{J}(\bw;\bmu) |} \sum_{j \in \mathcal{J}(\bw;\bmu)} d(\mu_{i}, \mu_{i, j}^{\bw} )  & \text{ if } i =1, \\
        \frac{1}{| \mathcal{J}(\bw;\bmu) |} d(\mu_{i}, \mu_{1, i}^{\bw} )  & \text{ if } i \in \mathcal{J}(\bw;\bmu).
    \end{cases}
\end{equation*}
Since our objective is to maximize the objective function $g$, one can easily consider a greedy approach that plays an arm with the maximum subgradient, that is
\begin{equation*}
    i(t) \in \argmax_{i\in [K]} v_i(\bw^t; \hbmu(t)),
\end{equation*}
which plays either the currently best arm $m(t) = \argmax_{i \in[K]} \hmu_i(t) $ or the challenger $j(t) \in \mathcal{J}_t= \mathcal{J}(\bw^t; \hbmu(t))$ at round $t$.
For the arbitrarily chosen challenger
\begin{equation}\label{eq: j(t)}
    j(t) = \argmin_{i \ne m(t)} f_i(\bw^t; \hbmu(t)),
\end{equation}
a variant of the Best Challenger (BC) rule introduced by \citet{menard2019gradient} can be expressed as
\begin{align*}
    i(t)  = \begin{cases}
        m(t)  &\text{if } d(\hmu_{m(t)}(t), \hmu_{m(t),j(t)}(t)) \geq d(\hmu_{j(t)}(t), \hmu_{m(t),j(t)}(t)), \\
        j(t) & \text{otherwise},
    \end{cases}
\end{align*}
where we denote $\hmu_{a,b}^{\bw^t}(t) = \frac{w_a^t}{w_a^t+w_b^t} \hmu_a(t) +  \frac{w_a^t}{w_a^t+w_b^t} \hmu_b(t)$ by $\hmu_{a,b}(t)$ for notational simplicity.
This simple heuristic with forced exploration was shown to be computationally very efficient and showed excellent empirical performance in the BAI problems despite its lack of theoretical guarantee.

Note that the use of subgradients instead of solving the optimization problem at every round has been considered by \citet{menard2019gradient}, where they applied the online mirror ascent method, and by \citet{wang2021fast}, where they applied the Frank-Wolfe-type algorithm to optimize the non-smooth concave objective function $g$.
It is worth noting that both policies are shown to be asymptotically optimal for various BAI problems.
Nevertheless, the families of top-two samplings (including BC rules) are especially simple, and for this reason, $\beta$-optimality is still considered despite its suboptimality~\citep{jourdan2022top, pmlr-v201-jourdan23a, mukherjee2022sprt}.


\begin{algorithm2e}[t]
   \DontPrintSemicolon
   \SetKwInOut{Initialization}{Initialization}
   \Initialization{Play every arm twice and set $\bw^{2K}=\frac{1}{K}$ and $t=2K$.}
   \While{stopping criterion is satisfied}{
    Sample $\tmu_i(t)$ from the posterior distribution in (\ref{eq: bai_post}).\;
       Set $m(t) = \argmax_{i\in [K]}\hmu_i(t)$ and $\tm(t)  = \argmax_{i\in [K]}\tmu_i(t)$.\;
      \eIf{$m(t) = \tm(t)$}{
      Find the subgradient $\bv^t$ of $g(\bw^t, \hbmu^t)$.\;
      Play $i(t+1) \in \argmax_{i \in [K]} v_i^t$ and observe the reward.\;
      }{
      Play $i(t+1) \in \argmin_{i \in \{m(t), \tm(t)\}}N_i(t)$.\;
      Update $t=t+1$, $\hbmu^t$ and $\bw^t$.
      }
   }
  \caption{Best challenger with Thompson Exploration (BC-TE)}
   \label{alg: BCTE}
\end{algorithm2e}

\subsection{The use of Thompson exploration}
Although the policies using gradient methods are asymptotically optimal and/or simple, they still include the forced exploration steps to ensure that the empirical means converge to their true values.
Therefore, it is worth finding a natural way to explore without forcing policies to explore.
Although \citet{barrier2022non} replaced the forced exploration steps by using the upper confidence bound-based approach, their policy was restricted to the Gaussian models and exhibited large sample complexity in numerical experiments.
Instead, in this paper, we employ TS as an exploration tool to eliminate the forced exploration steps, which can be applied to any SPEF bandits and performs well in practice.
To be precise, we play an arm according to the BC rule only when the empirical best arm and the best arm under the posterior sample agree, that is,
\begin{equation*}
    i(t) = \begin{cases}
        \argmax_{i \in [K]} v_i(\bw^t; \hbmu(t)) & \text{ if } m(t) = \tm(t) := \argmax_{i \in [K]} \tmu_i(t), \hspace{4em} \text{(BC)}\\
        \argmin_{i \in \{ m(t), \tm(t) \}} N_i(t)  &\text{ otherwise}, \hfill \text{(Thompson exploration)} 
    \end{cases}
\end{equation*}
where $\tmu_i(t)$ denotes the posterior sample of the arm $i$ generated by the posterior in (\ref{eq: bai_post}).
As the number of plays increases, the probability of observing a sample that deviates significantly from the current empirical mean decreases exponentially.
In other words, if an arm is played only a few times, its posterior sample is more likely to deviate from its empirical mean. 
This discrepancy between the best arm under the posterior sample and the empirical best arm can be a guide to the policy for further exploration.
By selecting an arm with a small number of plays only when the empirical best arm and the best arm under the posterior sample disagree, we can ensure the convergence of the empirical means to their true values without relying on forced exploration, which is formulated in Section~\ref{sec: bai_rslt}.
The proposed algorithm, called Best Challenger with Thompson Exploration (BC-TE), is described in Algorithm~\ref{alg: BCTE}.
Notice that BC-TE plays every arm twice at initialization steps to avoid an improper posterior distribution.

\section{Main Theoretical Results}\label{sec: bai_rslt}
In this section, we show the effectiveness of TE and prove that BC-TE is nearly optimal, similar to $\beta$-optimality.

\subsection{Main theorems}
Firstly, let us define a random variable $T_B \in \mathbb{N}$ such that for any $\eps < \frac{\mu_1 - \mu_2}{2}$
\begin{equation}\label{eq: bai_def_TB}
    T_B = \inf\{ T \in \mathbb{N}: \forall s \geq T, \forall i \in [K], |\hmu_i(s) - \mu_i| \leq \eps \}.
\end{equation}
Therefore, the empirical mean estimate $\hbmu(t)$ is sufficiently close to its true value $\bmu$ for all rounds after $T_B$.
The theorem below shows the expected value of $T_B$ is finite.
\begin{restatable}{restatethm}{baiTB}\label{thm: bai_TB}
Under Algorithm~\ref{alg: BCTE}, it holds that
\begin{equation*}
    \mathbb{E}[T_B] \leq \mathcal{O}(K^2 d_\eps^{-2}),
\end{equation*}
where
\begin{equation}\label{def: deps}
    d_\eps := \min_{i \in [K]} \min(d(\mu_i + \eps, \mu_i), d( \mu_i - \eps, \mu_i)).
\end{equation}
\end{restatable}
From the definition of $T_B$, one can expect that the sampling rule will behave as expected after $T_B$ rounds since the estimated means are close to the true ones.
Note that $T_B$ is not a stopping time with respect to the sequence of observations and we need a careful analysis for its expectation. 
The key property used in the proof is that BC-TE always plays an arm that increases the objective function $g(\bw^t; \hbmu(t))$ at every round $t$.
Since most arguments in the proof of Theorem 2 do not depend on the procedure when TE does not occur, we can expect that one can derive the same result for Theorem 2 for any policy designed to increase the objective function at every round such as Frank-Wolfe sampling~\citep{wang2021fast}.
Then, the sample complexity of BC-TE can be upper bounded as follows.
\begin{restatable}{restatethm}{baisample}\label{thm: bai_sample}
Let $\alpha \in [1, e/2]$ and $r(t) = \mathcal{O}(t^{\alpha})$.
Using the Chernoff's stopping rule in (\ref{eq: bai_chernoff_stopping}) with $\beta(t,\delta) = \log\left( r(t)/\delta \right)$ under Algorithm~\ref{alg: BCTE},
\begin{equation*}
    \limsup_{\delta \to 0} \frac{\mathbb{E}[\tau_\delta]}{\log(1/\delta)} \leq \alpha \underline{T}(\bmu),
\end{equation*}
where 
\begin{equation}\label{eq: bai_def_our_optimality}
    \underline{T}(\bmu) := \left(\sup_{\bw \in \Sigma_K,\frac{w_{2}}{w_{1}+ w_{2}} = \gamma}\min_{i\ne 1} f_{i}(\bw; \bmu) \right)^{-1}
\end{equation}
for $\gamma$ satisfying
\begin{equation}\label{eq: bai_def_gamma_star}
    d(\mu_1, (1-\gamma)\mu_1 + \gamma \mu_{2} ) = d(\mu_{2}, (1-\gamma)\mu_1 + \gamma \mu_{2} ).
\end{equation}
\end{restatable}
From the definition of $T^*(\bmu)$ in (\ref{eq: bai_tstar_f}), one can see the suboptimality of BC-TE from $\underline{T}(\bmu) \geq T^*(\bmu)$, which indicates that BC-TE may be not always optimal, as it only achieves optimality when the condition $\gamma = \frac{w_2^*}{w_1^* + w_2^*}$ is true.
This observation is akin to the result for $\beta$-optimality.

\subsection{Comparison with $\beta$-optimality and asymptotic optimality}\label{sec: bai_comp_asym}
Recall that the quantity $T^{\beta}(\bmu)$ in (\ref{eq: bai_beta_opt}) demonstrates that $\beta$-optimality is achieved when the allocation of the optimal arm is $\beta$.
On the other hand, $\underline{T}(\bmu)$ considers the scenario where $\frac{w_2}{w_1+w_2} = \gamma$, which is the best ratio between the best arm and the second best arm to distinguish them.
Both notions are more relaxed compared to asymptotic optimality, and it is not possible to determine definitively which one is better in general.

However, it is important to note that our policy does not require prior knowledge of $\gamma$, differently from existing $\beta$-optimal policies that take $\beta$ as an input to the algorithm~\citep{russo2016simple, shang2020fixed, jourdan2022top, jourdan2022non}.
Therefore, if there is no prior knowledge of $\beta$, using BC-TE would have its own advantages over $\beta$-optimal policies.
In general, it is challenging to compare the quantities $\underline{T}$ and $T^{\beta}$ for $\beta=1/2$ analytically due to the complex formulation of KL divergence and the optimization problem in (\ref{eq: bai_beta_opt}) and (\ref{eq: bai_def_our_optimality}). 
For this reason, in Section~\ref{sec: SPEF}, we provide numerical comparisons for $K\geq 2$ across various SPEF bandits.

Then, the natural question is the relationship between $T^*(\bmu)$ and $\underline{T}(\bmu)$.
Unlike the $\beta$-optimality where $\beta$ does not depend on the bandit instance, the quantity $\underline{T}(\bmu)$ is problem-dependent since $\gamma$ is determined by $\mu_1$, $\mu_2$, and $d(\cdot,\cdot)$.
Here, we provide a rough comparison with the quantity $T^*(\bmu)$.

\subsubsection{Two-armed bandits}\label{sec: two-arm}
When $K=2$, (\ref{eq: bai_tstar_f}) can be written as
\begin{equation*}
    (T^*(\bmu))^{-1} =  \sup_{\alpha \in (0,1)} \alpha d(\mu_1, \mu^{\alpha}) + (1-\alpha) d(\mu_2, \mu^{\alpha}),
\end{equation*}
where $\mu^\alpha = (1-\alpha) \mu_1 + \alpha \mu_2$.
Here, \citet{garivier2016optimal} showed that the maximum is reached at $\alpha^*$ satisfying $d(\mu_1, \mu^{\alpha^*}) = d(\mu_2, \mu^{\alpha^*})$.
From (\ref{eq: bai_def_gamma_star}), one can directly see that $\gamma = \alpha^*$ holds, which implies $\underline{T} = T^*(\bmu)$ for any $\bmu \in \mathcal{S}$ if $K=2$.
A more detailed discussion is given in the supplementary material for the sake of completeness.

\subsubsection{Gaussian bandits}\label{sec: gauss}
When $\bmu$ belongs to the Gaussian distributions with known variance $\sig^2 >0$, the KL divergence takes a simple form of $d(\mu, \mu') = \frac{(\mu-\mu')^2}{2\sig^2}$.
This allows us to derive a more explicit comparison with asymptotic optimality.
\begin{lemma}\label{lem: bai_Gaussian}
    Let $\Delta_i = \mu_1 - \mu_i$ for $i\ne 1$ and $\Delta_1 = \Delta_{2}$.
    When $\bmu$ belongs to the Gaussian distributions with known variance $\sig^2>0$,
    \begin{equation*}
    \underline{T}(\bmu) = \sum_{i=1}^K \frac{4\sig^2}{\Delta_i^2 + (\Delta_i^2- \Delta_2^2)}.
    \end{equation*}
\end{lemma}
Here, \citet{garivier2016optimal} showed the following inequalities for the Gaussian bandits
\begin{equation*}
    \sum_{i=1}^K \frac{2\sig^2}{\Delta_i^2} \leq T^*(\bmu) \leq 2  \sum_{i=1}^K \frac{2\sig^2}{\Delta_i^2},
\end{equation*}
which directly implies that
\begin{equation}\label{eq: gauss_T}
   T^*(\bmu) \leq \underline{T}(\bmu) \leq 2 T^*(\bmu),
\end{equation}
where the left equality holds when $w^*_1(\bmu) = w^*_2(\bmu)$ and the right equality holds only when $\mu_2 = \cdots = \mu_K$.
Notice that the same result as (\ref{eq: gauss_T}) holds for $T^{\beta}$ with $\beta = 1/2$~\citep{russo2016simple}, though $T^{1/2}(\bmu) \neq \underline{T}(\bmu)$ holds in general.

\subsubsection{Numerical comparison for various SPEF bandits}\label{sec: SPEF}
Here, we compare the quantities $\underline{T}(\bmu)$, $T^*(\bmu)$, and $T^{\beta}(\bmu)$ with $\beta = 1/2$ across different bandit models and varying numbers of arms.
Specifically, we consider two instances $\bmu^{(1)}$ and $\bmu^{(2)}$ for Gaussian (with unit variance), Bernoulli, Poisson, and Exponential distributions.

We consider two instances, $\bmu^{(1)} = (0.3, 0.21, 0.21-0.001, \ldots, 0.21-0.001(K-2))$ and $\bmu^{(2)} = (0.9, 0.7, 0.7-0.001, \ldots, 0.7-0.001(K-2))$.
For example, when $K=4$, $\bmu^{(1)}=(0.3, 0.21, 0.209, 0.208)$ and $\bmu^{(2)} = (0.9, 0.7, 0.699, 0.698)$ are considered.
In Figure~\ref{fig: bai_T_ratio}, the solid line represents the ratio $\underline{T}(\bmu)/T^*(\bmu)$, while the dashed line represents the ratio $T^{1/2}(\bmu)/T^*(\bmu)$. 
Each line corresponds to a different reward model, which is distinguished by a different color and marker.
From Figure~\ref{fig: bai_T_ratio}, we can observe that $\underline{T}(\bmu)$ keeps being close to $T^*$, while $T^{1/2}(\bmu)$ does not for large $K$.
This contrasting behavior indicates the advantage of BC-TE over $\beta$-optimal policies, particularly for large $K$, as it suggests that BC-TE enjoys a much tighter upper bound on its sample complexity.
Additional comparisons are provided in the supplementary material.

\begin{figure}
\centering
     \begin{subfigure}[b]{0.48\textwidth}
     \centering
     \label{fig: bai_one}
     \includegraphics[width=\textwidth]{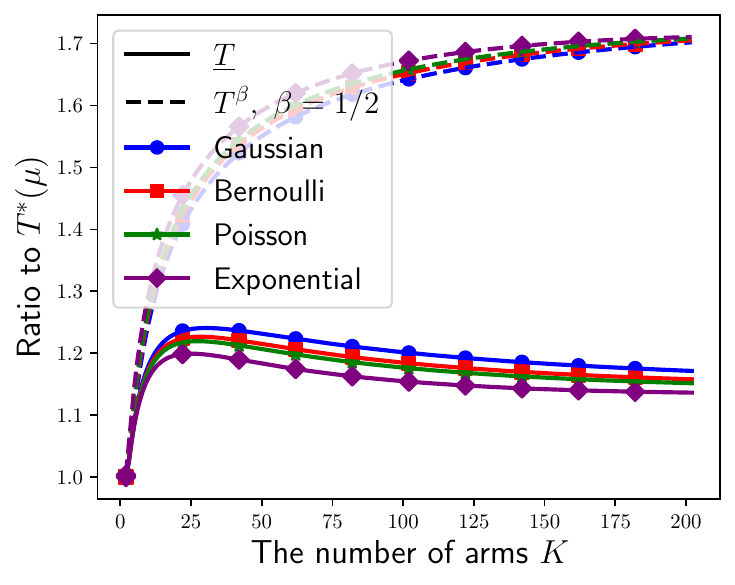}
     \caption{Instance $\bmu^{(1)}$ with varying $K$.}
     \end{subfigure}
     \hfil
     \begin{subfigure}[b]{0.48\textwidth}
     \centering
     \label{fig: bai_two}
         \includegraphics[width=\textwidth]{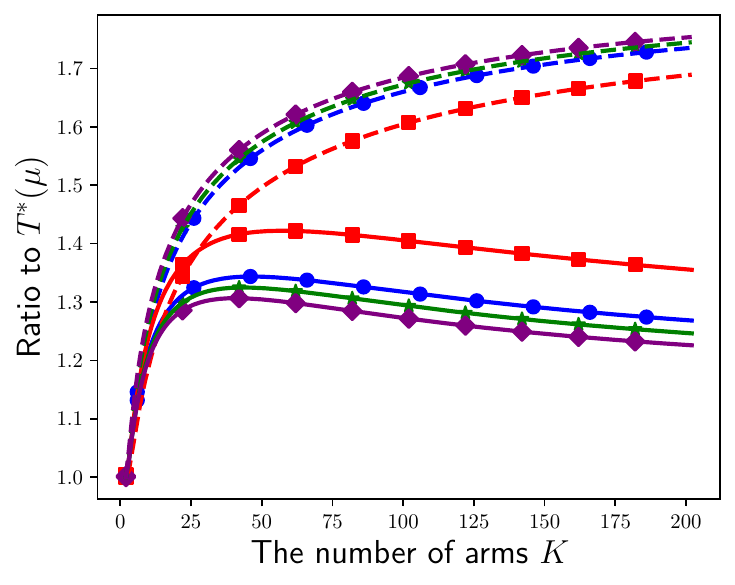}
        \caption{Instance $\bmu^{(2)}$ with varying $K$.}
     \end{subfigure}
\caption{The ratio of $\underline{T}(\bmu)$ and $T^{1/2}(\bmu)$ to $T^*(\bmu)$ for different reward distributions.}
\label{fig: bai_T_ratio}
\end{figure}

\section{Simulation Results}\label{sec: bai_exp}
In this section, we present numerical results to demonstrate the performance of BC-TE.

\paragraph{Compared policies}
We compare the performance of BC-TE with other policies, where $\diamond$ denotes that the policy requires forced exploration.
For policies with $\dagger$ and $\ddag$, we used the implementation by \citet{Wouter} and by \citet{wang2021fast}, respectively.
\begin{itemize}
    \item Track-and-Stop$^{\dagger, \diamond}$ (TaS): an asymptotically optimal policy that solves the optimization problem in (\ref{eq: bai_tstar_f}) at every round, which is computationally costly~\citep{garivier2016optimal}.
    Here, we focus on the TaS policy with D-tracking (T-D) in our experiment.
    \item Lazy Mirror Ascent$^{\dagger, \diamond}$ (LMA): a computationally efficient and asymptotically optimal policy that performs a single gradient ascent in an online fashion~\citep{menard2019gradient}.
    \item AdaHedge vs Best Response$^\dagger$ (AHBR): an asymptotically optimal policy that solves the optimization problem as an unknown game~\citep{degenne2019non}.
    \item Optimistic TaS$^{\ddag}$ (O-C): The optimistic TaS policies with C-tracking proposed by \citet{degenne2019non}, which is known to be very computationally expensive.
    \item Frank-Wolfe Sampling$^{\ddag, \diamond}$ (FWS): an asymptotically optimal policy that just relies on a single iteration FW algorithm instead of solving the optimization problems in (\ref{eq: bai_tstar_f}) at every round~\citep{wang2021fast}.
    \item Round Robin (RR): a simple baseline that samples arms in a round-robin manner.
    \item Top-Two Transportation Cost (T3C): a computationally efficient asymptotically $\beta$-optimal top-two policy based on TS~\citep{shang2020fixed}. Notice that its $\beta$-optimality was extended to bounded distributions by \citet{jourdan2022top} and we set $\beta = 1/2$.
\end{itemize}
In addition, we implement a modified version of FWS, called FWS-TE, where we replace the forced exploration step in FWS with our Thompson exploration step.
This adaptation is based on the discussion below Theorem~\ref{thm: bai_TB} that TE can be used for policies designed to increase the objective function $g$ at every round.

\paragraph{Stopping rule}
Following the experiments in the previous researches~\citep{garivier2016optimal, degenne2019non,menard2019gradient, wang2021fast}, we considered the same threshold $\beta(t,\delta) = \log((\log(t)+1)/\delta)$.

\paragraph{General setup}
Here, we provide the empirical sample complexities of various policies for a range of risk levels $\delta \in \{0.2, 0.1, 0.01, 0.001 \}$ averaged over 3,000 independent runs.
Following \citet{degenne2019non}, we consider the practical version of the lower bound (PLB), which refers to the first round where $t g(\bw^*;\bmu) \geq \beta(t,\delta)$ is satisfied.
Hence, this practical lower bound indicates the earliest round where the generalized likelihood ratio statistic approximately crosses the threshold, and is defined as round $s$ where $s= \beta(s,\delta) T^*(\bmu)$ holds.
Recall that the lower bound (LB) is given as $T^*(\bmu)\log\left(\frac{1}{2.4\delta}\right)$ according to (\ref{eq: bai_LB}).

\paragraph{Bernoulli bandits}
In the first experiment, we consider the 5-armed Bernoulli bandit instance $\bmu^{\mathrm{B}}_5=(0.3, 0.21, 0.2, 0.19, 0.18)$ where $\bw^*(\bmu^{\mathrm{B}}_5) =(0.43, 0.25, 0.18, 0.13$, $ 0.10)$.
This instance was considered in previous researches~\citep{garivier2016optimal, menard2019gradient, wang2021fast}.

\paragraph{Gaussian bandits}
In the second experiment, we consider the 4-armed Gaussian bandit instance $\bmu^{\rG}_4 = (1.0, 0.85, 0.8, 0.7)$ with unit variance $\sig^2=1$ where $\bw^*(\bmu^{\rG}_4) = (0.41, 0.38, 0.15, 0.06)$.
This instance was studied in \citet{wang2021fast}.

\begin{table}[t]
    \caption{Sample complexity over 3,000 independent runs, where outperforming policies are highlighted in boldface using one-sided Welch's t-test with the significance level 0.05.
    LB denotes the lower bound in (\ref{eq: bai_LB}), and PLB denotes the practical version of LB considered in \citet{degenne2019non}.
    $\bmu^{\mathrm{B}}_5$ denotes 5-armed Bernoulli bandit instance with means $(0.3, 0.21, 0.2, 0.19, 0.18)$ and $\bmu^{\rG}_4$ denotes 4-armed Gaussian bandit instance with means $(1.0, 0.85, 0.8, 0.7)$ and unit variance.}
    \label{tab: bai_all}
    \centering
    \resizebox{\columnwidth}{!}{\begin{tabular}{ll|cc|ccccccc|cc}
        $\bmu$ &$\delta$ & BC-TE & FWS-TE & FWS & LMA & T-D & O-C & AHBR & T3C & RR & PLB & LB \\
        \hline
       \multirow{4}{*}{$\bmu_5^{\rB}$} & 0.2 &  \textbf{1065} & 1077&  1176 & 1415& 1107 & 1545 & 1615& 1115& 1977 & 1208& 272 \\
        &0.1 &  \textbf{1288} & 1326 & 1373& 1668& 1337& 1818& 1859& 1372& 2326 & 1442 &  574 \\
        & 0.01 &  \textbf{2064} & 2102 & 2125& 2509&  \textbf{2066} & 2706& 2675& 2180 & 3460 & 2211 & 1471  \\
        & 0.001&  \textbf{2849} & 2870& 2880& 3362&  \textbf{2823} & 3584& 3469& 3011& 4555 &  2974 & 2252 \\
         \hline
       \multirow{4}{*}{$\bmu_4^{\rG}$} & 0.2 & \textbf{1415} & \textbf{1435} & 1499 & 1799 & 1472 & 1837 & 1959 &  1482 & 2555 & 1683 & 374 \\
        & 0.1 & \textbf{1759} & \textbf{1772} & 1829 & 2153 & \textbf{1806} & 2235 & 2339 & 1833 & 3078 &2004 & 791 \\
        & 0.01 & \textbf{2895} & \textbf{2887} & \textbf{2890} & 3300 & \textbf{2835} & 3501 & 3524 & 2947 & 4730 & 3062 & 2026  \\
        & 0.001& 3987 & \textbf{3967} & \textbf{3922} & 4445 & \textbf{3908} & 4732 & 4657 & 4042 & 6349 & 4112 & 3101
    \end{tabular}}
\end{table}
\paragraph{Results}
The overall results are presented in Table~\ref{tab: bai_all}.
Although our proposed policy BC-TE does not achieve the asymptotic optimality in general, it exhibits a better empirical performance than other optimal policies across most risk parameters, especially when large $\delta$ is considered.
Interestingly, Figure~\ref{fig: bai_all} shows that both BC-TE and FWS-TE consistently outperform other optimal policies especially when large $\delta$ is considered, demonstrating the practical effectiveness of TE as an alternative to the forced exploration steps.
Furthermore, we observe that BC-TE is more computationally efficient than other asymptotically optimal policies, and FWS-TE outperforms the original FWS in terms of efficiency, as demonstrated in Table~\ref{tab: bai_T}.

\begin{figure}
     \centering
     \begin{subfigure}[b]{0.48\textwidth}
        \centering
        \label{fig: bai_Ber}
        \includegraphics[width=\textwidth]{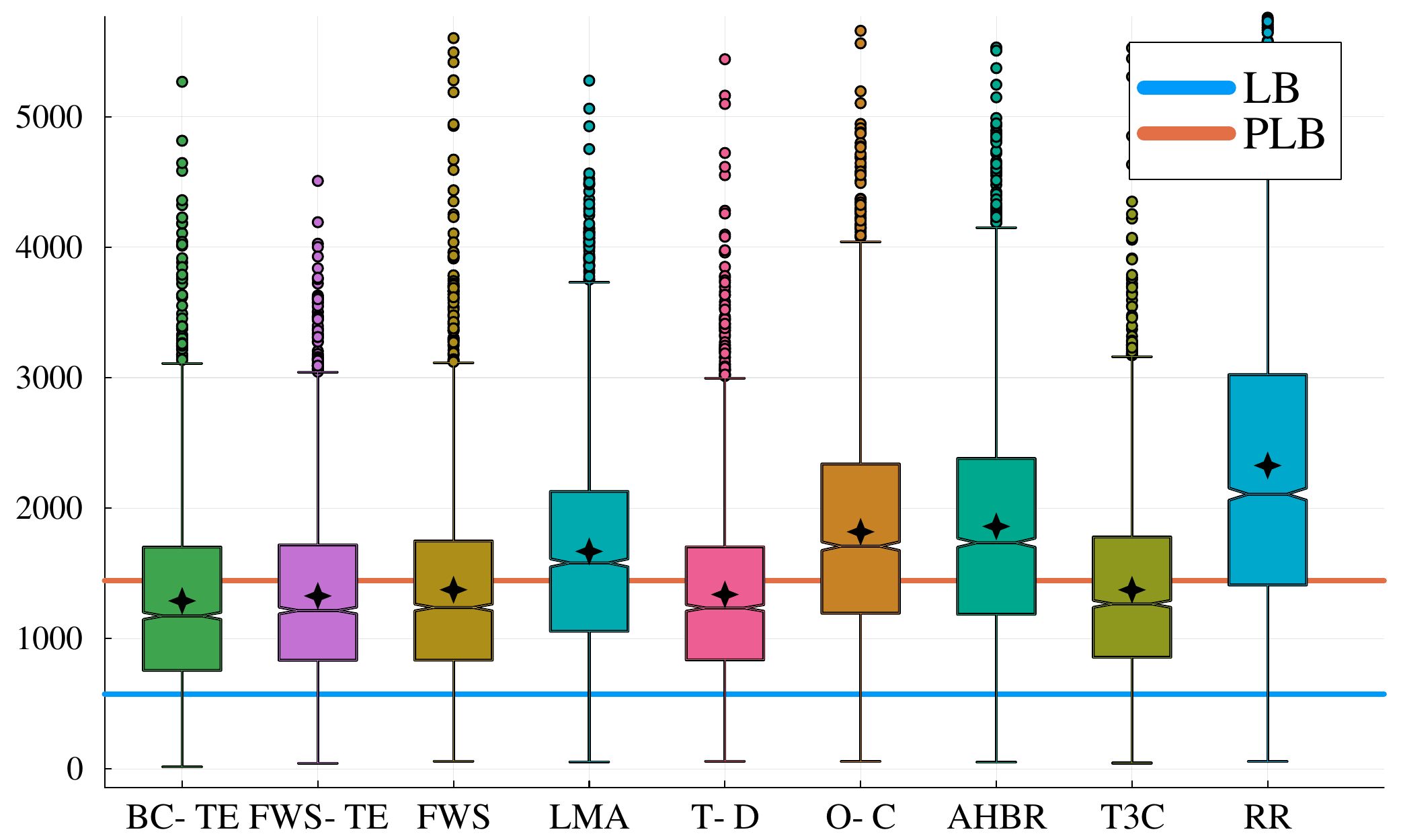}
        \caption{Bernoulli instance $\bmu_5^{\rB}$}
     \end{subfigure}
     \hfil
     \begin{subfigure}[b]{0.48\textwidth}
     \centering
     \label{fig: bai_gauss}
         \includegraphics[width=\textwidth]{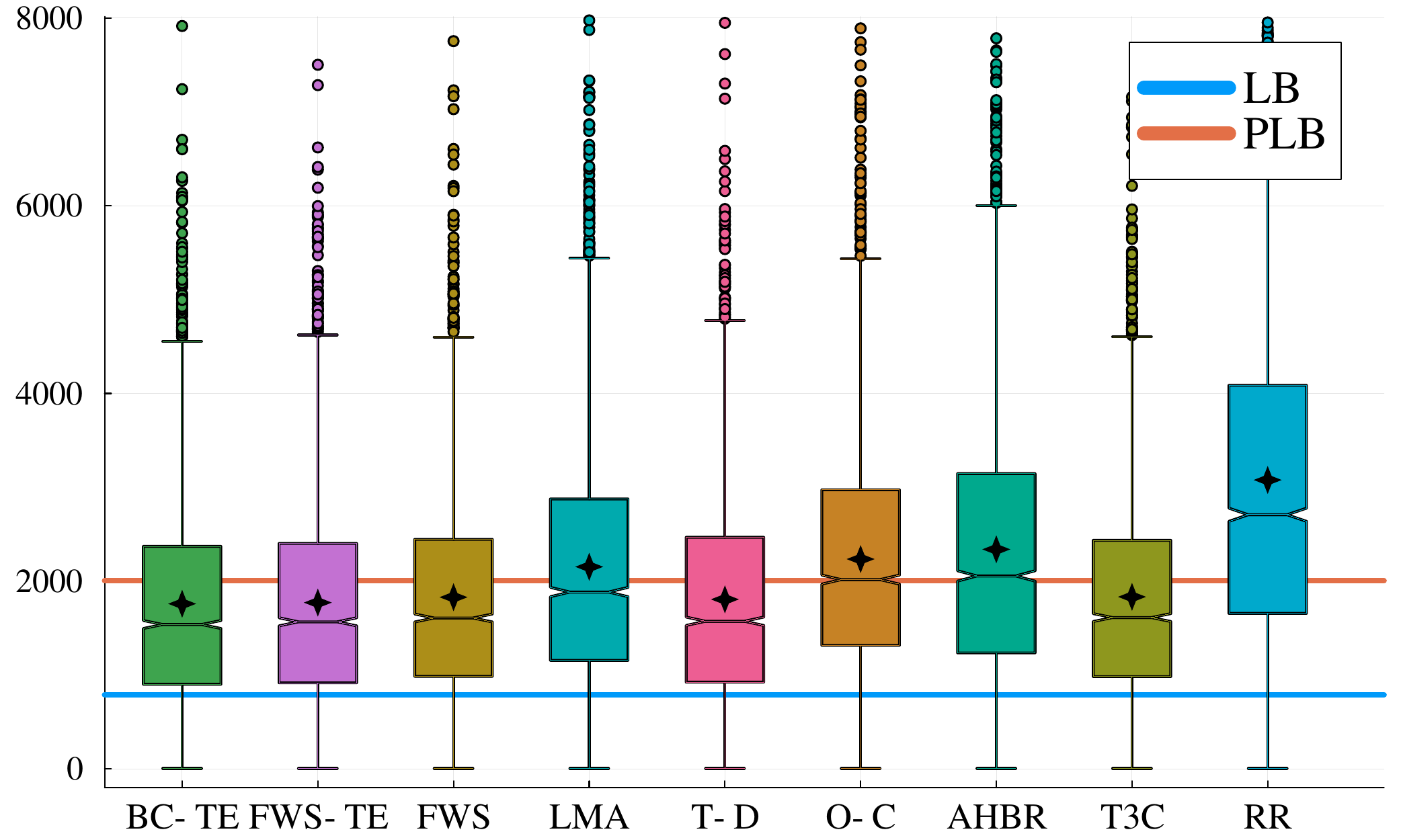}
    \caption{Gaussian instance $\bmu_4^{\rG}$}
     \end{subfigure}
\caption{Stopping times of various policies for $\delta=0.1$ over 3,000 independent runs. The black star denotes the mean of stopping times.
LB denotes the lower bound given in (\ref{eq: bai_LB}), and PLB denotes the practical version of LB considered in \citet{degenne2019non}.}
\label{fig: bai_all}
\end{figure}

\begin{table}[t]
    \caption{Relative average time of one step of various policies.}
    \vspace{0.5em}
    \label{tab: bai_T}
    \centering
    \begin{tabular}{c|ccccccccc}
       $\bmu$ & BC-TE & FWS-TE & FWS & LMA & T-D & O-C & AHBR & T3C & RR \\
        \hline
        $\bmu^{\mathrm{B}}_5$ &1 & 35.53 & 40.13 & 1.743 & 43.52 & 448.1 & 2.695 & 0.8415 & 0.3246 \\
        $\bmu^{\mathrm{G}}_4$ & 1 & 80.77 & 96.30 & 3.588 & 582.3 & 4533 & 3.935 & 0.7111 & 0.4226
    \end{tabular}
\end{table}

\section{Conclusion}\label{sec: bai_conc}
In this paper, we introduced BC-TE, a computationally efficient approach for solving the BAI problem in SPEF bandits. 
By combining a gradient-based policy with Thompson sampling, BC-TE overcame the limitations of existing approaches that involve computationally expensive optimization problems, forced exploration steps, or hyperparameter tuning.
Through theoretical analysis and experimental evaluation, we demonstrated that TS can serve as a substitute for the forced exploration steps in BAI problems.
Although BC-TE is not universally optimal in general, we showed its optimality for the two-armed bandits setting and provided a comparison with $\beta$-optimality.
Simulation results further validated the effectiveness of BC-TE, showing competitive sample complexity and improved computational efficiency compared to other optimal policies.

\section*{Acknowledgement}
JL was supported by JST SPRING, Grant Number JPMJSP2108. JH was supported by JSPS, KAKENHI Grant Number JP21K11747, Japan. MS was supported by the Institute for AI and Beyond, UTokyo.

\bibliographystyle{plainnat}
\bibliography{ref}

\begin{thebibliography}{35}
\providecommand{\natexlab}[1]{#1}
\providecommand{\url}[1]{\texttt{#1}}
\expandafter\ifx\csname urlstyle\endcsname\relax
  \providecommand{\doi}[1]{doi: #1}\else
  \providecommand{\doi}{doi: \begingroup \urlstyle{rm}\Url}\fi

\bibitem[Agrawal and Goyal(2012)]{agrawal2012analysis}
Shipra Agrawal and Navin Goyal.
\newblock Analysis of {T}hompson sampling for the multi-armed bandit problem.
\newblock In \emph{Annual Conference on Learning Theory}. PMLR, 2012.

\bibitem[Barrier et~al.(2022)Barrier, Garivier, and Koc{\'a}k]{barrier2022non}
Antoine Barrier, Aur{\'e}lien Garivier, and Tom{\'a}{\v{s}} Koc{\'a}k.
\newblock A non-asymptotic approach to best-arm identification for gaussian bandits.
\newblock In \emph{International Conference on Artificial Intelligence and Statistics}. PMLR, 2022.

\bibitem[Bubeck et~al.(2009)Bubeck, Munos, and Stoltz]{bubeck2009pure}
S{\'e}bastien Bubeck, R{\'e}mi Munos, and Gilles Stoltz.
\newblock Pure exploration in multi-armed bandits problems.
\newblock In \emph{International Conference on Algorithmic Learning Theory}. Springer, 2009.

\bibitem[Bubeck et~al.(2011)Bubeck, Munos, and Stoltz]{bubeck2011pure}
S{\'e}bastien Bubeck, R{\'e}mi Munos, and Gilles Stoltz.
\newblock Pure exploration in finitely-armed and continuous-armed bandits.
\newblock \emph{Theoretical Computer Science}, 2011.

\bibitem[Capp{\'e} et~al.(2013)Capp{\'e}, Garivier, Maillard, Munos, and Stoltz]{cappe2013kullback}
Olivier Capp{\'e}, Aur{\'e}lien Garivier, Odalric-Ambrym Maillard, R{\'e}mi Munos, and Gilles Stoltz.
\newblock Kullback-leibler upper confidence bounds for optimal sequential allocation.
\newblock \emph{The Annals of Statistics}, 2013.

\bibitem[Chen et~al.(2014)Chen, Lin, King, Lyu, and Chen]{chen2014combinatorial}
Shouyuan Chen, Tian Lin, Irwin King, Michael~R Lyu, and Wei Chen.
\newblock Combinatorial pure exploration of multi-armed bandits.
\newblock In \emph{Advances in Neural Information Processing Systems}. Curran Associates, Inc., 2014.

\bibitem[Degenne et~al.(2019)Degenne, Koolen, and M{\'e}nard]{degenne2019non}
R{\'e}my Degenne, Wouter~M Koolen, and Pierre M{\'e}nard.
\newblock Non-asymptotic pure exploration by solving games.
\newblock In \emph{Advances in Neural Information Processing Systems}. Curran Associates, Inc., 2019.

\bibitem[Even-Dar et~al.(2006)Even-Dar, Mannor, and Mansour]{even2006action}
Eyal Even-Dar, Shie Mannor, and Yishay Mansour.
\newblock Action elimination and stopping conditions for the multi-armed bandit and reinforcement learning problems.
\newblock \emph{Journal of Machine Learning Research}, 2006.

\bibitem[Gabillon et~al.(2012)Gabillon, Ghavamzadeh, and Lazaric]{gabillon2012best}
Victor Gabillon, Mohammad Ghavamzadeh, and Alessandro Lazaric.
\newblock Best arm identification: A unified approach to fixed budget and fixed confidence.
\newblock In \emph{Advances in Neural Information Processing Systems}. Curran Associates, Inc., 2012.

\bibitem[Garivier and Kaufmann(2016)]{garivier2016optimal}
Aur{\'e}lien Garivier and Emilie Kaufmann.
\newblock Optimal best arm identification with fixed confidence.
\newblock In \emph{Annual Conference on Learning Theory}. PMLR, 2016.

\bibitem[Ghosh(2011)]{ghosh2011objective}
Malay Ghosh.
\newblock Objective priors: An introduction for frequentists.
\newblock \emph{Statistical Science}, 2011.

\bibitem[Honda and Takemura(2014)]{honda2014optimality}
Junya Honda and Akimichi Takemura.
\newblock Optimality of {T}hompson sampling for {G}aussian bandits depends on priors.
\newblock In \emph{International Conference on Artificial Intelligence and Statistics}. PMLR, 2014.

\bibitem[Jedra and Proutiere(2020)]{jedra2020optimal}
Yassir Jedra and Alexandre Proutiere.
\newblock Optimal best-arm identification in linear bandits.
\newblock In \emph{Advances in Neural Information Processing Systems}. Curran Associates, Inc., 2020.

\bibitem[Jourdan and Degenne(2022)]{jourdan2022non}
Marc Jourdan and R{\'e}my Degenne.
\newblock Non-asymptotic analysis of a {UCB}-based top two algorithm.
\newblock \emph{arXiv preprint arXiv:2210.05431}, 2022.

\bibitem[Jourdan et~al.(2022)Jourdan, Degenne, Baudry, de~Heide, and Kaufmann]{jourdan2022top}
Marc Jourdan, R{\'e}my Degenne, Dorian Baudry, Rianne de~Heide, and Emilie Kaufmann.
\newblock Top two algorithms revisited.
\newblock In \emph{Advances in Neural Information Processing Systems}. Curran Associates, Inc., 2022.

\bibitem[Jourdan et~al.(2023)Jourdan, R{\'e}my, and Emilie]{pmlr-v201-jourdan23a}
Marc Jourdan, Degenne R{\'e}my, and Kaufmann Emilie.
\newblock Dealing with unknown variances in best-arm identification.
\newblock In \emph{International Conference on Algorithmic Learning Theory}. PMLR, 2023.

\bibitem[Kalyanakrishnan et~al.(2012)Kalyanakrishnan, Tewari, Auer, and Stone]{kalyanakrishnan2012pac}
Shivaram Kalyanakrishnan, Ambuj Tewari, Peter Auer, and Peter Stone.
\newblock {PAC} subset selection in stochastic multi-armed bandits.
\newblock In \emph{International Conference on Machine Learning}, 2012.

\bibitem[Kaufmann and Koolen(2021)]{kaufmann2021mixture}
Emilie Kaufmann and Wouter~M Koolen.
\newblock Mixture martingales revisited with applications to sequential tests and confidence intervals.
\newblock \emph{The Journal of Machine Learning Research}, 2021.

\bibitem[Kaufmann et~al.(2012)Kaufmann, Korda, and Munos]{kaufmann2012thompson}
Emilie Kaufmann, Nathaniel Korda, and R{\'e}mi Munos.
\newblock Thompson sampling: An asymptotically optimal finite-time analysis.
\newblock In \emph{International Conference on Algorithmic Learning Theory}. Springer, 2012.

\bibitem[Komiyama et~al.(2022)Komiyama, Tsuchiya, and Honda]{komiyama2022minimax}
Junpei Komiyama, Taira Tsuchiya, and Junya Honda.
\newblock Minimax optimal algorithms for fixed-budget best arm identification.
\newblock In \emph{Advances in Neural Information Processing Systems}. Curran Associates, Inc., 2022.

\bibitem[Koolen(2019)]{Wouter}
Wouter~M Koolen.
\newblock tidnabbil: Julia library for bandit experiments.
\newblock \url{https://bitbucket.org/wmkoolen/tidnabbil/src/master/}, 2019.

\bibitem[Korda et~al.(2013)Korda, Kaufmann, and Munos]{KordaTS}
Nathaniel Korda, Emilie Kaufmann, and Remi Munos.
\newblock Thompson sampling for 1-dimensional exponential family bandits.
\newblock In \emph{Advances in Neural Information Processing Systems}. Curran Associates, Inc., 2013.

\bibitem[Kuroki et~al.(2020)Kuroki, Xu, Miyauchi, Honda, and Sugiyama]{kuroki2020polynomial}
Yuko Kuroki, Liyuan Xu, Atsushi Miyauchi, Junya Honda, and Masashi Sugiyama.
\newblock Polynomial-time algorithms for multiple-arm identification with full-bandit feedback.
\newblock \emph{Neural Computation}, 2020.

\bibitem[Lee et~al.(2023)Lee, Honda, Chiang, and Sugiyama]{Lee2023}
Jongyeong Lee, Junya Honda, Chao-Kai Chiang, and Masashi Sugiyama.
\newblock Optimality of {T}hompson sampling with noninformative priors for {P}areto bandits.
\newblock In \emph{International Conference on Machine Learning}, 2023.

\bibitem[Lee et~al.(2024)Lee, Honda, and Sugiyama]{pmlr-v222-lee24a}
Jongyeong Lee, Junya Honda, and Masashi Sugiyama.
\newblock Thompson exploration with best challenger rule in best arm identification.
\newblock In \emph{Proceedings of the 15th Asian Conference on Machine Learning}, Proceedings of Machine Learning Research. PMLR, 2024.

\bibitem[Maron and Moore(1997)]{maron1997racing}
Oden Maron and Andrew~W Moore.
\newblock The racing algorithm: {M}odel selection for lazy learners.
\newblock \emph{Artificial Intelligence Review}, 1997.

\bibitem[M{\'e}nard(2019)]{menard2019gradient}
Pierre M{\'e}nard.
\newblock Gradient ascent for active exploration in bandit problems.
\newblock \emph{arXiv preprint arXiv:1905.08165}, 2019.

\bibitem[Mukherjee and Tajer(2022)]{mukherjee2022sprt}
Arpan Mukherjee and Ali Tajer.
\newblock {SPRT}-based best arm identification in stochastic bandits.
\newblock In \emph{International Symposium on Information Theory}. IEEE, 2022.

\bibitem[Qin et~al.(2017)Qin, Klabjan, and Russo]{qin2017improving}
Chao Qin, Diego Klabjan, and Daniel Russo.
\newblock Improving the expected improvement algorithm.
\newblock In \emph{Advances in Neural Information Processing Systems}. Curran Associates, Inc., 2017.

\bibitem[Riou and Honda(2020)]{riou2020bandit}
Charles Riou and Junya Honda.
\newblock Bandit algorithms based on {T}hompson sampling for bounded reward distributions.
\newblock In \emph{International Conference on Algorithmic Learning Theory}. PMLR, 2020.

\bibitem[Robert et~al.(2009)Robert, Chopin, and Rousseau]{robert2009rejoinder}
Christian~P Robert, Nicolas Chopin, and Judith Rousseau.
\newblock Rejoinder: {H}arold {J}effreys’s theory of probability revisited.
\newblock \emph{Statistical Science}, 2009.

\bibitem[Russo(2016)]{russo2016simple}
Daniel Russo.
\newblock Simple {B}ayesian algorithms for best arm identification.
\newblock In \emph{Annual Conference on Learning Theory}. PMLR, 2016.

\bibitem[Shang et~al.(2020)Shang, Heide, Menard, Kaufmann, and Valko]{shang2020fixed}
Xuedong Shang, Rianne Heide, Pierre Menard, Emilie Kaufmann, and Michal Valko.
\newblock Fixed-confidence guarantees for bayesian best-arm identification.
\newblock In \emph{International Conference on Artificial Intelligence and Statistics}, 2020.

\bibitem[Slivkins et~al.(2019)]{slivkins2019introduction}
Aleksandrs Slivkins et~al.
\newblock Introduction to multi-armed bandits.
\newblock \emph{Foundations and Trends{\textregistered} in Machine Learning}, 2019.

\bibitem[Wang et~al.(2021)Wang, Tzeng, and Proutiere]{wang2021fast}
Po-An Wang, Ruo-Chun Tzeng, and Alexandre Proutiere.
\newblock Fast pure exploration via frank-wolfe.
\newblock In \emph{Advances in Neural Information Processing Systems}. Curran Associates, Inc., 2021.

\end{thebibliography}

\clearpage
\appendix
\section{Additional notation}
Before beginning the proof, we first define good events on estimates $\hmu_i(t)$ and Thompson samples $\tmu_i(t)$ for any $\eps >0$,
\begin{align*}
    \eA_i(t)= \eA_{i,\eps}(t) &:= \begin{cases} \{ \hmu_1 (t) \geq \mu_1 - \epsilon \}, & \text{if } i=1, \\ 
    \{ \hmu_i(t) \leq \mu_i + \epsilon \}, &\text{otherwise},
    \end{cases} \\
    \eB_i(t)=\eB_{i,\eps}(t) &:= \{ | \hmu_i(t) - \mu_i | \leq \epsilon \} ,\\
    \etB_i(t)=\etB_{i,\eps}(t) &:= \{ | \tmu_i(t) - \mu_i | \leq \epsilon \} ,\\
    \eM(t) &:= \{ m(t) = \tm(t) \},
\end{align*}
Note that for all $i \in [K]$ and $t \in \mathbb{N}$, $\eB_i(t) \subset \eA_i(t)$ holds.

Next, let us define another random variables $D_1 = D_{1,\eps} := \max_{i \ne 1} D_{i,\eps}$ where
\begin{equation*}
    D_i = D_{i,\eps} := \sup_{t \geq 2K+1} \I[\eB_{i,\eps}^c(t)]N_i(t) d\left( \hmu_i(t), \hmu_1(t) \right)
\end{equation*}
denotes the supremum of $N_a(t) d\left( \hmu_a(t), \hmu_1(t) \right)$ when $\eB_{i,\eps}^c(t)$ occurs.
In other words, 
\begin{equation*}
    \{ N_a(t) d\left( \hmu_a(t), \hmu_1(t) \right) \geq D_{i,\eps} \} \implies \{ \I[\eB_{i,\eps}(t)] =1 \}.
\end{equation*}
We further define $\ud_1 = d(\mu_1 -\eps, \mu_2 + \eps)$ and for $i \ne 1$
\begin{equation}\label{def: under_d}
\ud_i=
\min_{\substack{\mu\in [\mu_i',\mu_1'],\\
\mu_i'\le \mu_i+\epsilon,\
\mu_1'\ge \mu_1-\epsilon,\ \\
d(\mu_i',\mu)\ge d(\mu_1',\mu)}}
d(\mu_i',\mu).
\end{equation}

\section{Proof of Lemma \ref{lem: subgfind}: Subdifferentials}\label{sec: bai_lem_subg}
Here, we derive the subdifferential of the objective function $g$.
\begin{proof}
    By abuse of notation, we define a characteristic function $I_{\Sigma_K}: \mathbb{R}^K \rightarrow \mathbb{R}$,
\begin{equation*}
    I_{\Sigma_K}(x)=\begin{cases}
    0, & \text{if $x\in {\Sigma_K}$}\\  
    -\infty, & \text{if $x \not\in {\Sigma_K}$}.
  \end{cases}
\end{equation*}
Then, the problem in (\ref{eq: bai_tstar_f}) can be written as
\begin{equation}\label{eq:optpro}
     \sup_{\bw \in \Sigma_K} \min_{i \ne 1} f_{i}(\bw; \bmu) = \max_{\bw \in \mathbb{R}^K}\left\{\min_{i\ne 1} f_i(\bw) + I_{\Sigma_K} (\bw) \right\}.
\end{equation}
Then, the set of differential of (\ref{eq:optpro}) is
\begin{equation*}
    \partial\left(\min_{a\ne 1} f_a(\bw) +  I_{\Sigma_K} (\bw)\right) = \left\{q+r : q \in \partial \min_{i\ne a} f_a(\bw), \ r \in \partial I_{\Sigma_K} (\bw)\right\}.
\end{equation*}
Let $\partial I_{\Sigma_K} (\bw)$ denote the set of subgradient $\bv$ of $I_{\Sigma_K}$ at point $(\bw; \bmu)$.
Then, $\partial I_{\Sigma_K} (\bw)$ is written as
\begin{equation}\label{eq:indicator}
    \partial I_{\Sigma_K} (\bw) = \{\bv \in \mathbb{R}^K: \forall \bmx \in \mathbb{R}^K, I_{\Sigma_K}(\bmx) \leq I_{\Sigma_K}(\bw) + \bv^\top (\bmx - \bw) \}
\end{equation}
From the definition of $I_{\Sigma_K}$, if $\bmx \not\in \Sigma_K$, the inequality constraint in (\ref{eq:indicator}) always holds for any $\bv \in \mathbb{R}^K$. 
Thus, it suffices to show that
\begin{align*}
    \partial I_{\Sigma_K} (\bw) &= \{\bv \in \mathbb{R}^K: \forall \bmx \in \Sigma_K, I_{\Sigma_K}(\bmx) \leq I_{\Sigma_K}(\bw) + \bv^\top (\bmx - \bw) \} \\
    &= \{r \mathbf{1}: r \in \mathbb{R} \}, \numberthis{\label{eq:r1}}
\end{align*}
which implies that all subgradients $\bv$ can be written as a multiple of the $K$-dimensional all-one vector $\mathbf{1} = [1, \ldots, 1]$.
To show the equivalence, we will show that 
\begin{align*}
    (B1) &: \{r \mathbf{1}: r \in \mathbb{R} \} \subset \partial I_{\Sigma_K} (\bw) \\
    (B2) &: \{r \mathbf{1}: r \in \mathbb{R} \} \supset \partial I_{\Sigma_K} (\bw).
\end{align*}

\subsection{Case (B1)}
Note that $\mathbf{0} \in  \partial I_{\Sigma_K} (\bw)$, which implies $ \partial I_{\Sigma_K} (\bw)\ne \emptyset$.
Since $\bmx \in \Sigma_K$, $\bv\in  \partial I_{\Sigma_K} (\bw)$ satisfies $0 \leq \bv^\top (\bmx -\bw)$ for all $\bmx \in \Sigma_K$.
One can see that $\{r \mathbf{1}: r \in \mathbb{R} \} \subset \partial I_{\Sigma_K} (\bw)$ for $\bw \in \Sigma_K$ since $\sum_{i=1}^K w_i = \sum_{i=1}^K x_i = 1$ from the assumption.

\subsection{Case (B2)}
Then, we need to show the equality in (\ref{eq:r1}) for $\bw \in \mathrm{Int}\,\Sigma_K$.
At first, let assume $K \geq 2$ and $\bv = r\mathbf{1} + \sum_{i=1}^K a_i e_i$, where $e_i$ is a standard basis for $\mathbb{R}^K$ and $a_i \in \mathbb{R}$.
Then, $\forall \bmx \in \Sigma_K$,
\begin{equation}\label{eq:15}
    0 \leq \sum_{i=1}^K a_i (x_i - w_i)
\end{equation}
holds.
We will prove the equality in (\ref{eq:r1}) by contradiction, i.e., we assume that there exists $i \ne j \in [K]$ such that $a_i \ne a_j$.
From the definition of $ \mathrm{Int}\Sigma_K$, we can take a positive constant $\epsilon\in\mathbb{R}_{+}$ satisfying $0< \epsilon < \min (\min_{i}w_i, 1-\max_{i}(w_i))$.\footnote{Note that such $\epsilon$ always exists by Archimedean property if $w$ is in the interior of the probability simplex, i.e., $\forall i \in [K]$, $w_i \ne 0, 1$.}

Define two $K$ dimensional vectors as
\begin{equation*}
    \bmx^1 = (x_i)_{i=1}^K = \begin{cases}
 w_i,   &\text{ if } i \in [K]\setminus \{ i_1, i_2 \}, \\
 w_i + \epsilon,&\text{ if } i = i_1, \\ 
 w_i - \epsilon, &\text{ if } i = i_2,
\end{cases}
\end{equation*}
and
\begin{equation*}
\bmx^2 = (x_i)_{i=1}^K = \begin{cases}
 w_i , &\text{ if } i \in [K]\setminus \{ i_1, i_2 \} ,\\
 w_i - \epsilon, &\text{ if } i = i_1 ,\\
 w_i + \epsilon, &\text{ if } i = i_2,
\end{cases}
\end{equation*}
where $i_1 \ne i_2 \in [K]$.
Then, both $\bmx^1$ and $\bmx^2$ are in $\Sigma_K$.
From (\ref{eq:15}), this implies that two inequalities
\begin{equation*}
    0 \leq \epsilon(a_{i_1} - a_{i_2}) \text{ and } 0 \leq - \epsilon(a_{i_1} - a_{i_2})
\end{equation*}
hold at the same time.
Thus, $a_{i_1} = a_{i_2}$ should hold.
However, we can make these kinds of vectors for every pair of bases, which means that $\not\exists i\ne j \in [K]$ such that $a_i \ne a_j$.
This is a contradiction, and thus (\ref{eq:r1}) holds.

\subsection{Conclusion}
Consequently, it holds $\forall \bw \in \mathrm{Int}\Sigma_K$ that
\begin{align*}
    \partial g &= \left\{ q + r \mathbf{1} : q \in \textbf{Co}\bigcup \{\partial f_i(\bw; \bmu) :  f_i(\bw; \bmu)=g(\bw; \bmu) \}, r \in \mathbb{R} \right\} \\
    &= \left\{ q + r \mathbf{1} : q \in \textbf{Co}\bigcup \{\nabla_{\bw} f_i(\bw; \bmu) :  f_i(\bw)=g(\bw) \}, r \in \mathbb{R} \right\},
\end{align*}
where $\textbf{Co}\bigcup \left\{\nabla_{\bw} f_i(\bw; \bmu) :  f_i(\bw; \bmu)=g(\bw; \bmu) \right\}$ is the convex hull of the union of superdifferentials of all active function at $\bw$.
Let us define the set 
\begin{equation*}
    \mathcal{J}(\bw; \bmu) := \argmin_{i\ne 1}f_i(\bw; \bmu) = \{ i\in [K]: f_i = g \},
\end{equation*}
which concludes the proof.
\end{proof}

\section{Comparison with other optimality notions}\label{sec: bai_comp}
In this section, we provide more detail that completes Sections~\ref{sec: bai_rslt} and~\ref{sec: bai_exp}.

\subsection{Two-armed bandits}\label{sec: bai_twoarm}
Firstly, let us introduce a function that enables us to derive a more explicit formula for $\bw^*(\bmu)$, for any $i \ne 1$,
\begin{equation*}
    k_i(x; \bmu) = d\left(\mu_1, \frac{1}{1+x}\mu_1 + \frac{x}{1+x}\mu_i\right) + x d\left(\mu_i, \frac{1}{1+x}\mu_1 + \frac{x}{1+x}\mu_i\right).
\end{equation*}
As demonstrated in \citet{garivier2016optimal}, this function is a strictly increasing bijective mapping from $[0, \infty)$ onto $[0, d(\mu_1, \mu_a))$.
Therefore, one can define $l_i$ as the inverse function of $k_i$ for any $i \ne 1$ and $l_1$ as a constant function, which is
\begin{align*}
   k_i^{-1} = l_i &: [0, d(\mu_1, \mu_i)) \mapsto [0, \infty) \numberthis{\label{eq: bai_def_li}}\\
    l_1 &: [0, d(\mu_1, \mu_i)) \mapsto 1.
\end{align*}
Then, \citet{garivier2016optimal} provided the following characterization of $\bw^*(\bmu)$.
\begin{lemma}[Theorem 5 in \citet{garivier2016optimal}]
    For every $i \in [K]$, 
    \begin{equation*}
        \bw^*_i(\bmu) = \frac{l_i(y^*)}{\sum_{a=1}^K l_a(y^*)},
    \end{equation*}
    where $y^*$ is the unique solution of the equation $F_{\bmu}(y)=1$, and where
    \begin{equation*}
        F_{\bmu}: y \mapsto \sum_{i=2}^K \frac{d\left(\mu_1, \frac{\mu_1 + l_i(y)\mu_i}{1+l_i(y)} \right)}{d\left(\mu_i, \frac{\mu_1 + l_i(y)\mu_i}{1+l_i(y)} \right)}
    \end{equation*}
    is a continuous, increasing function on $[0, d(\mu_1, \mu_2))$ such that $F_{\bmu}(0) = 0$ and $F_{\bmu}(y) =\infty$ when $y \to d(\mu_1, \mu_2)$.
\end{lemma}
However, to derive a more explicit formula for the maximizer of (\ref{eq: bai_def_our_optimality}), we require another function for any $i\ne 1$
\begin{equation*}
    h_i(z; \bmu) = (1-z) d(\mu_1, (1-z)\mu_1 + z\mu_i ) + zd(\mu_i, (1-z)\mu_1 + z\mu_i ),
\end{equation*}
whose domain is $[0,1]$. 
The derivative of this function is
\begin{equation*}
    h_i'(z; \bmu) = d(\mu_i, (1-z)\mu_1 + z\mu_i ) - d(\mu_1, (1-z)\mu_1 + z\mu_i ).
\end{equation*}
Thus, $h_i(z;\bmu)$ is a concave function with $h_i(0;\bmu) = 0$ and $h_i(1, \bmu) = 0$. 
It reaches its maximum at
\begin{equation}\label{eq: bai_def_zstar}
    z_i^*(\bmu) : d(\mu_i, (1-z_i^*)\mu_1 + z_i^* \mu_i ) = d(\mu_1, (1-z_i^*)\mu_1 + z_i^* \mu_i ).
\end{equation}
Therefore, one can see that $\gamma = z_{2}^*$.
From the definitions of $f_i,k_i,$ and $h_i$, one can find the following relationship
\begin{align}\label{eq: bai_fkh_relation}
    f_i(\bw; \bmu) = w_1 k_i\left( \frac{w_i}{w_1} ; \bmu\right) = (w_1 + w_i) h_i \left( \frac{w_i}{w_1 + w_i } ; \bmu\right).
\end{align}
For $z_i = \frac{w_i}{w_1+w_i}$, the equality between $h_i$ and $k_i$ can be written as
\begin{equation*}
    h_i(z_i; \bmu) = (1-z_i) k_i \left( \frac{z_i}{1-z_i} ;\bmu \right).
\end{equation*}
We further define the problem-dependent constant $\uz_i \in [0,1]$ for $i\ne 1$ satisfying
\begin{equation}\label{eq: bai_def_uz}
    \uz_i : k_i \left( \frac{ \uz_i}{1- \uz_i} ;\bmu \right) = k_{2} \left( \frac{z_{2}^*}{1-z_{2}^*} ;\bmu \right)
\end{equation}
and $\uz_1 = \frac{1}{2}$.
Here, we have $\uz_{2} = z_{2}^*$ and $\uz_i \leq z_{2}^*$ since $k_i$ is strictly increasing and $k_i(x;\bmu) \leq k_j(x;\bmu)$ holds for any $x \in \mathbb{R}_{+}$ if $\mu_i \leq \mu_j$~\citep[see][Appendix A.3.]{garivier2016optimal}.
Based on $\uz_i$, we define a normalized proportion $\uw \in \Sigma_K$ by
\begin{equation}\label{eq: bai_def_uw_uz}
    \uw_i(\bmu) = \frac{\frac{\uz_i}{1-\uz_i}}{\sum_{i=1}^K \frac{\uz_i}{1-\uz_i}} = \frac{l_i(\underline{y})}{\sum_{i=1}^K l_i(\underline{y})},
\end{equation}
where $\underline{y}= k_i\left(\frac{\uz_i}{1-\uz_i} ;\bmu \right)$ for any $i \ne 1$.
Therefore, Theorem~\ref{thm: bai_sample} implies that the empirical proportion of arm plays of BC-TE will converge to $\uw$, which is equivalent to $g(\bw^t ; \hbmu(t)) \to g(\uw; \bmu)$.
Here, one can see that $F_{\bmu}(\underline{y}) \geq 1$ since 
\begin{equation*}
    \frac{d\left(\mu_1, \frac{\mu_1 + \frac{\uz_{2}}{1-\uz_{2}}\mu_{2}}{1+\frac{\uz_{2}}{1-\uz_{2}}} \right)}{d\left(\mu_{2}, \frac{\mu_1 + \frac{\uz_{2}}{1-\uz_{2}}\mu_{2}}{1+\frac{\uz_{2}}{1-\uz_{2}}} \right)} = 
    \frac{d\left(\mu_1, (1-\uz_{2})\mu_1 + \uz_{2} \mu_{2}) \right)}{d\left(\mu_{2}, (1-\uz_{2})\mu_1 + \uz_{2} \mu_{2}) \right)}
    = 1
\end{equation*}
holds from the definition of $\uz_{2} = z_{2}^*$ in (\ref{eq: bai_def_zstar}), which directly implies that $\underline{y} \geq y^*$.
However, it is important to note that from $\uz_i \leq z_i^*$, it always hold that for any $i \ne 1$
\begin{equation*}
     \frac{d\left(\mu_1, (1-\uz_{i})\mu_1 + \uz_{i} \mu_{i}) \right)}{d\left(\mu_i, (1-\uz_{i})\mu_1 + \uz_{i} \mu_{2}) \right)} \leq \frac{d\left(\mu_1, (1-z_{i}^*)\mu_1 + z_{i}^* \mu_{i}) \right)}{d\left(\mu_i, (1-z_{i}^*)\mu_1 + z_{i}^* \mu_{i}) \right)} =1.
\end{equation*}
This implies that
\begin{equation}\label{eq: bai_worst_BCTE}
    1 \leq F_{\bmu}(\underline{y}) \leq K-1,
\end{equation}
where the right equality holds only when $\mu_2 = \mu_3 =\ldots = \mu_K$.
Here, it is important to note that the left equality is always valid for two-armed bandit problems.
In other words, BC-TE is \emph{asymptotically optimal} in the context of two-armed bandit problems.

\subsection{Gaussian bandits}\label{sec: bai_gauss}
Here, we prove Lemma~\ref{lem: bai_Gaussian} based on the definitions provided in Section~\ref{sec: bai_twoarm}.

\begin{proof}[Proof of Lemma~\ref{lem: bai_Gaussian}]
    Since $d(\mu, \mu') = \frac{(\mu-\mu')^2}{2\sig^2}$, for any $i\ne 1$ and $\Delta_i = \mu_1 - \mu_i$
\begin{align*}
    k_i(x; \bmu) &= \left( \frac{x}{1+x} \right)^2 \frac{\Delta_i^2}{2\sig^2} + \frac{x}{(1+x)^2} \frac{\Delta_i^2}{2\sig^2} = \frac{x}{1+x}\frac{\Delta_i^2}{2\sig^2} \\
   h_i(z; \bmu) &= z(1-z) \frac{\Delta_i^2}{2\sig^2}.
\end{align*}
Firstly, from (\ref{eq: bai_def_zstar}), the maximizers of $h_i$, $z_i^*$ satisfies
\begin{equation*}
    \frac{\Delta_i^2}{2\sig^2} (1-z_i^*)^2 = \frac{\Delta_i^2}{2\sig^2} (z_i^*)^2,
\end{equation*}
which implies that $z_i^* = 1/2$ for any $i\ne 1$.
Then, for any $i \ne 1$, from the definition of $\uz_i$ in (\ref{eq: bai_def_uz}), it holds
\begin{align*}
    k_2(1;\bmu) = \frac{\Delta_2^2}{4\sig^2} &= k_i\left(\frac{\uz_i}{1-\uz_i};\bmu\right) \\
    &= \frac{\Delta_i^2}{2\sig^2} \uz_i,
\end{align*}
which implies $\uz_i = \frac{\Delta_2^2}{2\Delta_i^2}$ for $i \ne 1$.
Therefore, we obtain that $\uw_i = \frac{\frac{\Delta_2^2}{2\Delta_i^2 - \Delta_2^2}}{\sum_{a=1}^K \frac{\Delta_2^2}{2\Delta_a^2 - \Delta_2^2}}$.
By letting $\Delta_1 = \Delta_2$, the objective function $g$ at $\uw$ can be written as
\begin{equation*}
    g(\uw; \bmu) = \uw_1 k_i\left(\frac{\uz_i}{1-\uz_i};\bmu\right) = \frac{1}{\sum_{a=1}^K \frac{\Delta_2^2}{2\Delta_a^2 - \Delta_2^2}} \frac{\Delta_2^2}{4\sig^2},
\end{equation*}
which implies that
\begin{equation*}
    \underline{T}(\bmu) = \sum_{i=1}^K \frac{4\sig^2}{\Delta_i^2 + (\Delta_i^2- \Delta_2^2)}. \qedhere
\end{equation*}
\end{proof}

\subsection{Additional numerical results}
Here, we first provide additional comparisons between $\underline{T}(\bmu)$ and $T^{1/2}(\bmu)$.

In Figure 3.(a), we zoom in on Figure 1.(a) from the main paper specifically for $K \leq 50$. It can be observed that $\underline{T}(\bmu^{(1)})$ is closer to $T^*(\bmu^{(1)})$ compared to $T^{1/2}(\bmu^{(1)})$.
Next, we consider a worst-case instance $\bmu'$ based on $\bmu^{(1)} = (0.3, 0.21)$, where we add additional arm $\mu_K = \mu_2$ for any $K$ in Figure 3.(b).
Therefore, in $\bmu'$, all suboptimal arms share the same expected rewards, e.g., $\mu' = (0.3, 0.21, 0.21, 0.21)$ for $K=4$.
This instance is of specific interest since one can observe that $\underline{T}(\bmu)$ differs from $T^*(\bmu)$ at most when all suboptimal arms have the same expected rewards according to (\ref{eq: bai_worst_BCTE}).
Even in such cases, $\underline{T}(\bmu')$ and $T^{1/2}(\bmu')$ exhibit a similar tendency, which would make BC-TE a reasonable policy in general.

Next, for the implementation in Section~\ref{sec: bai_exp}, we focus on T-D in our experiments although there exist two versions of the TaS policy.
T-D directly tracks the optimal proportion of arm plays at each round ($N(t) \rightsquigarrow t\bw^*(\hbmu(t))$), and it has been found to outperform the version with C-tracking in experiments, which tracks the cumulative optimal proportions ($N(t)~\rightsquigarrow~\sum_{s\leq t}\bw^*(\hbmu(s))$).

\begin{figure}
     \centering
     \begin{subfigure}[b]{0.48\textwidth}
     \label{fig: bai_one_small}
         \includegraphics[width=\textwidth]{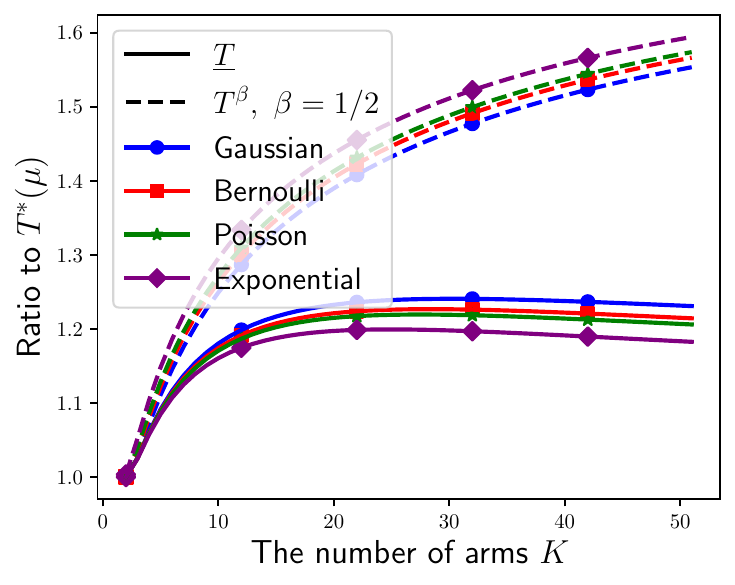}
    \caption{Instance $\bmu^{(1)}$ with small $K$.}
     \end{subfigure}
     \hfil
     \begin{subfigure}[b]{0.48\textwidth}
     \label{fig: bai_one_worst}
         \includegraphics[width=\textwidth]{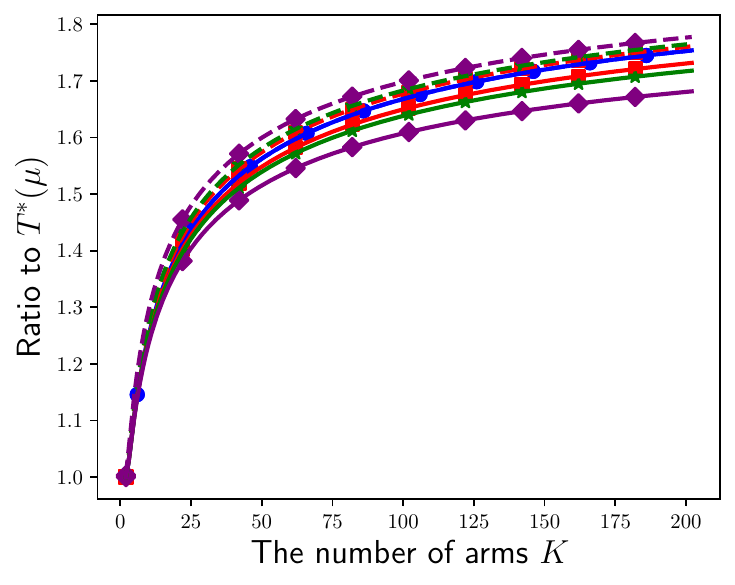}
    \caption{Worst case instance $\bmu'$ with varying $K$.}
     \end{subfigure}
\caption{The ratio of $\underline{T}(\bmu)$ and $T^{1/2}(\bmu)$ to $T^*(\bmu)$ for different reward distributions.}
\label{fig: bai_T_ratio_app}
\end{figure}

\section{Additional experimental results}\label{sec: bai_add_rslt}
In this section, we provide additional experimental results where the rewards follow the exponential distribution and Pareto distribution.

\paragraph{Exponential bandits}
In the first experiment, we consider the 5-armed Bernoulli bandit instance $\bmu^{\mathrm{E}}_5=(0.5, 0.45, 0.43, 0.4, 0.3)$ where $\bw^*(\bmu^{\mathrm{B}}_5) =(0.41, 0.40, 0.13, 0.05, 0.01)$.

\paragraph{Pareto bandits}
In the second experiment, we consider the 4-armed Pareto bandit instance $\bmu^{\mathrm{P}}_4 = (5.0, 3.0,$ $ 2.0, 1.5)$ with unit scale $\sig=1$ where $\bw^*(\bmu^{\mathrm{P}}_4) = (0.34, 0.60, 0.04, 0.01)$.
The density function of the Pareto distribution with shape $\theta > 0$ and scale $\sig >0$ is written as
\begin{equation*}
    f_{\mathrm{P}}(x; \theta, \sigma) = \frac{\theta \sigma^{\theta}}{x^{\theta +1}}.
\end{equation*}
Notice that since $\sig = 1$, the shape parameter is given as $\bm{\theta} = (1.25, 1.5, 2, 3)$, where the first three arms have \emph{infinite} variance.
It is worth noting that the sample complexity of T3C for $\delta \in \qty{0.01, 0.001}$ becomes extremely larger than other policies (e.g., more than 25,000), we exclude the result of T3C in this section although it performs well in the Gaussian and Bernoulli bandits.

\paragraph{Results}
The overall results are presented in Table~\ref{tab: bai_add}.
Similarly to the Gaussian and Bernoulli cases, both BC-TE and FWS-TE consistently show a better empirical performance than other optimal policies across most risk parameters, especially when large $\delta$ is considered.
Although the empirical probability of misidentification (error rate) for each policy is less than the given threshold $\delta$ for most cases, their error rates exceed the threshold when we consider $\bmu^{\mathrm{P}}_4$ with $\delta = 0.001$ as shown in Table~\ref{tab: par_error}.
This implies that the current choice of stopping rule, $\beta(t, \delta) = \log(\log(t)+1)/\delta)$, a widely-used heuristic, may be not appropriate when one considers the bandit instance possibly with infinite variance.

\begin{table}[t]
    \caption{Sample complexity over 3,000 independent runs, where outperforming policies are highlighted in boldface using one-sided Welch's t-test with the significance level 0.05.
    LB denotes the lower bound in (\ref{eq: bai_LB}), and PLB denotes the practical version of LB considered in \citet{degenne2019non}.
    $\bmu^{\mathrm{E}}_5$ denotes 5-armed Exponential bandit instance with means $(0.5, 0.45, 0.43, 0.4, 0.3)$ and $\bmu^{\mathrm{P}}_4$ denotes 4-armed Pareto bandit instance with means $(5.0, 3.0, 2.0, 1.5)$ and unit scale.}
    \label{tab: bai_add}
    \centering
    \begin{tabular}{ll|cc|cccc|cc}
        $\bmu$ &$\delta$ & BC-TE & FWS-TE & FWS & T-D & LMA & RR & PLB & LB \\
        \hline
       \multirow{4}{*}{$\bmu_5^{\rB}$} & 0.2 &  \textbf{2910} & \textbf{2938}&  3086 & 3158&  4092 & 6471& 3434& 747  \\
        &0.1 &  \textbf{3568} &  \textbf{3623} & 3791& 3840& 4851 &7753& 4074& 1579 \\
        & 0.01 &  \textbf{5743} &  \textbf{5849} &  5938 & 5977&  7165 & 12032& 6182& 4046  \\
        & 0.001&  \textbf{7977} &  \textbf{8010}&  \textbf{8085}&  \textbf{8023}&  9533 & 16201& 8278& 6194 \\
         \hline
       \multirow{4}{*}{$\bmu_4^{\mathrm{P}}$} & 0.2 &  \textbf{1164} & \textbf{1171}&  \textbf{1178} & 1268 & 1695 & 2329 & 937& 212  \\
        &0.1 &  \textbf{1447} &  \textbf{1478} & \textbf{1457}& 1554& 2016 &2792& 1120& 449 \\
        & 0.01 &  \textbf{2396} &  \textbf{2379} &  \textbf{2376}& 2493&  3059 & 4323& 1720& 1150  \\
        & 0.001&  \textbf{3270} &  \textbf{3249}&  \textbf{3174}& 3366&  4026 & 5792& 2318& 1760 \\
    \end{tabular}
\end{table}

\begin{table}[t]
    \caption{Error rate for $\bmu_4^{\mathrm{P}}$ and $\delta=0.001$.}
    \centering
    \begin{tabular}{cc| c c c c}
        BC-TE & FWS-TE & FWS & T-D & LMA & RR \\
        0.004 & 0.0047 & 0.0073 & 0.005 & 0.008 & 0.005  \\
    \end{tabular}
    \label{tab: par_error}
\end{table}

\section{Proof of Theorem \ref{thm: bai_TB}: Convergence of empirical means}\label{sec: bai_pfTB}
We begin the proof of Theorem~\ref{thm: bai_TB} by introducing two lemmas that show a sufficient condition to occur $\eB_i(t)$ for $i =1$ and $i\ne 1$, respectively.
\begin{lemma}\label{lem: OptGood}
For any constant $M>0$, assume that
\begin{align*}
\left\{m(t)=1,\,j(t)=j,\,i(t)=j,\,\eA_1(t), \eB_{j}(t),
\eM(t),\,N_j(t)>
\max\left\{M, D_1/\underline{d}_j\right\}
\right\}
\end{align*}
occurred for some $t$.
Then, for all $t'\ge t$, we have $\I[\eB_1(t')]=1$
and
\begin{align*}
N_1(t)\ge 
\frac{\max\{\underline{d}_j M, D_1\}}{d(\mu_1+\epsilon,\mu_j-\epsilon)}.
\end{align*}
\end{lemma}

\begin{lemma}\label{lem: SubGood}
For any constant $M>0$, assume that
\begin{align*}
\left\{m(t)=1,\,i(t)=1,\,\eA_{j(t)}(t), \eB_{1}(t),
\eM(t),\,N_1(t)>
\max\left\{M, \max_{i \ne 1} \frac{D_i}{\underline{d}_i} \right\}
\right\}
\end{align*}
occurred for some $t$.
Then, for all $i\neq 1$ and $t'\ge t$, we have $\I[\eB_i(t')]=1$ and
\begin{align*}
N_i(t)\ge 
\frac{\max\{\underline{d}_{i} M, D_{i}\}}{d(\mu_1+\epsilon,\mu_i-\epsilon)}.
\end{align*}
\end{lemma}
Therefore, if both events in Lemmas~\ref{lem: OptGood} and~\ref{lem: SubGood} occurred until rounds $T$, only $\{\eB_i(t)\}$ occurs for all $i\in[K]$ and $t \geq T$. 
The proofs of these lemmas are postponed to Section~\ref{sec: bai_TBlem}.

\begin{proof}[Proof of Theorem~\ref{thm: bai_TB}]
    Firstly, let us define another random variable $T_C \leq T_B$ such that
\begin{equation*}
    \forall s \geq T_C : \I[\eB_1(s)] = 1,
\end{equation*}
which implies that the mean estimate of the optimal arm is close to its true value after $T_C$ rounds.
Let $D=\max\left\{M, \frac{D_1}{\min_{a \in[K]} \ud_a} \right\}$ for some positive constant $M$ specified later and $T_M = \max (KD, T_C)$.
Let us consider a subset of rounds with any fixed $T > T_M$
\begin{align*}
    S_1(T) &:= \{ s \in [T_M, T]\cap \mathbb{N}: m(s)=1, i(s)=j(s), \eB_1(s), \eB_{j(s)}(s), \eM(s)\} \\
    &\hspace{0.2em}= \{ T_{S_1} =: s_{S_1,1},  s_{S_1,2}, \ldots, s_{S_1,|S_1(T)|}\} \\
    S_2(T) &:= \{ s \in [T_M, T]\cap \mathbb{N}: m(s)=1, i(s)=1, \eA_{j(s)}(s), \eB_1(s), \eM(s) \} \\
    &\hspace{0.2em}= \{ T_{S_2} =: s_{S_2,1},  s_{S_2,2}, \ldots, s_{S_2,|S_2(T)|}\},
\end{align*}
where $s_{S_{m}, k}$ implies the round when the event occurs $k$-th time for $m=1,2$, respectively.

Similarly, let us define a subset of rounds with any fixed $T >T_M$
\begin{align*}
    S_0(T) := \bigg\{& s \in [T_M, T]\cap \mathbb{N}: \{\eB_1(s) ,\eM^c(s)\} \cup \{\eB_1(s), \eB_{i(s)}^c, \eM(s) \}  \\ 
     &\hspace{1em}\cup \{  m(s)=1, i(s)=1,  \eB_1(s), \eA_{j(s)}^c(s), \eM(s)\}  \\
     &\hspace{2em}\cup \{m(s)\ne 1, i(s)=j(s), \eB_1(s),  \eA_{m(s)}^c(s), \eB_{j(s)}(s), \eM(s)\} \bigg\}
\end{align*}
and a random variable
\begin{multline*}
    T_S :=T_M + \sum_{s=T_M+1}^T \I[\eB_1(s), \eM^c(s)]  + \I[ \eB_1(s), \eB_{i(s)}^c, \eM(s)]     
    \\ + \I[m(s)=1, i(s)=1, \eB_1(s), \eA_{j(s)}^c(s), \eM(s)]  \\
    + \I[m(s)\ne 1, i(s)=j(s), \eB_1(s), \eB_{m(s)}^c(s), \eB_{j(s)}(s), \eM(s)],
\end{multline*}
such that $T_S = |S_0(T)| + T_M$ holds.
\paragraph{First objective}
Here, we first aim to show that for $t \geq T_M$, it holds
\begin{equation*}
    1 = \I[t \in S_0(T)] + \I[t \in S_1(T)] + \I[t \in S_2(T)].
\end{equation*}
Since $\eB_1(s)$ always holds for $s \geq T_M$, it holds that
\begin{align*}
    1 &= \I[\eB_1(s)] \\
    &=  \I[\eM^c(s), \eB_1(s)] + \I[\eM(s), \eB_1(s)] \\
    &=  \I[\eM^c(s), \eB_1(s)] + \I[\eM(s), \eB_1(s), m(s)=1]  + \I[\eM(s), \eB_1(s), m(s)\ne 1] \\
    &= \I[\eM^c(s), \eB_1(s)] \\
    & \hspace{1em} + \I[\eM(s), \eB_1(s), m(s)=1, i(s)=1] + \I[\eM(s), \eB_1(s), m(s)=1, i(s)=j(s)] \\
   & \hspace{1em}  + \I[\eM(s), \eB_1(s), m(s)\ne 1, i(s)=m(s), \eB_1(s), \eB^c_{m(s)}(s)] \\
   &\hspace{2em}+  \I[\eM(s), \eB_1(s), m(s)\ne 1, i(s)=j(s), \eB_1(s) \eA^c_{m(s)}(s)]  \numberthis{\label{eq: star_1}} \\
   &= \I[\eM^c(s), \eB_1(s)] \\
   &\hspace{1em} + \I[\eM(s), \eB_1(s), m(s)=1, i(s)=1, \eA^c_{j(s)}(s)] \\
   &\hspace{2em}+ \I[\eM(s), \eB_1(s), m(s)=1, i(s)=1, \eA_{j(s)}(s)] \\
   & \hspace{1em} + \I[\eM(s), \eB_1(s), m(s)=1, i(s)=j(s), \eB^c_{j(s)}(s)] \\
   &\hspace{2em}+ \I[\eM(s), \eB_1(s), m(s)=1, i(s)=j(s), \eB_{j(s)}(s)] \\
   & \hspace{1em} + \I[\eM(s), \eB_1(s), m(s)\ne 1, i(s)=m(s),  \eB^c_{m(s)}(s)]\\
   &\hspace{1em} + \I[\eM(s), \eB_1(s), m(s)\ne 1, i(s)=j(s),  \eA^c_{m(s)}(s), \eB^c_{j(s)}(s)] \\
   &\hspace{2em}+ \I[\eM(s), \eB_1(s), m(s)\ne 1, i(s)=j(s),  \eA^c_{m(s)}(s), \eB_{j(s)}(s)] \numberthis{\label{eq: star_2}} \\
   &= \I[\eM^c(s), \eB_1(s)] \\
   &\hspace{1em}+ \I[\eM(s), \eB_1(s), m(s)=1, i(s)=1, \eA^c_{j(s)}(s)] \\
   &\hspace{2em}+ \I[\eM(s), \eB_1(s), m(s)=1, i(s)=1, \eA_{j(s)}(s)] \\
   & \hspace{1em} + \I[\eM(s), \eB_1(s), m(s)=1, i(s)=j(s), \eB_{j(s)}(s)] \\
   & \hspace{1em} + \I[\eM(s), \eB_1(s), \eB^c_{i(s)}(s)] \\ 
   &\hspace{2em}+ \I[\eM(s), \eB_1(s), m(s)\ne 1, i(s)=j(s), \eA^c_{m(s)}(s), \eB_{j(s)}(s)] \\
   &= \I[s \in S_0(T)] + \I[s \in S_1(T)] + \I[s \in S_2(T)],
\end{align*}
where (\ref{eq: star_1}) and (\ref{eq: star_2}) hold from 
\begin{equation}\label{eq: 45eq}
    \I[m(s) \ne 1, \eB_1(s)] = \I[m(s) \ne 1, \eB_1(s), \eB_{m(s)}^c(s)] =  \I[m(s) \ne 1, \eB_1(s), \eA_{m(s)}^c(s)].
\end{equation}
The last equality holds from
\begin{align*}
    \I[&\eM(s), \eB_1(s), \eB^c_{i(s)}(s)] \\
    &= \I[\eM(s), \eB_1(s), \eB^c_{i(s)}(s), m(s)=1] + \I[\eM(s), \eB_1(s), \eB^c_{i(s)}(s), m(s)\ne 1] \\
    &= \I[\eM(s), \eB_1(s), \eB^c_{i(s)}(s), m(s)=1, i(s)=j(s)]  \\
    &\hspace{1em}+ \I[\eM(s), \eB_1(s), \eB^c_{i(s)}(s), m(s)\ne 1, i(s)=m(s)] \\
    &\hspace{2em}+ \I[\eM(s), \eB_1(s), \eB^c_{i(s)}(s), m(s)\ne 1, i(s)=j(s)] \\
    &= \I[\eM(s), \eB_1(s), m(s)=1, i(s)=j(s), \eB^c_{j(s)}(s)]  \\
    &\hspace{1em}+ \I[\eM(s), \eB_1(s), m(s)\ne 1, i(s)=m(s),  \eB^c_{m(s)}(s)] \\
    &\hspace{2em}+ \I[\eM(s), \eB_1(s), m(s)\ne 1, i(s)=j(s), \eA^c_{m(s)}(s), \eB^c_{j(s)}(s)] \numberthis{\label{eq: star_3}},
\end{align*}
where we used (\ref{eq: 45eq}) in (\ref{eq: star_3}) again.
This implies that if $T \geq T_M$, then $[T_M, T]\cap \mathbb{N} =  S_0(T) \cup S_1(T) \cup S_2(T) $ holds.
Note that if $s = T_M \geq KD$, there exists at least one arm $a \in [K]$ satisfying $N_a(s) \geq D$.

\paragraph{(1) If $N_1(s) \geq D$}
Recall the definition $T_{S_1} = \inf S_1(T)$ and $T_{S_2} = \inf S_2(T)$, which implies the first round when the events in Lemmas~\ref{lem: OptGood} and~\ref{lem: SubGood} occur, respectively.

\paragraph{(a) $S_0(T)$ is a subinterval}
If $S_0(T)$ consists of consecutive natural numbers, i.e., the subinterval in $[T_M, T]\cap \mathbb{N}$, then $\min(T_{S_1},T_{S_2}) \leq T_S + 1$ holds since we can only observe events in $S_1(T)$ or $S_2(T)$ for $s > T_S$.

\paragraph{(b) $S_0(T)$ is not a subinterval}
If $S_0(T)$ is not a subinterval of $[T_M, T]\cap \mathbb{N}$, this directly implies that $\min(T_{S_1},T_{S_2}) \leq T_S$ from $[T_M, T]\cap \mathbb{N} =  S_0(T) \cup S_1(T) \cup S_2(T)$.

\paragraph{(a+b)}
Therefore, we have $\min(T_{S_1},T_{S_2}) \leq T_S + 1$.

\paragraph{(1-i) If $T_{S_2} \leq T_{S_1}$}
By definition of $T_{S_2}$, $T_B \leq T_{S_2} \leq T_S +1$ can be directly derived from Lemma~\ref{lem: SubGood} with the assumption $N_1(T_{S_2}) \geq D$.

\paragraph{(1-ii) If $T_{S_1} \leq T_{S_2}$}
By Lemma 7, whenever $s \in S_2(T)$, we have $T_B \leq s$.
Therefore, $T_B$ increases only during rounds in $S_0(T)\cup S_1(T)$, and we immediately obtain $T_B \leq s$ if $s\in S_2(T)$ holds.

Consider $s\in S_1(T)$.
If $N_{j(s)}(s) \geq \frac{D_1}{\ud_{j(s)}}$, then
\begin{equation}\label{eq: Da1}
    s f_{j(s)}(\bw^s; \hmu(s)) \geq N_{j(s)}(s) d(\hmu_{j(s)}(s), \hmu_{1,j(s)}(s)) \geq D_1
\end{equation}
since for $s\in S_1(T)$, BC-TE implies
\begin{equation*}
    d(\hat{\mu}_j, \hat{\mu}_{1,j}) \geq d(\hat{\mu}_1, \hat{\mu}_{1,j}),
\end{equation*}
and by definition of $\ud$ in (\ref{def: under_d}) we have $d(\hmu_{j(s)}(s), \hmu_{1,j(s)}(s)) \geq \ud_{j(s)}$.
This implies that $\eB_a(t)$ holds for all $a$ and $t \geq s$, i.e., $T_B \leq s$.

Therefore, the worst case that maximizes $T_B$ is when only events in $S_0(T)$ and $S_1(T)$ occurs, with no $S_2(T)$, and every $s\in S_1(T)$ satisfies $N_{j(s)} < D_1/ \ud_{j(s)}$.
Since $i(s)=j(s)$ for $s\in S_1(T)$, which increases $N_{j(s)}(s)$ each time, such $s$ can occur at most $(K-1)D$ times.

\paragraph{(1-iii) Summary}
In all cases, we obtain
\begin{equation*}
    T_B \leq T_S + (K-1)D + 1,
\end{equation*}
where $T_S = T_M + |S_0(T)| = \max(T_C, KD) + |S_0(T)|$.

\paragraph{(2) If $N_i(s) \geq D$ for $i\ne 1$}
From (1), one can expect that $T_B$ will be bounded at least if either $N_{j(s)}(s)$ or $N_1(s)$ satisfies the condition in (\ref{eq: Da1}) for any $s \leq T$.

\paragraph{(2-i) $j(s)=i$ holds for some $s \in S_1(T)$}
In this case, we have for $a \ne 1,i$
\begin{equation*}
     N_1(s)d(\hmu_1(s), \hmu_{1,i}(s)) + N_i(s)d(\hmu_i(s), \hmu_{1,i}(s)) = sf_{i} < sf_a \leq N_a(s) d(\hmu_a(s), \hmu_1(s)),
\end{equation*}
where we denote $\hmu_{1,i}^{\bw^s}(s)$ by $\hmu_{1,i}(s)$ for notational simplicity. 
From $N_i(s) \geq D$,
\begin{equation}\label{eq: DNai}
    \max_{a\in [K]} D_a \leq N_i(s)d(\hmu_i(s), \hmu_{1,i}(s)) \leq \min_{a\ne 1} N_a(s)d(\hmu_a(s), \hmu_1(s)),
\end{equation}
which implies $T_B \leq s$.

\paragraph{(2-ii) $j(s) \ne a$ holds for all $ s \in S_1(T)$}
Take arbitrary $t' \in (T_M, \infty) \cap \mathbb{N}$ and assume that there exists an arm $j' \ne 1$ and a round $s' \geq t'$ such that $\I[\eB^c_{j'}(s')]=1$ holds.
Note that whenever $N_{j(s)}(s) \geq D$ holds, substituting $a=j(s)$ in  (\ref{eq: DNai}) leads to the same inequality, which implies $T_B \leq s$.

\paragraph{(2-iii) Summary}
Therefore, for all $j\ne 1$, $\sum_{s\in S_1(T)}\I[j(s)=j] \leq D$ should hold since $\sum_{s\in S_1(T)}\I[j(s) =j] > D$ admits the existence of $s\in S_1(T)$ such that satisfies (\ref{eq: DNai}), which contradicts to the assumption of the existence of such $s'$.
In other words, $\sum_{s\in S_1(T)}\I[j(s)=j] \leq D$ is a necessary condition to satisfy the assumption of the existence of $j'$ and $s'$ satisfying $\I[\eB_{j'}^c(s')]=1$.
From the definition of $S_1(T)$, for any $s \in S_1(T)$, $N_{j(s)}(s+1) = N_{j(s)}(s)+1$ holds.
Hence, at worst, if $|S_1(T) \cap [T_M, t')| \geq (K-2)D$ holds at some round $t'$, there exists $s \in S_1(T) \cap [T_M, t')$ such that $N_{j(s)}(s) \geq D$.
Therefore, $T_B$ is at most the round until $S_1(T)$ occur $(K-2)D$ times.

Similarly, if the event in $S_2(T)$ occurs $D$ times at some round $t''$, then $N_1(t'') \geq D$ holds from the sampling rule. 
This implies that $B_i(s)$ holds for all $i \in [K]$ for $s\geq t''$ from (\ref{eq: Da1}), i.e., $T_B$ is at most the round until $S_2(T)$ occur $D$ times.

\paragraph{(3) Conclusion}
In summary, we have $[T_M, T]\cap \mathbb{N}= S_0(T)\cup S_1(T)\cup S_2(T)$ and there exists an arm $i$ satisfying $N_i(t) \geq D$.
If $N_1(s) \geq D$, then $T_B \leq T_S + (K-1)D +1$ holds.
If $N_i(s)\geq D$ holds for $i\ne 1$, then $T_B$ is at most the round $s$
after the event in $S_1(T)$ occurs $(K-2)D$ times or $s_{S_2, D}$ when the event in $S_2(T)$ occur $D$ times.
Hence, we have
\begin{equation*}
    T_B \leq T_S + (K-2)D + D +1 = T_S + (K-1)D +1,
\end{equation*}
where $T_S = T_M + |S_0(T)| = \max(T_C, KD) + |S_0(T)|$.
Then, we have
\begin{align*}
    \mathbb{E}[T_B] &\leq \mathbb{E}[T_S] + (K-1)\mathbb{E}[D] + 1 \\
    &\leq \mathbb{E}[T_C] + (2K-1)\mathbb{E}\left[\sup_{i\ne 1}\sup_{s\geq t} \I[B_i^c(s)] N_i(s)d(\hmu_i(s), \hmu_1(s))\right] \\
    &  \hspace{1em}+ \mathbb{E}\Bigg[ \sum_{t=T_M}^T  \I[\eM^c(t)]  + \I[m(t)=1, i(t)=1, \eB_1(t), \eA_{j(t)}^c(t), \eM(t)]
    \\ 
    & \hspace{3em} + \I[m(t)\ne 1, i(t)=j(t), \eB_1(t), \eA_{m(t)}^c(t), \eB_{j(t)}(t), \eM(t)] 
   \\
   & \hspace{5em}+ \I[\eB_1(t), \eB_{i(t)}^c(t), \eM(t)] \Bigg] + 1.
\end{align*}
Then, the following five lemmas conclude the proofs.
\end{proof}

\begin{lemma}\label{lem: Dibound}
For a bounded region of parameters $R \subset \mathbb{R}$, it holds that for arbitrary $\mu' \in R$ and $i\in [K]$
\begin{equation*}
    \mathbb{E}\left[ \sup_{n\in \mathbb{N}, \mu' \in R} \mathbbm{1}[|\hmu_{i,n} - \mu_i | \geq \epsilon ] nd(\hmu_{i,n}, \mu') \right] = \mathcal{O}\left(d_\eps^{-1}\right),
\end{equation*}
where $\hmu_{i,n}$ is the empirical mean reward of the arm $i$ when it is played $n$ times.
\end{lemma}
Here, note that $\hmu_{i,n}$ is different from $\hmu_{a,b}(t)$ that denotes the weighted average of their empirical mean.
Lemma~\ref{lem: Dibound} provides the finiteness of the expectation of $D_i$ for any $i \in [K]$.
\begin{lemma}\label{lem: AcB}
For the finite number of arms $K$ and any $T \in \mathbb{N}$, it holds that
\begin{align*}
    \mathbb{E}\left[\sum_{t=1}^T \mathbbm{1}\left[ m(t)=1, i(t)=1,\eB_1(t), \eA_{j(t)}^c(t), \eM(t) \right]\right] &\leq \mathcal{O}\left(Kd_\eps^{-1}\right), \\
    \mathbb{E}\left[\sum_{t=1}^T \mathbbm{1}\left[ i(t)=j(t), \eA_{m(t)}^c(t), \eB_{j(t)}(t), \eM(t) \right] \right] &\leq \mathcal{O}\left(K^2d_\eps^{-1}\right). 
\end{align*}
\end{lemma}
\begin{lemma}\label{lem: wrongit}
For the finite number of arms $K$ and any $T \in \mathbb{N}$, it holds that
\begin{equation*}
    \mathbb{E}\left[\sum_{t=1}^T \mathbbm{1}\left[ \eB_{i(t)}^c(t), \eM(t) \right]\right] \leq \mathcal{O}\left(Kd_\eps^{-1}\right).
\end{equation*}
\end{lemma}
The proofs of Lemmas~\ref{lem: Dibound}--\ref{lem: wrongit} are provided in Section~\ref{sec: bai_TBlem_simple}.
\begin{lemma}\label{lem: TSexplore}
For the finite number of arms $K$ and any $T \in \mathbb{N}$, it holds that
\begin{equation*}
    \mathbb{E}\left[\sum_{t=1}^T \mathbbm{1}[\eM^c(t)]\right] \leq \mathcal{O}\left(K^2d_\eps^{-2} \right).
\end{equation*}
\end{lemma}
The proof of Lemma~\ref{lem: TSexplore} is given in Section~\ref{sec: bai_TBlem_UBTE}.
\begin{lemma}\label{lem: PCsum}
Under Algorithm~\ref{alg: BCTE}, it holds for any $\eps \in \left( 0, \frac{\mu_1-\mu_2}{2} \right)$ that
\begin{equation*}
    \mathbb{E}[T_C] \leq C(\pi_{\mathrm{j}}, \bmu,\eps) + 4 d_\eps^{-3},
\end{equation*}
where $ C(\pi_{\mathrm{j}}, \bmu, \eps)$ specified in Lemma~\ref{lem: N1control}.
\end{lemma}
The proof of Lemma~\ref{lem: PCsum} is given in Section~\ref{sec: bai_TBlem_TS}, where we adapt the analysis in \citet{KordaTS} to our problem.

\subsection{Proofs of technical lemmas for Theorem~\ref{thm: bai_TB}: Sufficient conditions for the convergence of estimates}\label{sec: bai_TBlem}
Here, we provide the proof of Lemmas~\ref{lem: OptGood} and~\ref{lem: SubGood}.

\begin{proof}[Proof of Lemma~\ref{lem: OptGood}]
    Since $i(t)=j$ implies
\begin{align*}
d(\hat{\mu}_j(t),\hat{\mu}_{1,j}(t)) \geq d(\hat{\mu}_1(t),\hat{\mu}_{1,j}(t)),
\end{align*}
we have
\begin{align*}
d(\hat{\mu}_j(t),\hat{\mu}_{1,j}(t)) \geq \underline{d}_j,
\end{align*}
from the definition of $\underline{d}_j$ in (\ref{def: under_d}).

Then, we have
\begin{align*}
tf_j(\bw^t,\hat{\bm{\mu}}(t)) &= N_1(t) d(\hat{\mu}_1(t),\hat{\mu}_{1,j}(t)) + N_j(t) d(\hat{\mu}_j(t),\hat{\mu}_{1,j}(t)) \\
&\geq N_j(t) \underline{d}_j > D_1
\end{align*}
On the other hand, if $|\hat{\mu}_1(t)-\mu_1|\ge \epsilon$ and $|\hat{\mu}_j(t)-\mu_j|\le \epsilon$, then
\begin{align*}
tf_j(\bw^t,\hat{\bm{\mu}}(t)) \le N_1(t)d(\hmu_1(t), \hmu_j(t)) \le  D_1
\end{align*}
by the definition of $D_1 = \sup_{i \ne 1} D_{i}$.
Therefore, $|\hat{\mu}_1(t)-\mu_1|\ge \epsilon$ cannot hold.

Under $|\hat{\mu}_1(t)-\mu_1|\le \epsilon$ and $|\hat{\mu}_j(t)-\mu_j|\le \epsilon$, we see that
\begin{align*}
N_j(t) \underline{d}_j \le tf_j(\bw^t,\hat{\bm{\mu}}(t)) &\le N_1(t) d(\hat{\mu}_1(t),\hat{\mu}_{j}(t)) \\
&\le N_1(t)d(\mu_1+\epsilon,\mu_j-\epsilon),
\end{align*}
which completes the proof.
\end{proof}

\begin{proof}[Proof of Lemma~\ref{lem: SubGood}]
    Since $j(t)=\argmin_{i \ne m(t)} tf_i(\bw^t,\hat{\bm{\mu}}(t))$ and $i(t)=1$, it holds for all $i\neq 1$ that
\begin{align*}
tf_i(\bw^t,\hat{\bm{\mu}}(t))
\ge
tf_{j(t)}(\bw^t,\hat{\bm{\mu}}(t))
\end{align*}
and
\begin{align*}
d(\hat{\mu}_1(t),\hat{\mu}_{1,j(t)}(t)) \ge d(\hat{\mu}_{j(t)}(t),\hat{\mu}_{1,j(t)}(t)).
\end{align*}
Then, we can use the same argument as Lemma \ref{lem: OptGood} by exchanging the roles of $1$ and~$j$.
\end{proof}

\subsection{Proofs of technical lemmas for Theorem~\ref{thm: bai_TB}: Boundedness of the number of rounds where estimates do not converge}\label{sec: bai_TBlem_simple}
Here, we provide the proof of Lemmas~\ref{lem: Dibound}--\ref{lem: wrongit}.
Firstly, to prove Lemma~\ref{lem: Dibound}, we require the lemma below, whose proof is postponed to Section~\ref{sec: bai_lemKLdiv}.
\begin{lemma}\label{lem: KLdiv}
Let $R \subset \mathbb{R}$ be a bounded region of parameters and fix arbitrary $\mu_0$.
Then, there exists $a,b \geq 0$ such that
\begin{equation*}
    d(\mu, \mu') \leq ad(\mu, \mu_0) + b
\end{equation*}
for arbitrary $\mu \in \mathbb{R}$ and $\mu' \in R$.
\end{lemma}

\begin{proof}[Proof of Lemma~\ref{lem: Dibound}]
    Let $P(z) := \mathbb{P}[d(\hmu_{i,n}, \mu_i ) \geq z]$.
Then, by Chernoff bound, we have $P(z) \leq 2e^{-nz}$.
Therefore,
\begin{align*}
    \mathbb{E}\left[\mathbbm{1}[|\hmu_{i,n} - \mu_i| \geq \epsilon]\sup_{\mu' \in R} d(\hmu_{i,n}, \mu') \right] &\leq \mathbb{E}[\mathbbm{1}[|\hmu_{i,n} - \mu_i | \geq \epsilon] (ad(\hmu_{i,n}, \mu_i)+b)] \\
    &\hspace{-3em}\leq 2be^{-nd_{\epsilon}} + a \int_{d_\epsilon}^\infty z\mathrm{d}(-P(z)) \\
    &\hspace{-3em}= 2be^{-nd_{\epsilon}} + a\left( -[zP(z)]_{d_\epsilon}^\infty + \int_{d_\epsilon}^\infty zP(z)\mathrm{d}z \right) \\
    &\hspace{-3em}\leq 2be^{-nd_{\epsilon}} + 2ad_{\epsilon}e^{-nd_{\epsilon}} + a\int_{d_\epsilon}^\infty zP(z)\mathrm{d}z \\
    &\hspace{-3em}\leq  2be^{-nd_{\epsilon}} + 2ad_{\epsilon}e^{-nd_{\epsilon}} + 2a \left[ -\frac{ze^{-nz}}{n}- \frac{e^{-nz}}{n^2} \right]_{d_{\epsilon}}^\infty \\
    &\hspace{-3em}\leq 2\left( b + a \left( d_\epsilon + \frac{d_\epsilon}{n} + \frac{1}{n^2} \right)\right)e^{-nd_{\epsilon}},
\end{align*}
where $d_\epsilon := \min_{i\in[K]} \{d(\mu_i -\epsilon, \mu_i), d(\mu_i + \epsilon, \mu_i) \}$ and the first inequality holds from Lemma~\ref{lem: KLdiv}.
Since this quality decays exponentially in $n$, it is straightforward that
\begin{align*}
    \mathbb{E}\Bigg[ \sup_{n\in \mathbb{N}, \mu' \in R} \mathbbm{1}[|\hmu_{i,n} - \mu_i | \geq \epsilon ] nd(\hmu_{i,n}, \mu') \Bigg] &\leq \sum_{n=1}^\infty \mathbb{E} \left[ \mathbbm{1}[|\hmu_{i,n} - \mu_i | \geq \epsilon ] \sup_{\mu' \in A} d(\hmu_{i,n}, \mu') \right] \\
    &= \mathcal{O}(d_\eps^{-1}). \qedhere
\end{align*}
\end{proof}

\begin{proof}[Proof of Lemma~\ref{lem: AcB}]
    For $j(t)=j$, we first consider
\begin{equation*}
    D_j = \sup_{t} \left\{ \mathbbm{1}[ | \hmu_{j}(t) - \mu_i | \, \geq \epsilon] N_j(t) d(\hmu_j(t), \hmu_{1}(t)) \right\}.
\end{equation*}
Note that on $\eB_1(t)$, $\hmu_{1}(t) \in [\mu_1 - \epsilon, \mu_1 + \epsilon]$ is bounded so that we can apply Lemmas~\ref{lem: Dibound} and~\ref{lem: KLdiv}.
We first show the existence of a bounded constant $c_j^*\in \mathbb{R}_{+}$ such that
\begin{equation*}
    N_1(t) \leq c_j^* D_j,
\end{equation*}
where
\begin{equation*}
    c_j^* = \min\left(c_j, \frac{x_j'}{d_{\zeta}}\right)
\end{equation*} 
for constants $c_j$, $x_j'$ and $d_{\zeta}$ that depend on models. 

\paragraph{(1) When $\hmu_j(t) \not\approx \hmu_{m(t)}(t)$}
From their definitions, we have
\begin{align*}
    0 \leq N_{j}(t) d(\hmu_i(t), \hmu_{1,j}(t)(t)) \leq N_{j}(t) d(\hmu_{j}(t), \hmu_{1}(t)) \leq D_i
\end{align*}
and
\begin{align*}
  N_1(t) d(\hmu_1(t), \hmu_{1,j}(t)) &\leq N_1(t) d(\hmu_1(t), \hmu_{1,j}(t))+  N_{j}(t) d(\hmu_i(t), \hmu_{1,j}(t)) \\
  &= tg(\bw^t; \hbmu(t)).
\end{align*}
Let us consider
\begin{equation*}
    \psi(x;t) = xd(\hmu_{m(t)}(t), \hmu_{m(t),j}(x;t)) + d(\hmu_{j}(t), \hmu_{m(t),j}(x;t)),
\end{equation*}
where $\hmu_{a,b}(x;t) =\frac{x\hmu_a(t) + \hmu_b(t)}{x+1}$.
One can see that $\psi(x;t)$ is strictly increasing with respect to $x$ since $\psi'(x;t) = d(\hmu_{m(t)}(t), \hmu_{m(t),j}(x;t)) >0$ and it tends to $d(\hmu_{j}(t), \hmu_{m(t)}(t))$ when $x$ goes to infinity~\citep{garivier2016optimal}.
Then, under the condition $\{m(t)=1, j(t)=j\}$, it holds that
\begin{align*}
    tg(\bw^t; \hbmu(t)) = N_{j}(t) \psi\left(\frac{N_{1}(t)}{N_{j}(t)};t\right) &\leq N_{j}(t) d(\hmu_{j}(t), \hmu_{1}(t)) \\
    &\leq D_{j}.
\end{align*}
Therefore,
\begin{equation*}
     N_1(t) \leq \frac{1}{d(\hmu_1(t), \hmu_{1,j}(t))} D_{j}.
\end{equation*}
Note that there exists a constant $c_j$ such that $\frac{1}{d(\hmu_1(t), \hmu_{1,j}(t))} \leq c_j <  \infty$ when $\hmu_a(t) \not\approx \hmu_{m(t)}(t)$, which shows the existence of $c_j^*$.

\paragraph{(2) When $\hmu_j(t) \approx \hmu_{m(t)}(t)$}
Here, $i(t)=1$ implies that
\begin{equation} \label{eq: KLz}
d\left(\hmu_1(t), \hmu_{1,j}^{\bw^t}(t)\right)
\geq 
d\left(\hmu_j(t), \hmu_{1,j}^{\bw^t}(t)\right).
\end{equation}
Note that as $\frac{w_1(t)}{w_j(t)}$ increases, RHS of (\ref{eq: KLz}) decreases and LHS of (\ref{eq: KLz}) increases simultaneously.
Therefore,
\begin{equation*}
    \forall t \in \mathbb{N}, \, \exists\,  x_{j,t}^* \in \mathbb{R}_{+} \text{ s.t. } \frac{w_1(t)}{w_j(t)} = x_{j,t}^* \Leftrightarrow  
    d(\hmu_1(t), \hmu_{1,j}^{\bw^t}(t)) = d(\hmu_j(t), \hmu_{1,j}^{\bw^t}(t)).
\end{equation*}
Note that $x_{j,t}^*$ depends on the distribution of reward and history $H_t$ until round $t$, e.g., $\forall t \in \mathbb{N}$, $x_{j,t}^* =1$ for the Gaussian distribution.
Since $\hmu_1(t)$ is bounded under $\{\eB_1(t)\}$ and $\hmu_j(t) \in (\mu_j + \epsilon, \hmu_1(t)]\subset (\mu_j + \eps, \mu_1 + \eps]$ holds under $\{\eB_1(t), \eA^c_j(t), m(t)=1\}$, there exists $x_j' \in \mathbb{R}_{+}$ such that for any $t\in \mathbb{N}$
\begin{equation*}
    N_1(t) > x_j' N_j(t) \implies d(\hmu_1(t), \hmu_{1,j}(t))  < d(\hmu_j(t), \hmu_{1,j}(t)) \text{, i.e., } i(t)=j.
\end{equation*}
Let consider a bounded region $R = [\mu_1 - \epsilon, \mu_1 + \epsilon] \subset \mathbb{R}$ and a random variable
\begin{equation*}
    D_j = \sup_{t\in \mathbb{N}} \sup_{\mu' \in A} \left\{ \mathbbm{1}[ | \hmu_j(t) - \mu_j | \, \geq \epsilon] N_j(t) d(\hmu_j(t), \mu') \right\}, \quad j \in [K]\setminus \{1 \}.
\end{equation*}
Since $m(t)=1$ holds under the condition, we have
\begin{equation*}
    \sup_{\mu' \in A}d(\hmu_j(t), \mu') = \max \{ d(\hmu_j(t), \mu_1 -\epsilon), d(\hmu_j(t), \mu_1 + \epsilon) \} 
\end{equation*}
and $\hmu_1(t) > \hmu_j(t)$.
Let $\zeta(\eps) \in A$ be a point such that $d(\zeta, \mu_1 -\eps) = d(\zeta, \mu_1 + \eps) = d_{\zeta}$.
Then, it holds that
\begin{equation*}
    \sup_{\mu' \in A}d(\hmu_j(t), \mu') > d_{\zeta}.
\end{equation*}
Note that $d_{\zeta}$ and $x_j'$ only depend on the models.
Therefore, there exists a constant $c_j^* \in \mathbb{R}_{+}$ such that
\begin{equation*}
    N_1(t) \le \frac{x_j'}{d_{\zeta}} D_j \leq c_j^* D_j.
\end{equation*}

\paragraph{(3) Conclusion}
From Lemma~\ref{lem: Dibound}, we obtain
\begin{align*}
    \mathbb{E}\Bigg[\sum_{i\in [K] \setminus \{ 1\}}\sum_{t=1}^\tau \mathbbm{1}\bigg[ m(t)=1,& i(t)=1, \eB_1(t), j(t)=i, \eA_{j(t)}^c(t), \eM(t) \bigg]\Bigg]  \\ 
    &\leq \mathbb{E}\left[\sum_{i\in [K] \setminus \{ 1\}} \sum_{t=1}^\infty \mathbbm{1}[i(t)=1, N_1(t) \leq c_i^* D_i] \right] 
    \\ 
    &\leq \sum_{i\in [K] \setminus \{ 1\}} c_j^* \mathbb{E}[D_j] \leq \mathcal{O}(Kd_\eps^{-1}),
\end{align*}
which concludes the first case.

Similarly, the second case can be bounded by considering $R_j = [\mu_j -\epsilon, \mu_j + \epsilon]$ and 
\begin{equation*}
     D_{m(t),j} = \sup_n \sup_{\mu' \in R_j} \{ \mathbbm{1}[ | \hmu_{m(t)}(n) - \mu_{m(t)}| \geq \epsilon  ] n d(\hmu_{m(t)}(n), \mu')\}
\end{equation*}
for every $m(t) \in [K]$ and $j\in [K] \setminus \{ m(t)\}$.
Since $\hmu_{j}(t) \in R_j$ holds under $\{ B_{j}(t)\}$, we can apply Lemmas~\ref{lem: Dibound} and~\ref{lem: KLdiv} by exchanging the role of $m(t)$ and $j$, which concludes the proof. 
\end{proof}

\begin{proof}[Proof of Lemma~\ref{lem: wrongit}]
    From the Chernoff bound, it holds for any arm $i \in [K]$ that
\begin{equation}\label{eq:Chernoff}
    \mathbb{P}[|\hmu_{i}(t) - \mu_i| \geq \epsilon | N_i(t) = n] \leq 2e^{-n d_\epsilon},
\end{equation}
where $d_\epsilon$ is defined in (\ref{def: deps}).
One can rewrite the expectation as
\begin{align*}
    \mathbb{E}\left[\sum_{t=1}^T \mathbbm{1}\left[ \eB_{i(t)}^c(t), \eM(t) \right]\right] &= \mathbb{E}\left[\sum_{i=1}^K \sum_{t=1}^{T} \sum_{n=1}^{\infty} \mathbbm{1}\left[i(t)=i, \eB_{i(t)}^c(t), \eM(t), N_{i(t)}(t)=n\right]\right]. \\
    &= \mathbb{E}\left[\sum_{i=1}^K \sum_{t=1}^{T} \sum_{n=1}^{\infty} \mathbbm{1}\left[i(t)=i, \eB_{i}^c(t),\eM(t),  N_{i}(t)=n\right]\right]
\end{align*}
For every arm $i\in[K]$, an event $\{i(t)=i, N_{i}(t) = n\}$ could happen at most once for any $n \in \mathbb{N}$.
Therefore, by applying (\ref{eq:Chernoff}), one has
\begin{equation*}
    \mathbb{E}\left[\sum_{t=1}^T \mathbbm{1}\left[ \eB_{i(t)}^c(t), \eM(t) \right]\right] \leq \sum_{i=1}^K \sum_{n=1}^\infty 2e^{-n d_{\epsilon}} \leq \mathcal{O}(Kd_\eps^{-1}),
\end{equation*}
which concludes the proof.
\end{proof}

\subsection{Proof of technical lemma for Theorem~\ref{thm: bai_TB}: An upper bound on the number of rounds where TE occurs}\label{sec: bai_TBlem_UBTE}
Here, we provide the proof of Lemma~\ref{lem: TSexplore}, which shows that the expected number of rounds where Thompson samples and the empirical mean estimates disagree is finite.
Before beginning the proof, we present the posterior concentration result when we employ the Jeffreys prior in the SPEF.
\begin{lemma}[Theorem 4 in~\citet{KordaTS}] \label{lem: Kordath4}
For the Jeffreys prior and $d_\eps$ defined in (\ref{def: deps}), there exists constants $C_{1,a}=C_1(\theta_a,A) > 0$, $C_{2,a}=C_2(\theta_a, A, \epsilon)>0$ and $N(\theta_a, A)$ such that for any $N_a(t) \geq  N(\theta_a, A)$,
\begin{equation*}
    \mathbbm{1}[\eB_a(t)]\mathbb{P}[\etB_a^c(t) | X_{a,N_a(t)}] 
    \leq 
    2C_{1,a}N_a(t)e^{-(N_a(t)-1)(1-\epsilon C_{2,a} )d_\eps}
\end{equation*}
whenever $\epsilon$ is such that $1-\epsilon C_{2,a}(\epsilon) > 0$.
Note that $A$ is a convex function in (\ref{eq: bai_SPEF_form}).
\end{lemma}

\begin{proof}[Proof of Lemma~\ref{lem: TSexplore}]
    Let us define $L(\theta) := \frac{1}{2}\min(\sup_y p(y|\theta), 1)$ and an event
\begin{equation*}
    \tilde{E}_a(t) = \left(\exists 1 \leq s' \leq N_a(t) : p(x_{a,s'}|\theta_a) \geq L(\theta_a) , \;\middle|\; \frac{\sum_{s=1,s\ne s'}^{N_a(t)} x_{a,s}}{N_a(t)-1} - \mu_a \;\middle|\;\leq \epsilon \right).
\end{equation*}
Consider
\begin{align*}
    \sum_{t=1}^T \mathbbm{1}[\eM^c(t)] &= \sum_{t=1}^T \sum_{i\in [K]} \mathbbm{1}[i(t)=i, \eM^c(t)] \\
    &= \sum_{t=1}^T \sum_{i\in [K]} \mathbbm{1}[i(t)=i, \tilde{E}_a^c(t), \eM^c(t)] + \mathbbm{1}[i(t)=i, \tilde{E}_a(t), \eM^c(t)]
\end{align*}
It is shown by \citet{KordaTS} that 
\begin{align*}
   \mathbb{E}\left[\sum_{t=1}^T \mathbbm{1}[i(t)=i, \tilde{E}_i^c(t), \eM^c(t)]\right] 
   &\leq
   \sum_{t=1}^\infty \mathbb{P}(p(x_{i,1}| \theta_a) \leq L(\theta_a))^t + \sum_{t=1}^\infty 2te^{-(t-1)d_\epsilon}  \\
   &\leq 
   \mathcal{O}\left(d_\eps^{-2}\right). \numberthis{\label{eq: Kordaec}}
\end{align*}
Then, consider 
\begin{align*}
   \sum_{t=1}^T \mathbbm{1}[i(t)=i, \tilde{E}_i(t), \eM^c(t)]
  &=
   \sum_{t=1}^T \Bigg(\mathbbm{1}[i(t)=i, \etB_i(t), \tilde{E}_i(t), \eM^c(t)] \\
   &\hspace{5em}+ \mathbbm{1}[i(t)=i, \etB_i^c(t), \tilde{E}_i(t), \eM^c(t)] \Bigg).
\end{align*}
On $\tilde{E}_i(t)$, the following holds for a constant $N(\theta_i, A)$ from Lemma~\ref{lem: Kordath4}.
\begin{align*}
     \mathbb{E}\Bigg[\sum_{t=1}^T \sum_{i\in [K]} &\mathbbm{1}[i(t)=i, \etB_i^c(t), \tilde{E}_i(t), \eM^c(t)]\Bigg] \\
     &\hspace{-2em}\leq
     \sum_{i\in [K]}N(\theta_i, A) + \sum_{i\in [K]} \sum_{\substack{t:i(t)=i \\ N_a(t) \geq N(\theta_i, A)}}^T 2C_{1,i}e^{-(N_i(t)-1)(1-\epsilon C_{2,i} )d_\epsilon + \log (N_i(t))} \\
     &\hspace{-2em}\leq \sum_{i\in [K]} N(\theta_i, A) + \sum_{i\in [K]} \sum_{n=N(\theta_i, A)}^\infty  2C_{1,i}ne^{-(n-1)(1-\epsilon C_{2,i} )d_\epsilon} \\
     &\hspace{-2em}\leq
     \mathcal{O}\left(Kd_\eps^{-2}\right),
\end{align*}
where the second inequality holds since $N_i(t)$ increases when $\{i(t)=i\}$ happens. 

Finally, we will show that
\begin{equation*}
    \sum_{t=1}^T \sum_{i\in[K] }\mathbbm{1}[i(t)=i, \etB_i(t), \tilde{E}_i(t), \eM^c(t)] \leq \mathcal{O}\left(K^2d_\eps^{-2} \right).
\end{equation*}
On $\eM^c(t)$, $i(t) \in \{m(t), \tm(t) \}$ holds so that
\begin{align*}
    \sum_{t=1}^T \sum_{i\in[K] }\mathbbm{1}[i(t)=i, \etB_i(t),\tilde{E}_i(t), \eM^c(t)] 
    &\leq
    \sum_{t=1}^T \sum_{i\in[K] } \mathbbm{1}[i(t)=m(t)=i, \etB_{i}(t), \tilde{E}_{i}(t), \eM^c(t)] 
    \\
    & \hspace{1em} + \sum_{t=1}^T \sum_{i\in[K] } \mathbbm{1}[i(t)= \tm(t)= i, \etB_{i}(t), \tilde{E}_{i}(t), \eM^c(t)].
\end{align*}
Let us define $N_A = \max_{a\in [K]}N(\theta_a, A)$.
For any $i \in [K]$, we have
\begin{multline*}
    \sum_{t=1}^T \mathbbm{1}[i(t)=m(t)=i, \etB_{i}(t), \tilde{E}_{i}(t), \eM^c(t)] 
    \\ \leq 
    N_A + \sum_{t=1}^T \mathbbm{1}[i(t)=m(t)=i, \etB_{i}(t), \tilde{E}_{i}(t), \eM^c(t), N_i(t) \geq N_A]
\end{multline*}
and
\begin{multline*}
    \sum_{t=1}^T \mathbbm{1}[i(t)= \tm(t)= i, \etB_{i}(t), \tilde{E}_{i}(t), \eM^c(t)] 
    \\ 
    \leq
    N_A + \sum_{t=1}^T \mathbbm{1}[i(t)=\tm(t)=i, \etB_{i}(t), \tilde{E}_{i}(t), \eM^c(t), N_i(t) \geq N_A].
\end{multline*}
Consider
\begin{multline*}
    \mathbbm{1}[i(t)=m(t)=i, \etB_{i}(t), \tilde{E}_{i}(t), \eM^c(t), N_i(t) \geq N_A]
    = \\
    \sum_{j \in [K]\setminus \{i\}} \underbrace{\mathbbm{1}[i(t)=m(t)=i, \etB_{i}(t), \tilde{E}_{i}(t), \eM^c(t), N_i(t) \geq N_A,  \tm(t)=j,\tilde{E}_{j}(t) ] }_{(\divideontimes)}
    \\ + 
    \underbrace{\mathbbm{1}[i(t)=m(t)=i, \etB_{i}(t), \tilde{E}_{i}(t), \eM^c(t), N_i(t) \geq N_A,  \tm(t)=j, \tilde{E}_{j}^c(t) ]}_{(\star)}.
\end{multline*}
Similarly to (\ref{eq: Kordaec}), it holds that $\mathbb{E}\left[\sum_t (\star)\right] \leq \mathcal{O}\left(d_\eps^{-2}\right)$.
On $\eM^c(t)$, $\{i(t)=m(t)\}$ implies that $\{ N_{m(t)}(t) \leq N_{\tm(t)}(t) \}$, i.e., $N_j(t) \geq N_i(t) \geq N_A$ so that one can apply Lemma~\ref{lem: Kordath4}.
Hence,
\begin{align*}
\sum_t\mathbb{E}[(\divideontimes)] &\leq \mathcal{O}(d_\eps^{-2}) 
+ \sum_t \mathbb{E}\bigg[\mathbbm{1}[i(t)=m(t)=i, \etB_{i}(t), \tilde{E}_{i}(t)] \\ 
&\hspace{9em} \cdot\I[\eM^c(t), N_i(t) \geq N_A,  \tm(t)=j,\tilde{E}_{j}(t), \etB_j(t) ]\bigg].
\end{align*}
From its definition, on $\tilde{E}_i(t)$, the empirical mean reward of arm $i$ is well concentrated around its true mean.
Thus, 
\begin{equation*}
    m(t)=i, \tilde{E}_i(t), \tilde{E}_j(t)
    \implies
    i > j.
\end{equation*}
However, on $\{ \etB_i(t), \etB_j(t), \tm(t)=j \}$, $i < j$ holds, which is a contradiction.
Therefore,
\begin{equation*}
    \mathbbm{1}[i(t)=m(t)=i, \etB_{i}(t), \tilde{E}_{i}(t), \eM^c(t), N_i(t) \geq N_A,  \tm(t)=j,\tilde{E}_{j}(t), \etB_j(t) ] = 0,
\end{equation*}
which leads to
\begin{equation*}
    \mathbb{E}\left[\sum_{t=1}^T \mathbbm{1}[\eM^c(t)]\right] 
    = 
    \mathcal{O}\left(K^2d_\eps^{-2} \right). \qedhere
\end{equation*}
\end{proof}

\subsection{Proof of technical lemma for Theorem~\ref{thm: bai_TB}: Analysis with TS}\label{sec: bai_TBlem_TS}
Here, we provide the proof of Lemma~\ref{lem: PCsum}.

\begin{proof}[Proof of Lemma~\ref{lem: PCsum}]
    Let us define an event
\begin{equation*}
    \eC(t) := \bigcup_{s=t}^\infty \{ \eB_1^c(s) \}
\end{equation*}
so that $\eC^c(t) =\bigcap_{s=t}^\infty \{ \eB_1(s) \}$ implies only $\eB_1(s)$ occurs for $s \geq t$, meaning that $\eC(t) \Leftrightarrow \{ T_C \geq t\}$.
Therefore.
\begin{align*}
    \mathbb{E}[T_C] = \sum_{s=1}^{\infty} \mathbb{P}[T_C \geq s] &= \sum_{s=1}^{\infty}\mathbb{P}[\eC(s)] \\
    &=\sum_{s=1}^{\infty}\mathbb{P}[\eC(s), N_1(s) \leq \sqrt{s}] + \mathbb{P}[\eC(s), N_1(s) \geq \sqrt{s}].
\end{align*}
From the Chernoff bound, we can derive the upper bound of the second term as
\begin{align*}
    \sum_{s=1}^{\infty} \mathbb{P}[\eC(s), N_1(s) \geq \sqrt{s}] &\leq  \sum_{s=1}^{\infty} \sum_{n=\sqrt{s}}^\infty \mathbb{P}[|\hmu_{1,n} - \mu_1 | \geq \eps] \\
    &\leq\sum_{s=1}^{\infty} \sum_{n=\sqrt{s}}^\infty 2e^{-nd_{\eps}} \\
    &\leq \sum_{s=1}^{\infty} \frac{2}{d_\eps} e^{-\sqrt{s} d_\eps} \\
    &\leq \frac{2}{d_\eps} \int_0^\infty  e^{-\sqrt{s} d_\eps} \dx s = \frac{2}{d_\eps} \int_0^\infty 2xe^{-d_\eps x} \dx x\\
    &= 4 d_\eps^{-3}.
\end{align*} 
Then, the Lemma~\ref{lem: N1control} below concludes the proof. 
\end{proof}

\begin{lemma}\label{lem: N1control}
For the finite number of arms $K< \infty$, and $\eps \in \left(0, \frac{\mu_1 - \mu_2}{2} \right)$, there exists some constants $ C(\pi_{\mathrm{j}}, \bmu, \eps)< \infty$ such that
\begin{equation*}
    \sum_{s=1}^{\infty} \mathbb{P}[\eC(s), N_1(s) \leq \sqrt{s}] \leq C(\pi_{\mathrm{j}}, \bmu, \eps).
\end{equation*}
\end{lemma}
The proof of Lemma~\ref{lem: N1control} is given in \ref{sec: bai_pf_N1}.

\section{Proofs of additional lemmas}
In this section, we provide proofs of additional lemmas that prove the lemmas for proving Theorem~\ref{thm: bai_TB}.

\subsection{Proof of technical lemma for Lemma~\ref{lem: Dibound}: Lemma~\ref{lem: KLdiv}}\label{sec: bai_lemKLdiv}

\begin{proof}[Proof of Lemma~\ref{lem: KLdiv}]
    It holds from the expression of KL divergence that
\begin{align*}
    d(\mu, \mu') - d(\mu, \mu_0) &= A(\theta(\mu_0)) -A(\theta(\mu')) + (\theta(\mu') - \theta(\mu_0))\mu \\
    &\leq A(\theta(\mu_0)) - \inf_{x\in R} A(\theta(x)) + |\mu| \sup_{x\in A} | \theta(x) - \theta(\mu_0)|.
\end{align*}
Since $d(\mu, \mu_0)$ is convex with respect to $\mu$, there exist constant $a', b' \geq 0$ such that $|\mu| \leq a' d(\mu, \mu_0) + b'$.
Letting $a:= 1+a' \sup_{x\in A} | \theta(x) - \theta(\mu_0)|$ and $b := b' \sup_{x\in A} | \theta(x) - \theta(\mu_0)| + A(\theta(\mu_0)) - \inf_{x\in A} A(\theta(x)) $ concludes the proof.
\end{proof}

\subsection{Proof of technical lemma for Lemma~\ref{lem: PCsum}: Lemma~\ref{lem: N1control}}\label{sec: bai_pf_N1}
Here, we present the proof of Lemma~\ref{lem: N1control}, where we adapt the proof techniques considered in \citet{kaufmann2012thompson} and \citet{KordaTS}.
Before beginning, we introduce some results in \citet{KordaTS}.

The following Lemma shows the concentration inequality when an arm is played sufficiently.
\begin{lemma}[Lemma 10 in \cite{KordaTS}]\label{lem: Korda10}
For every $a \in [K]$ and $\eps >0$, there exist constants $C_a' = C' (\mu_a, \eps, A)$ and $N$ such that for $t \geq N_K$,
\begin{align*}
    \mathbb{P}[\exists s \leq t, \exists a \ne 1 : | \hmu_a(s) - \mu_a | \geq \eps, N_a(s) > C_a' \log t] &\leq \frac{2(K-1)}{t^3} \\
    \mathbb{P}[\exists s \leq t, \exists a \ne 1 : | \tmu_a(s) - \mu_a | \geq \eps, N_a(s) > C_a' \log t] &\leq \frac{4(K-1)}{t^3}.
\end{align*}
\end{lemma}
Note that we use the upper bound with the order of $\mathcal{O}(t^{-3})$ differently from the original lemma whose order is $\mathcal{O}(t^{-2})$.
This can be done simply by changing the constant term with a multiplication of $3/2$.

The following lemma holds for the SPEF.
\begin{lemma}[Lemma 9 in \cite{KordaTS}]\label{lem: Korda9}
There exists a constant $C = C(\pi_{\mathrm{j}}) < 1$, such that for every (random) interval $I$ and for every positive function $\ell$, one has
\begin{equation*}
    \mathbb{P}[\forall s \in I, \tmu_1(s) \leq \mu_2 + \eps, | I | \geq \ell(t)] \leq C^{\ell(t)}.
\end{equation*}
\end{lemma}

\begin{proof}[Proof of Lemma~\ref{lem: N1control}]
    Let $\tau_n$ denote $n$-th time when arm $1$ is played and $\xi_n = (\tau_{n+1} -1 ) - \tau_n$ be the time between $n+1$-th and $n$-th time of arm $1$ playing.
From the definition, it holds that
\begin{equation*}
    \mathbb{P}[N_1(t) \leq \sqrt{t}, \eC(t)] \leq \sum_{n=0}^{\lfloor \sqrt{t} \rfloor} \mathbb{P}[\xi_n \geq \sqrt{t}-1, \eC(t)].
\end{equation*}
For simplicity, let us define an event
\begin{equation*}
    G_n := \{ \xi_n \geq \sqrt{t}-1 , \eC(t)\} =  \{ \xi_n \geq \sqrt{t}-1 ,\{ \exists n \geq N_1(t)  : |\hmu_{1,n}-\mu_1| \geq \eps \}\}
\end{equation*}
so that
\begin{equation*}
    \mathbb{P}[N_1(t) \leq \sqrt{t}, \eC(t)] \leq \sum_{n=0}^{\lfloor \sqrt{t} \rfloor} \mathbb{P}[G_n].
\end{equation*}
On $G_n$, we define an index set $I_n$ and its subset $I_{n,l}$
\begin{align*}
    I_n &:= [ \tau_n, \tau_n + \lceil \sqrt{t}-1 \rceil] \subset [\tau_n, \tau_{n+1}] \\
    I_{n,l} &:= \left[ \tau_n + \left\lceil \frac{l-1}{K}(\sqrt{t}-1) \right\rceil, \tau_n + \left\lceil \frac{l}{K}(\sqrt{t}-1) \right\rceil \right], \quad l \in [K].
\end{align*}
Note that the inclusion on $I_n$ holds under $G_n$.
In the analysis of Thompson sampling~\citep{agrawal2012analysis, kaufmann2012thompson, KordaTS}, an arm $a$ is called \emph{saturated} if $N_a(t) \geq C_a' \log t$ for a constant $C_a'$ that depends on the model.

In this chapter, we call an arm $i$ is saturated if $N_i(t) \geq \max_{a\in [K]} C_{a} \log t$ for a constant $ C_{a}$ such that
\begin{equation*}
    C_a \geq C_a' \frac{d(\mu_2 + \eps, \mu_K - \eps)}{\underline{d}_a}.
\end{equation*}
Note that $C_a$'s are also constants that only depend on the model, and $C_a \geq C_a'$ holds from the definition of $\underline{d}_a$, so that Lemma~\ref{lem: Korda10} is still applicable.
For each interval $I_n$, let introduce
\begin{itemize}
    \item $F_{n,l} $: the event that by the end of the interval $I_{n,l}$ at least $l$ suboptimal arms are saturated.
    \item $r_{n,l}$: the number of playing unsaturated suboptimal arms, which is called interruptions during $I_{n,l}$.
\end{itemize}
Let us consider 
\begin{equation}\label{eq: Gsum}
    \mathbb{P}[G_n ] = \underbrace{\mathbb{P}[G_n,  F_{n,K-1}]}_{(D1)} +  \underbrace{\mathbb{P}[G_n,  F_{n,K-1}^c]}_{(E1)}.
\end{equation}

\subsubsection{Bounds on (D1)}
From the definition, one can rewrite
\begin{align*}
    (D1) &= \mathbb{P}[\{ \exists s \in I_{n, K}, \exists a \ne 1 : \tmu_a(s) \geq \mu_2 + \eps \}, G_n, F_{n, K-1}] \\
    &\hspace{3em} + \mathbb{P}[\{ \forall s \in I_{n, K}, \forall a \ne 1 : \tmu_a(s) \leq \mu_2 + \eps \}, G_n, F_{n, K-1}] \\
    &\leq \frac{2(K-1)}{t^3} +  \overbrace{\mathbb{P}[ \underbrace{\{ \forall s \in I_{n, K}, \forall a \ne 1 : \tmu_a(s) \leq \mu_2 + \eps \}}_{=: D_{n,K}}, G_n, F_{n, K-1}]}^{(D2)},
\end{align*}
where the inequality holds from Lemma~\ref{lem: Korda10}.
Here, (D2) can be decomposed as
\begin{multline*}
    (D2) = \mathbb{P}[D_{n,K} , G_n, F_{n,K-1},\{ \forall a \ne 1, \exists s \in I_{n,K}: \eB_a^c(s) \cup \etB_a^c (s)  \}]  \\ +  \mathbb{P}[D_{n,K} , G_n, F_{n,K-1},\{ \forall a \ne 1, \forall s \in I_{n,K}: \eB_a(s) \cap \etB_a(s)  \}].
\end{multline*}
From Lemma~\ref{lem: Korda10}, we obtain
\begin{align*}
    (D2) &\leq \frac{6(K-1)}{t^3} + \mathbb{P}[D_{n,K} , G_n, F_{n,K-1},\{ \forall a \ne 1, \forall s \in I_{n,K}: \eB_a(s) \cap \etB_a(s)  \}] \\
    &\leq \frac{6(K-1)}{t^3} \\
    &\quad+ \mathbb{P}[D_{n,K} , G_n, F_{n,K-1},\{ \forall a \ne 1, \forall s \in I_{n,K}: \eB_a(s) \cap \etB_a(s), \tm(s)\ne 1 \}]
    \\& \quad + \mathbb{P}[D_{n,K} , G_n, F_{n,K-1},\{ \forall a \ne 1, \forall s \in I_{n,K}: \eB_a(s) \cap \etB_a(s)\} \\ 
    &\hspace{20em},\{\exists s \in I_{n,K} : \tm(s) =1\}] \\
    &\leq  \frac{6(K-1)}{t^3} + C^{\frac{\sqrt{t}-1}{K}} \\
    &\quad + 
    \begin{aligned}
        &\mathbb{P}[D_{n,K} , G_n, F_{n,K-1},\{ \forall a \ne 1, \forall s \in I_{n,K}: \eB_a(s) \cap \etB_a(s)\} \\
        &\hspace{3em},\{\exists s \in I_{n,K} : \tm(s) =1\}]
    \end{aligned} \Biggr\}(D_3),
\end{align*}
where the last inequality holds from Lemma~\ref{lem: Korda9}.
Next, one can see
\begin{align*}
    (D3) &= \mathbb{P}[D_{n,K} , G_n, F_{n,K-1},\{ \forall a \ne 1, \forall s \in I_{n,K}: \eB_a(s) \cap \etB_a(s) \} \\
    & \hspace{13em}, \{\exists s \in I_{n,K} : \tm(s) =1, m(s)=1\}] \\
    &\quad + \mathbb{P}[D_{n,K} , G_n, F_{n,K-1},\{ \forall a \ne 1, \forall s \in I_{n,K}: \eB_a(s) \cap \etB_a(s) \} \\ 
    &\hspace{13em}, \{\exists s \in I_{n,K} : \tm(s) =1, m(s)\ne 1\}] \\
    &\leq  \mathbb{P}[D_{n,K} , G_n, F_{n,K-1},\{ \forall a \ne 1, \forall s \in I_{n,K}: \eB_a(s) \cap \etB_a(s)\} \\
    &\hspace{13em}, \{\exists s \in I_{n,K} : \tm(s) =1, m(s)=1\}] 
    \\
    &\quad+
    \mathbb{P}[D_{n,K} , G_n, F_{n,K-1}, \{\text{arm } 1 \text{ is saturated}\}, \{\exists s \in I_{n,K}: \eB_1^c(s)\}] \numberthis{\label{eq: bai_star_4}}
\end{align*}
where (\ref{eq: bai_star_4}) holds from Thompson exploration since $i(t)\ne 1$ on $\eM^c(t)$ implies that $N_1(t) \geq N_{i(t)}$, i.e., arm $1$ is saturated.
From Lemma~\ref{lem: Korda9}, it holds that
\begin{align*}
    (D3) &\leq \frac{2(K-1)}{t^3} + \mathbb{P}[D_{n,K} , G_n, F_{n,K-1},\{ \forall a \ne 1, \forall s \in I_{n,K}: \eB_a(s) \cap \etB_a(s)\} \\
    &\hspace{13em} ,\{\exists s \in I_{n,K} : \tm(s)=m(s)=1\}] \\
    &= \frac{2(K-1)}{t^3} + (D4),
\end{align*}
where $(D4)$ denotes the second term.
Note that Thompson exploration with $\{m(s)=1 \}$ will choose only $j(s)$ under the event $G_n$, i.e., only $\{i(s) = j(s)\}$ happens during $I_n$ for any $n$ when $m(s)=\tm(s)$ holds.
It holds that
\begin{multline*}
    (D4) \leq \underbrace{\sum_{s\in I_{n,K}} \sum_{a=2}^K \mathbb{P}[m(s) =1, i(s)=j(s)=a, \eA_1(s), \eB_a(s), \eM(s), G_n]}_{(D5)}
    \\ + \underbrace{\sum_{s\in I_{n,K}} \sum_{a=2}^K \mathbb{P}[m(s) =1, i(s)=j(s)=a, \eA_1^c(s), \eB_a(s), \eM(s)]}_{(D6)}.
\end{multline*}
From Lemma~\ref{lem: SubGood}, if an event in $(D5)$ occurs for some $s$, then it implies that $\eB_1(t)$ holds for all $t \geq s$ such that for all $t\geq N'$, $C_a^* \log t \geq \max\{ M, D_1/ \underline{d}_a\}$ for all $a \in [K]\setminus \{1\}$ holds, which contradicts to the event $G_n$ that implies the existence of $t \geq s$ such that $\eB_1^c(t)$ holds.
Therefore, we have
\begin{equation*}
    (D5) =0.
\end{equation*}
Note that $(D6)$ is the form considered in Lemma~\ref{lem: AcB}.
Therefore, we have
\begin{equation*}
   (D6) \leq \frac{\sqrt{t}-1}{K}\sum_{a=2}^K \mathbb{P}\left[N_a(s) \leq c_{a}^* D_{a}\right],
\end{equation*}
for some constants $c_{a}^*$ and random variables $D_{a}$ in Lemma~\ref{lem: AcB} such that its expectation is finite.
Let $N_{\bmu, A}(\eps)$ be a constant that depends on the model and epsilon such that for $t \geq N_{\bmu, A}(\eps)$, it holds for any $a \in \{ 2,\ldots, K\}$
\begin{equation*}
    C_a^* \log t \geq c_{a}^* D_{a},
\end{equation*}
i.e., the event in $(D6)$ cannot occur for $t \geq N_{\bmu, A}(\eps)$.
Hence, there exist some constant $C_D(\pi_{\mathrm{j}}, \bmu, b, \eps) < \infty$ such that
\begin{align*}
    \sum_{t=1}^T \sum_{n=0}^{\lfloor \sqrt{t} \rfloor} (D1) & \leq \max\left\{ N', N_{\bmu, A}(\eps) \right\}+ \sum_{t=N_{\bmu, A}(\eps)+1}^{\infty} \frac{8(K-1)}{t^2\sqrt{t}} + \sqrt{t} C^{\frac{\sqrt{t}-1}{K}} \\
    &\leq C_D(\pi_{\mathrm{j}}, \bmu, b, \eps). \numberthis{\label{eq: Drslt}}
\end{align*}

\subsubsection{Bounds on (E1)}
By adapting the proof of \citet{kaufmann2012thompson, KordaTS}, we prove $(E1)$ is upper bounded by some constants through the mathematical induction, i.e., we will show
\begin{equation*}
     \mathbb{P}[G_n, F_{n,K-1}^c] \leq (K-2)\left( \frac{10(K-1)}{t^3} +  k(\bmu, b, n, t) \right),
\end{equation*}
where $k$ is a function such that $\sum_{t\geq 1}\sum_{n \leq \sqrt{t}} k < \infty$.

First, for the base case, it can be easily seen that for $t \geq N_{\bmu, b}$ such that
\begin{equation*}
    \forall t \geq N_{\bmu, b}, \, \left \lceil \frac{\sqrt{t}-1}{K^2}  \right \rceil \geq C_* \log t,
\end{equation*} 
where $C_* = \max_{a\ne 1} C_a$ since only suboptimal arms are selected during $I_{n,l}$ under $G_n$.
Then, for $t \geq N_{\bmu, b}$,
\begin{equation*}
    \mathbb{P}[G_n, F_{n,1}^c] = 0.
\end{equation*}
We refer the reader to \citet{kaufmann2012thompson} for more explanations in the base case.
Then, we assume that for some $2 \leq l \leq K-1$ if $t \geq N_{\bmu, b}$, then
\begin{equation*}
    \mathbb{P}[G_n, F_{n,l-1}^c] \leq (l -2) \left( \frac{10(K-1)}{t^3} + k(\bmu, b, n, t) \right).
\end{equation*}
Therefore, we remain to show that
\begin{equation*}
    \mathbb{P}[G_n, F_{n,l}^c, F_{n,l-1}] \leq \frac{10(K-1)}{t^3} +  k(\bmu, b, n, t).
\end{equation*}
On the event $(G_n, F_{n,l}^c, F_{n,l-1} )$, there are exactly $l-1$ saturated suboptimal arms at the beginning of interval $I_{n,l}$ and no new arm is saturated during this interval, which implies that $r_{n,l} \leq KC_*\log t$.
For the set of saturated suboptimal arms $\mathcal{S}_{l}$ at the end of $I_{n,l}$, it holds that
\begin{align*}
    \mathbb{P}[G_n, F_{n,l}^c, F_{n,l-1}] &\leq \mathbb{P}[G_n, F_{n,l-1}, \{r_{n,l} \leq KC_* \log t\}] \\
    &\leq \mathbb{P}[G_n, F_{n,l-1}, \{\exists s \in I_{n,l}, a \in \mathcal{S}_{l-1}: \etB_a^c(s) \cup \eB_a^c(s) \}] 
    \\
    &\quad + \begin{aligned}
        &\mathbb{P}[G_n, F_{n,l-1},  \{r_{n,l} \leq KC_* \log t\},  \\
        &\hspace{1em} \{\forall s \in I_{n,l}, a \in \mathcal{S}_{l-1} : \etB_a(s) \cap \eB_a(s)\}]
    \end{aligned} \Biggr\}(E2),
\end{align*}
By applying Lemma~\ref{lem: Korda10} again, we have
\begin{equation*}
    \mathbb{P}[G_n, F_{n,l-1}, \{\exists s \in I_{n,l}, a \in \mathcal{S}_{l-1}: \etB_a^c(s) \cup \eB_a^c(s)\}]  \leq \frac{6(K-1)}{t^3}.
\end{equation*}
To bound $(E2)$, we introduce a random interval $\mathcal{J}_k$ for $k \in \{0,\ldots, r_{n,l}-1\}$ as the time between $k$-th and $k+1$-th interruption in $I_{n,l}$ and set $\mathcal{J}_k = \emptyset$ for $k \geq r_{n,l}$.
On $(E2)$, there is a subinterval where no interruptions occur with length $\lceil \frac{\sqrt{t}-1}{C_*K^2 \log t}\rceil$.
Then, it holds that
\begin{align*}
    (E2) &\leq \mathbb{P}\Bigg[ \left\{ \exists k \in \{ 0, \ldots, r_{n,l}\} : |\mathcal{J}_k | \geq \frac{\sqrt{t}-1}{C_* K^2 \log t} \right\}, \\
    &\hspace{10em} \{ \forall s \in I_{n,l}, a \in \mathcal{S}_l : \etB_a(s) \cap \eB_a(s) \}, G_n, F_{n,l-1} \Bigg] \\
    &\leq \sum_{k=1}^{KC_* \log t}  \mathbb{P}\left[ \left\{ |\mathcal{J}_k | \geq \frac{\sqrt{t}-1}{C_* K^2 \log t} \right\}, \{ \forall s \in \mathcal{J}_k, a \in \mathcal{S}_l : \etB_a(s) \cap \eB_a(s) \}, G_n \right].
\end{align*}
Note that on $G_n$ and $\forall s \in \mathcal{J}_k$, only $i(s) \in \mathcal{S}_{l}$ happens, i.e., $\{m(s)\ne \tm(s), m(s) \not\in \mathcal{S}_l, \tm(s) \not\in \mathcal{S}_l\}$ cannot occur.
Therefore, for any $s \in \mathcal{J}_k$ under $\{ \forall a \in \mathcal{S}_l : \etB_a(s) \cap \eB_a(s) \}$, we have 
\begin{align*}
    \I[ m(s) \ne \tm(s), G_n, \etB_{\tm(s)}(s)] &= \I[ m(s) \in \mathcal{S}_l , \tm(s) \in \mathcal{S}_l \setminus \{ m(s) \}, G_n,  \etB_{\tm(s)}(s) ]
    \\
    &\quad+\I[ m(s) =1 , \tm(s) \in \mathcal{S}_l , G_n,  \etB_{\tm(s)}(s),  \etB_{1}^c(s)] \\
    &\quad+ \I[ m(s) \in \mathcal{S}_l , \tm(s) =1  , G_n,  \etB_{1}(s), \eB_{1}^c(s) ].
\end{align*}
Here, it holds that
\begin{equation*}
     \{ m(s) \in \mathcal{S}_l , \tm(s) \in \mathcal{S}_l \setminus \{ m(s) \}, G_n,  \etB_{\tm(s)}(s) \} \subset \{ \tmu_1(s) \leq \mu_2 +\eps, G_n \}.
\end{equation*}
Similarly to the (D3), $i(s)\ne 1$ implies that arm $1$ is already played more than the saturated arm.
Let us define an event
\begin{equation*}
    E2(s) := \{ m(s) = \tm(s) \in \mathcal{S}_l^c \cup \{ 1 \}] \} \cap \{ \tmu_1(s) \geq \mu_2 + \eps \}.
\end{equation*}
Then, from the above inclusive relationship, we have
\begin{align*}
    \mathbb{P}\Bigg[ \Bigg\{ &|\mathcal{J}_k | \geq \frac{\sqrt{t}-1}{C_* K^2 \log t} \Bigg\}, \{ \forall s \in \mathcal{J}_k, a \in \mathcal{S}_l : \etB_a(s) \cap \eB_a(s) \}, G_n \Bigg] \\
    &\leq 
    \mathbb{P}\Bigg[ \left\{ |\mathcal{J}_k | \geq \frac{\sqrt{t}-1}{C_* K^2 \log t} \right\}, \bigg\{ \forall s \in \mathcal{J}_k: \{\forall a \in \mathcal{S}_l: \etB_a(s) \cap \eB_a(s) \}  \\
    & \hspace{20em} \cap \{ \tmu_1(s) \leq \mu_2 + \eps \} \bigg\}, G_n \Bigg] \\
    &+ \mathbb{P}\Bigg[ \left\{ |\mathcal{J}_k | \geq \frac{\sqrt{t}-1}{C_* K^2 \log t} \right\},  \{ \forall s \in \mathcal{J}_k,  a \in \mathcal{S}_l: \etB_a(s) \cap \eB_a(s) \}, \\
    & \hspace{17em} \{ \exists s \in \mathcal{J}_k : \eB_1^c(s) \cup \etB_1^c(s)\},  G_n \Bigg]
    \\
    &+      \begin{aligned}
        &\mathbb{P}\Bigg[ \left\{ |\mathcal{J}_k | \geq \frac{\sqrt{t}-1}{C_* K^2 \log t} \right\},  \{ \forall s \in \mathcal{J}_k,  a \in \mathcal{S}_l: \etB_a(s) \cap \eB_a(s) \} \\
        &\hspace{16em} \{ \exists s \in \mathcal{J}_k : E2(s) \},  G_n \Bigg]
    \end{aligned} \Biggr\}(E3).
\end{align*}
By applying Lemmas~\ref{lem: Korda10} and~\ref{lem: Korda9}, we have
\begin{multline*}
    \mathbb{P}\Bigg[ \left\{ |\mathcal{J}_k | \geq \frac{\sqrt{t}-1}{C_* K^2 \log t} \right\},\{ \forall s \in \mathcal{J}_k, a \in \mathcal{S}_l : \etB_a(s) \cap \eB_a(s) \}, G_n \Bigg]
    \leq C^{\frac{\sqrt{t}-1}{C_* K^2 \log t}} + \frac{6}{t^3} + (E3).
\end{multline*}

From the definition of $\mathcal{J}_k$ and $G_n$, one can see that
\begin{align*}
    (E3) &= \mathbb{P}\Bigg[\left\{ |\mathcal{J}_k | \geq \frac{\sqrt{t}-1}{C_* K^2 \log t} \right\},  \{ \forall s \in \mathcal{J}_k:  a \in \mathcal{S}_l: \etB_a(s) \cap \eB_a(s) \} \\
    &\hspace{12em}, \{ \exists s \in \mathcal{J}_k : E2(s) \cap \{j(s) = i(s) \in \mathcal{S}_l\} \},  G_n \Bigg] 
    \\
    &\leq \mathbb{P}\bigg[\exists s \in \mathcal{J}_k: m(s) = \tm(s) \in \mathcal{S}_l^c \cup \{1 \}, j(s) \in \mathcal{S}_l, i(s)=j(s), \mathcal{A}_{m(s)}^c \\
    &\hspace{20em} ,\mathcal{B}_{j(s)}, \tmu_1(s) \geq \mu_2 + \eps, G_n \bigg]
    \\
    &\hspace{1em} +  \mathbb{P}\bigg[\exists s \in \mathcal{J}_k: m(s) = \tm(s) \in \mathcal{S}_l^c \cup \{1 \}, j(s) \in \mathcal{S}_l, i(s)=j(s), \mathcal{A}_{m(s)} \\ 
    &\hspace{17em}, \mathcal{B}_{j(s)}, \tmu_1(s) \geq \mu_2 + \eps, G_n \bigg]. \numberthis{\label{eq: bai_star_5}} \\
    &=: (E4) + (E5).
\end{align*}
The first equation holds since only saturated suboptimal arms have to be played on $\mathcal{J}_k$ when $m(s)=\tm(s)$ is unsaturated or optimal arm, which makes $j(s) = i(s) \in \mathcal{S}_l$.
Let us denote the event in the first term and the second term of RHS in (\ref{eq: bai_star_5}) by $(E4)$ and $(E5)$, respectively.

From Lemma~\ref{lem: AcB}, we have
\begin{align*}
    \mathbbm{1}[(E4)] &\leq \sum_{s\in \mathcal{J}_k} \sum_{a\in \mathcal{S}_l} \sum_{m \in \mathcal{S}_l \cup \{1 \}} \mathbbm{1}[m(s)=m, i(s)=j(s)=a, \eA_m^c(s), \eB_a(s)] \\
    &\leq \sum_{s\in \mathcal{J}_k} \sum_{a\in \mathcal{S}_l} \sum_{m \in \mathcal{S}_l \cup \{1 \}} \mathbbm{1}[N_a(s) \leq c_{m,a}^* D_{m,a} ].
\end{align*}
Similarly to the case of $(D4)$, there exists some deterministic constant $N_{\bmu, A}(\eps)'$ such that for $t \geq N_{\bmu, A}(\eps)'$, $\forall (m,a) \in ( \mathcal{S}_l^c \cup \{1\},\mathcal{S}_l)$ 
\begin{equation*}
    C_a^* \log t \geq c_{m,a}^* D_{m,a},
\end{equation*}
where we replace $1$ by $m$ in $c_a^*$ and $D_a$ to define $c_{m,a}^*$ and $D_{m,a}$. 

Further, $(E5)$ can be decomposed by
\begin{equation*}
    (E5) = (E6) + (E7),
\end{equation*}
where
\begin{align*}
    (E6) &:= \mathbb{P}\bigg[\exists s \in \mathcal{J}_k: m(s) = \tm(s) \in \mathcal{S}_l^c, j(s) \in \mathcal{S}_l, i(s)=j(s), \mathcal{A}_{m(s)}, \mathcal{B}_{j(s)}, \tmu_1(s) \geq \mu_2 + \eps, G_n \bigg]  \\
    (E7) &:= \mathbb{P}\bigg[\exists s \in \mathcal{J}_k: m(s) = \tm(s)=1, j(s) \in \mathcal{S}_l, i(s)=j(s),\mathcal{A}_{1}, \mathcal{B}_{j(s)}, \tmu_1(s) \geq \mu_2 + \eps, G_n \bigg].
\end{align*}
Note that on $(E6)$, $\etB_m^c(s)$ always holds since $\tmu_1 > \mu_2 + \eps$ but $\tm(s) \ne 1$ and $(E5)$ is a subset of the event we consider in Lemma~\ref{lem: SubGood}, i.e., event $(E6)$ implies the existence of $s \in \mathcal{J}_k$ such that
\begin{equation*}
    N_m(s) \geq N_{j(s)} \frac{\underline{d}_{j(s)}}{d(\mu_m + \eps, \mu_j - \eps)} \geq C_*  \frac{\underline{d}_{j(s)}}{d(\mu_m + \eps, \mu_{j(s)} - \eps)} \log t.
\end{equation*}
From the definition of $C_*$ and saturation, it holds that for any $m \in \mathcal{S}_l^c$
\begin{equation*}
    C_* \frac{\underline{d}_{j(s)}}{d(\mu_m + \eps, \mu_{j(s)}-\eps)} \geq C_* \frac{\min_{a \ne 1}\underline{d}_{a}}{d(\mu_2 + \eps, \mu_{K}-\eps)} \geq C_m' \log t.
\end{equation*}
As a result, we have
\begin{equation*}
    \mathbb{P}[(E6)]  = \mathbb{P}[\{\exists s \in \mathcal{J}_k, m \in \mathcal{S}_l^c : \etB_m^c(s)\} \cap (E5)]\leq \frac{4(K-1)}{t^3}.
\end{equation*}

Similarly to the case of $(D5)$, if the event in $(E7)$ occurs some $s \in \mathcal{J}_k$ for $t$ such that $t\geq N'$, $C_a^* \log t \geq \max\{ M, D_1/ \underline{d}_a\}$ for all $a \in [K]\setminus \{1\}$, then only $\eB_1(t)$ holds for $s\geq t$ holds, which contradicts to the event $G_n$.

Therefore, for $t \geq N_0 := \max( N_{\bmu, b}, N_{\bmu, A}(\eps)', N_K, N')$, where $N_K$ in Lemma~\ref{lem: Korda10}, it holds
\begin{equation*}
    (E2) \leq KC_* \log t \left( C^{\frac{\sqrt{t}-1}{C_* K^2 \log t}} + \frac{10(K-1)}{t^3} \right) =: k(\bmu, b, n, t).
\end{equation*}
Hence, there exists some constants $C_E(\pi_{\mathrm{j}}, \bmu, b, \eps) < \infty$ such that
\begin{align*}
    \sum_{T=1}^\infty \sum_{t=T+1}^{\infty} \sum_{n=1}^{\lfloor \sqrt{t} \rfloor} (E1) &\leq  N_0 + \sum_{T=N_0 + 1}^\infty \sum_{t=T+1}^{\infty} \frac{6(K-1)^2}{t^2\sqrt{t}} \\ 
    &\quad + \sum_{T=N_0 + 1}^\infty \sum_{t=T+1}^{\infty} KC_* \log t \left( \sqrt{t} C^{\frac{\sqrt{t}-1}{C_* K^2 \log t}} + \frac{10(K-1)}{t^2\sqrt{t}} \right) \\
    &\leq N_0 + C_E(\pi_{\mathrm{j}}, \bmu, b, \eps). \numberthis{\label{eq: Erslt}}
\end{align*}

\subsubsection{Conclusion}
By combining (\ref{eq: Drslt}) and (\ref{eq: Erslt}) with (\ref{eq: Gsum}), we obtain
\begin{align*}
    \sum_{T=1}^\infty \sum_{t=T+1}^{\infty} \mathbb{P}[N_1(t) \leq \sqrt{t}, \eC(t)] &\leq \sum_{T=1}^\infty \sum_{t=T+1}^{\infty} \sum_{n=N_1(T+1)}^{\lfloor \sqrt{t} \rfloor} (D1) + (E1)  \\ 
    &\leq N_0 + C_D(\pi_{\mathrm{j}}, \bmu, b, \eps) + C_E(\pi_{\mathrm{j}}, \bmu, b, \eps) \\
    &=: C(\pi_{\mathrm{j}}, \bmu, b, \eps) < \infty ,
\end{align*}
which concludes the proof. 
\end{proof}

\section{Proof of Theorem \ref{thm: bai_sample}: Sample complexity}\label{sec: bai_pfsample}
Here, we derive the upper bound on the sample complexity of BC-TE.

Before beginning the proof, we first provide a technical lemma provided in \citet{garivier2016optimal}.
\begin{lemma}[Lemma 18 in \citet{garivier2016optimal}]\label{lem: bai_garivier18}
For every $\alpha \in [1, \frac{e}{2}]$, for any two constants $c_1, c_2 > 0$,
\begin{equation*}
    x = \frac{\alpha}{c_1} \left[\log \left(\frac{c_2e}{c_1^\alpha}\right)+ \log\log \left(\frac{c_2}{c_1^\alpha}\right)\right]
\end{equation*}
is such that $c_1 x \geq \log (c_2 x^\alpha)$.
\end{lemma}
Next, we define a set of bandit instances $\mathcal{S}$ for any $\eps > 0$ as follows:
\begin{equation*}
    \mathcal{S} = \mathcal{S}(\nu, \eps) := \{ \bmu' : | \bmu' - \bmu | \leq \eps \},
\end{equation*}
where $\bmu$ denotes the true mean reward vector.
For any $i\ne 1$, if $\bmu' \in \mathcal{S}$, we have the following inequality:
\begin{equation}\label{eq: f_epsilon}
    \forall \bw \in \Sigma_K : \frac{1}{1+\eps} f_i(\bw; \bmu) \leq f_i(\bw;\bmu') \leq (1+\eps) f_i(\bw; \bmu).
\end{equation}
From the relationship in (\ref{eq: bai_fkh_relation}), (\ref{eq: f_epsilon}) is equivalent to
\begin{align*}
     \forall \bw \in \Sigma_K  &: \frac{1}{1+\eps} g(\bw; \bmu) \leq  g(\bw; \bmu') \leq (1+\eps) g(\bw; \bmu) \\
     \forall x \in [0,1] &: \frac{1}{1+\eps} k_i(x; \bmu) \leq  k_i(x; \bmu') \leq (1+\eps) k_i(x; \bmu) \\
     \forall z \in [0,1] &: \frac{1}{1+\eps} h_i(z; \bmu) \leq  h_i(z; \bmu') \leq (1+\eps) h_i(z; \bmu).
\end{align*}
Notice that that for any $t \geq T_B$, $\hbmu(t) \in \mathcal{S}$ holds from the the definition of $T_B$ in (\ref{eq: bai_def_TB}).

Therefore, we can assume
\begin{align}
    \frac{1}{1+\eps} \frac{z_{i}^*}{1-z_{i}^*}  &\leq \frac{z_{i}^*(\bmu')}{1-z_{i}^*(\bmu')} \leq (1+\eps) \frac{z_{i}^*}{1-z_{i}^*} \label{eq: asm_z}\\
    \frac{1}{1+\eps} \frac{\uz_{i}}{1-\uz_{i}}  &\leq \frac{\uz_{i}(\bmu')}{1-\uz_{i}(\bmu')} \leq (1+\eps) \frac{\uz_{i}}{1-\uz_{i}}.  \label{eq: uzat_close}
\end{align}
and for $t \geq T_B$ and the definition of a challenger at round $t$, $j(t)$ in (\ref{eq: j(t)}),
\begin{align}
    \frac{1}{1+\eps} \min_{a\ne 1} f_{i}(x; \bmu) \leq  f_{j(t)}(x; \bmu) \leq (1+\eps)\min_{a\ne 1} f_{i}(x; \bmu). \label{eq: asm_f_jt}
\end{align}
Notice that (\ref{eq: asm_f_jt}) provides 
\begin{align}
    \frac{1}{1+\eps} \min_{a\ne 1} k_{i}(x; \bmu) &\leq  k_{j(t)}(x; \bmu) \leq (1+\eps)\min_{i\ne 1} k_{i}(x; \bmu). \label{eq: asm_k_jt}
\end{align}
Since $tf_{i}(\bw^t;\bmu) = (N_1(t) + N_{i}(t))h_{i}(z_{i}^t; \bmu)$ holds from their relationship in (\ref{eq: bai_fkh_relation}) and $z_i^t = \frac{w_i^t}{w_1^t+ w_i^t}$, (\ref{eq: asm_f_jt}) also implies that
\begin{align*}
    \frac{1}{1+\eps} \min_{i\ne 1} (N_1(t) + N_i(t)) h_{i}(z_i^t; \bmu) &\leq  (N_1(t) + N_{j(t)}(t)) h_{j(t)}(z_{j(t)}^t; \bmu) \\
    &\leq (1+\eps)\min_{i\ne 1} (N_1(t) + N_i(t)) h_i(z_i^t; \bmu).
\end{align*}
From the concavity of the objective function, we have the following result, whose proof is provided in Section~\ref{sec: bai_tf_nondecrease}.
\begin{lemma}\label{lem: bai_tf_nondecrease}
    For any $i \ne 1$, $tf_i(\bw^t; \bmu)$ is non-decreasing with respect to $t \in \mathbb{N}$.
\end{lemma}

\begin{proof}[Proof of Theorem~\ref{thm: bai_sample}]
    We first introduce a positive increasing sequence $(G_m)_{m\in \mathbb{N}}$ and let $\psi_m$ be the first round where $t g(\bw^t; \bmu) > G_m$ holds, which is defined as
\begin{equation*}
    \psi_m := \inf \{t\in \mathbb{N}_{\geq T_B}: t g(\bw^t; \bmu) \geq G_m \}.
\end{equation*}
Notice that Lemma~\ref{lem: bai_tf_nondecrease} ensures $\psi_m \leq \psi_{m+1}$ for any $m \in \mathbb{N}$ since $tg(\bw^t;\bmu) = t\min_{i\ne 1} f_i(\bw^t; \bmu)$ is non-decreasing.

For notational simplicity, $\ug$ denotes the value of the objective function $g(\bw; \bmu)$ at $\bw=\uw$ defined in (\ref{eq: bai_def_uw_uz}).
Then from (\ref{eq: bai_fkh_relation})
\begin{equation}\label{eq: bai_ug_uw_ki}
    \forall i \ne 1 : \ug = \uw_1 k_i(\uw_i/\uw_1;\bmu) = (\uw_1+\uw_i) h_i(\uz_i;\bmu).
\end{equation}

Here, we set $G_1$ to satisfy
\begin{equation}\label{eq: bai_def_G1}
    \forall i \in [K]: N_i(T_B) \leq \frac{\uw_i}{\ug} G_1.
\end{equation}
Then, the stopping time $\tau_\delta$ can be written as
\begin{align*}
    \tau_\delta &= \inf\{t \in \mathbb{N}: tg(\bw^t; \hbmu(t)) \geq \beta(t,\delta) \} \\ 
    &\leq \inf\{t \in \mathbb{N}_{\geq T_B}: \frac{tg(\bw^t; \bmu)}{1+\eps}  \geq \beta(t,\delta) \} \\
    &\leq T_B + \inf\left\{\psi_m : \frac{1}{1+\eps}G_m \geq \beta(\psi_m,\delta), m\in \mathbb{N} \right\}. \numberthis{\label{eq: bai_tau_delta}}
\end{align*}
To find the upper bound of the stopping time, we require the relationship between $G_m$ and $\psi_m$.
To do this, we first derive the bounds on the number of plays $N_i(t)$.

\subsection{Bounds on the number of plays}
Here, we aim to derive the upper bounds on $N_i(t)$ for $t \in [\psi_m, \psi_{m+1})$ and for any $i \in [K]$.

For $t \geq T_B$, only $m(t) = 1$ occurs.
Therefore, an arm $i \ne 1$ is played either when TE occurs or when $j(t) = i$ and $d(\hmu_i(t), \hmu_{1,i}(t)) \geq d(\hmu_1(t), \hmu_{1,i}(t))$ for $t \geq T_B$.
Thus, if $j(t) \ne i$ holds for all $t \in [\psi_{m}, \psi_{m+1})$, then
\begin{equation*}
    N_i(\psi_{m+1}) = N_i(\psi_m) + M_{i,m},
\end{equation*}
where $M_{i,m}$ denote the number of the arm $i$ being played by TE during $[\psi_m, \psi_{m+1})$, which is
\begin{equation*}
    M_{i,m} = \sum_{t=\psi_m}^{\psi_{m+1}-1} \I[\eM^c(t), i(t)=i].
\end{equation*}
The latter condition can be rewritten as $j(t)=i$ and $z_{i}^t \leq z_{i}^*(\hbmu(t))$ from the definition of $z_{i}^*$ in (\ref{eq: bai_def_zstar}).
For notational simplicity, we denote $z_{i}^*(\hbmu(t))$ and $ \uz_i (\hbmu(t))$ by $z_{i,t}^*$ and $\uz_{i,t}$, respectively.

\paragraph{(1) Upper bound for the second-best arm}
Firstly, let us consider the second-best arm $j^*(\nu)$, which is assumed to be the arm $2$ in this chapter.
It should be noted that the second-best arm may not be unique.
Then let us define a partition of $Q_m := [\psi_m, \psi_{m+1})$
\begin{align*}
    (Q1) &:= \left\{ t \in [\psi_m, \psi_{m+1}) : N_1(t) \leq \frac{\uw_1}{\ug}G_{m+1} \right\} \\
    (Q2) &:= \left\{ t \in [\psi_m, \psi_{m+1}) : N_1(t) > \frac{\uw_1}{\ug}G_{m+1} \right\}.
\end{align*}

Then, we define $\eps_1 = \eps_1(\eps, G_{m+1}/G_m) > \eps$ to be a constant satisfying
\begin{equation}\label{eq: bai_eps1_prop}
    k_2\left( (1+\eps_1) \frac{\uw_2}{\uw_1} ; \bmu \right) \geq \frac{G_{m+1}}{G_m} \frac{\ug}{\uw_1},
\end{equation}
Here, one can see that $\eps_1 \to 0_{+}$ as $\eps \to 0_{+}$ and $\frac{G_{m+1}}{G_m} \to 1_{+}$ from (\ref{eq: bai_ug_uw_ki}).
Then we will show that if $N_2(t) \geq N' = (1+\eps_1) \frac{\uw_2}{\ug}G_{m+1}$, then $i(t) = 2$ holds only when TE occurs.

\paragraph{(1-i) When $t \in (Q1)$}
In this case,
\begin{align*}
    N_2(t) \geq N' = (1+\eps_1) \frac{\uw_2}{\ug}G_{m} &= (1+\eps_1) \frac{\uw_2}{\uw_1} \frac{\uw_1}{\ug}G_{m} \\
    &\geq (1+\eps_1) \frac{\uw_2}{\uw_1} N_1(t) \tag*{$\because t \in (Q1)$} \\
    &= (1+\eps_1) \frac{\uz_2}{1-\uz_2} N_1(t) \tag*{by definition of $\uw$ in (\ref{eq: bai_def_uw_uz})} \\
    &=  (1+\eps_1) \frac{z_2^*}{1-z_2^*} N_1(t) \tag*{by definition of $\uz$ in (\ref{eq: bai_def_uz})} \\
    &> \frac{z_{2,t}^*}{1-z_{2,t}^*} N_1(t). \tag*{by (\ref{eq: asm_z}) and $\eps_1 > \eps$}
\end{align*}
This implies that for $t \in (Q1)$, if $N_2(t) \geq N'$, then $z_2^t > z_{2,t}^*$ holds.
Therefore, only $i(t) = 1$ happens unless TE occurs.

\paragraph{(1-ii) When $t \in (Q2)$}
From the relationship between $f_i$ and $k_i$ in (\ref{eq: bai_fkh_relation}), one can see that $tf_i(\bw^t;\bmu) = N_1(t)k_i(w_i^t/w_1^t; \bmu)$.
Therefore, one can extend Lemma~\ref{lem: bai_tf_nondecrease} to show that $y k_i(c/y; \bmu)$ is non-decreasing with respect to $y \geq 0$ for fixed $c >0$ and any $i \ne 1$.
Recall that the $k_i(x;\bmu)$ is a strictly increasing function with respect to $x >0$.
Then we can obtain that
\begin{align*}
    N_1(t) k_2\left( \frac{N_2(t)}{N_1(t)} ; \bmu \right) &\geq N_1(t) k_2\left( \frac{N'}{N_1(t)} ; \bmu \right) \\
    &\geq G_{m} \frac{\uw_1}{\ug} k_2\left( N' \frac{\ug}{G_{m} \uw_1} ; \bmu \right) \tag*{$\because t \in (Q2)$} \\
    &=  G_{m} \frac{\uw_1}{\ug} k_2\left( (1+\eps_1) \frac{\uw_2}{\uw_1} ; \bmu \right) \\
    &\geq G_{m} \frac{\uw_1}{\ug} \frac{G_{m+1}}{G_m} \frac{\ug}{\uw_1} \tag*{by definition of $\eps_1$ in (\ref{eq: bai_eps1_prop})}\\
    &= G_{m+1},
\end{align*}
which contradicts the assumption $t \in (Q2)$.

\paragraph{(1-iii) Conclusion}
Therefore, for any $t \in Q_m$, 
\begin{equation*}
    \left\{ N_2(t) \geq (1+\eps_1) \frac{\uw_2}{\ug}G_{m}\right\} \implies \{ j(t) \ne 2 \},
\end{equation*}
which directly implies that
\begin{equation*}
    N_2(t) \leq \max\left( N_2(\psi_m) , (1+\eps_1) \frac{\uw_2}{\ug}G_{m} \right) + M_{2,m}.
\end{equation*}
Here, from the definition of $G_1$ in (\ref{eq: bai_def_G1}), $N_1(t) \leq \frac{\uw_1}{\ug} G_1$ holds for all $t < \psi_1$, which implies that $N_2(\psi_m) \leq (1+\eps_1) \frac{\uw_2}{\ug}G_{m} + M_{2,0}$.
Therefore, for any $t \in [\psi_m, \psi_{m+1})$,
\begin{align*}
    N_2(t) \leq (1+\eps_1) \frac{\uw_2}{\ug}G_{m} + M_{2}(\psi_{m+1}) 
\end{align*}
where $M_{i}(\psi_{m+1}) = \sum_{l=0}^{m}M_{i,l}$ for any $i\in[K]$.

Here, let use define a random variable $M_T = \sum_{t=T_B}^T \I[\eM^c(t)] = \sum_{i=1}^{K} \sum_{m} M_{i,m}$, which satisfies $\mathbb{E}[M_T] < \infty$ by Lemma~\ref{lem: TSexplore}.
Then we can set $G_m$ sufficiently large to satisfy
\begin{equation*}
    G_m \geq \frac{\ug}{\eps} M_T,
\end{equation*}
which directly implies that
\begin{equation}\label{eq: bai_N2_upper}
    N_2(t) \leq (1+\eps_1) \frac{\uw_2}{\ug}G_{m} + \frac{\eps}{\ug}G_m.
\end{equation}

\paragraph{(2) Lower bound for the optimal arm}
For any $t \in Q_m$, it holds that
\begin{align*}
    G_m &\leq N_1(t) \min_{i\ne 1} k_i\left( \frac{N_i(t)}{N_1(t)}; \bmu \right) \\
    &= \min_{i\ne1} (N_1(t)+N_i(t))h_i(z_i^t;\bmu) \tag*{by (\ref{eq: bai_fkh_relation})}\\
    &\leq (N_1(t) + N_2(t)) h_2(z_2^t;\bmu) \\
    &\leq (N_1(t) + N_2(t)) h_2(\uz_2;\bmu) \tag*{by $\uz_2 = z_2^*$} \\
    &= \frac{N_1(t)+N_2(t)}{\uw_1 + \uw_2} \ug. \tag*{by (\ref{eq: bai_ug_uw_ki})}
\end{align*}
Therefore, for $t = \psi_m$, the upper bound of $N_2(\psi_m)$ in (\ref{eq: bai_N2_upper}) provides
\begin{equation*}
    N_1(\psi_m) \geq \frac{\uw_1 + \uw_2}{\ug} G_m - (1+\eps_1) \frac{\uw_2}{\ug}G_m - \frac{\eps}{\ug} G_m.
\end{equation*}
Since $N_1(t)$ is non-decreasing from its definition, for any $t \geq \psi_m$, 
\begin{equation}\label{eq: bai_N1_lower}
    N_1(t) \geq \frac{\uw_1}{\ug} G_m - \eps_1 \frac{\uw_2}{\ug} G_m - \frac{\eps}{\ug} G_m.
\end{equation}

\paragraph{(3) Upper bound on the challenger arms}
Based on the results obtained in (1) and (2), we will derive the upper bound of $N_{j(t)}(t)$ for $t \geq T_B$.
For $t\in Q_m$, it holds that
\begin{equation*}
    G_m \leq N_1(t) \min_{i\ne 1} k_i\left( \frac{N_i(t)}{N_1(t)}; \bmu \right) < G_{m+1}.
\end{equation*}
Since $j(t) = \argmin_{i=1} f_i(\bw^t; \hbmu(t))$, by using (\ref{eq: asm_k_jt}), one can obtain that
\begin{equation*}
    \frac{1}{1+\eps}k_{j(t)}\left( \frac{N_{j(t)}(t)}{N_1(t)}; \bmu \right) \leq \min_{i\ne 1} k_i\left( \frac{N_i(t)}{N_1(t)}; \bmu \right).
\end{equation*}
Then, by (\ref{eq: bai_N1_lower})
\begin{align*}
    N_1(t) \min_{i\ne 1} k_i\left( \frac{N_i(t)}{N_1(t)}; \bmu \right) &\geq \frac{1}{1+\eps} N_1(t)k_{j(t)}\left( \frac{N_{j(t)}(t)}{N_1(t)}; \bmu \right) \\
    &\hspace{-5em}\geq \frac{1}{1+\eps} \frac{G_m}{\ug} ( \uw_1 -\eps_1 \uw_2 -\eps) k_{j(t)}\left( \frac{\ug N_{j(t)}(t)}{( \uw_1 -\eps_1 \uw_2 -\eps)G_m}; \bmu \right),
\end{align*}
which implies
\begin{equation*}
     k_{j(t)}\left( \frac{\ug N_{j(t)}(t)}{( \uw_1 -\eps_1 \uw_2 -\eps)G_m}; \bmu \right) < (1+\eps) \frac{G_{m+1}}{G_m} \frac{\ug}{\uw_1 - \eps_1 \uw_2 - \eps}.
\end{equation*}
This directly implies that
\begin{align*}
    \frac{\ug N_{j(t)}(t)}{( \uw_1 -\eps_1 \uw_2 -\eps)G_m} &< l_{j(t)}\left( (1+\eps) \frac{G_{m+1}}{G_m} \frac{\ug}{\uw_1 - \eps_1 \uw_2 - \eps} ; \bmu \right) \\
    &\leq (1+\eps_2) \frac{\uw_{j(t)}}{\uw_1},
\end{align*}
where $l_i$ is the inverse function of $k_i$ defined in (\ref{eq: bai_def_li}) and $\eps_2 > \eps_1$ is a constant such that $\eps_2 \to 0_{+}$ as $\eps \to 0_{+}$ and $\frac{G_{m+1}}{G_m} \to 1_{+}$.
Then, we have for any $t \in Q_m$ that
\begin{equation*}\label{eq: bai_Nj_upper_t}
    N_{j(t)}(t) < (1+\eps_2) \frac{\uw_{j(t)}}{\ug} G_m.
\end{equation*}
In other words, if there exists $s \in Q_m$ such that
\begin{equation*}
    N_i(t) \geq (1+\eps_2) \frac{\uw_i}{\ug}G_m,
\end{equation*}
then only $j(s) \ne 1$ occurs for $t \in [s, \psi_{m+1})$, which implies that such arm $i$ will be played only when TE occurs until $\psi_{m+1}$.
Therefore, for $t \in Q_m$
\begin{align*}
    N_i(t) &\leq \max\left( N_i(\psi_{m}, (1+\eps_2) \frac{\uw_i}{\ug} G_m \right) + M_{i,m} \\
    &\leq (1+\eps_2) \frac{\uw_i}{\ug} G_m  + M_i(\psi_{m+1}) \\
    &\leq (1+\eps_2) \frac{\uw_i}{\ug} G_m + \frac{\eps}{\ug}G_m. 
\end{align*}

\paragraph{(4) Upper bound on the optimal arm}
Here, let us assume that there exists $t' \in Q_m$ such that $N_1(t') \geq (1+\eps)(1+\eps_2) \frac{\uw_1}{\ug}G_m$.
If there exists no such $t'$, then one can directly obtain that $N_1(t) \leq (1+\eps)(1+\eps_2) \frac{\uw_1}{\ug}G_m$ for all $t \in Q_m$.

Since $N_{j(t)}(t) < (1+\eps_2) \frac{\uw_{j(t)}}{\ug} G_m$ holds from (\ref{eq: bai_Nj_upper_t}), then for any $t \in [t', \psi_{m+1})$
\begin{align*}
    \frac{N_{j(t)}(t)}{N_1(t)} &< \frac{1}{1+\eps} \frac{\uw_{j(t)}}{\uw_1} = \frac{1}{1+\eps} \frac{\uz_{j(t)}}{1-\uz_{j(t)}} \\
    &\leq \frac{\uz_{j(t),t}}{1-\uz_{j(t),t}}, \tag*{by (\ref{eq: uzat_close})}
\end{align*}
which implies that $z_{j(t)}^t < \uz_{j(t),t } \leq z_{j(t),t}^*$.
Since BC-TE plays the optimal arm $1$ if $z_{j(t),t} \geq z_{j(t),t}^*$, only $i(t) = j(t)$ is possible unless TE occurs until $\psi_{m+1}$.
Therefore, for $t \in Q_m$, it holds that
\begin{align*}
    N_1(t) &\leq \max\left( N_1(\psi_m), (1+\eps)(1+\eps_2) \frac{\uw_1}{\ug} G_m \right) + M_{1,m} \\
    &\leq (1+\eps)(1+\eps_2) \frac{\uw_1}{\ug} G_m + M_1(\psi_{m+1}) \\
    &\leq (1+\eps_3) \frac{\uw_1}{\ug} G_m + \frac{\eps}{\ug} G_m,
\end{align*}
where $\eps_3$ is a constant such that $(1+\eps)(1+\eps_2) = 1+\eps_3$.
One can see that $\eps_3 \to 0_{+}$ as $\eps \to 0_{+}$ and $\frac{G_{m+1}}{G_m} \to 1_{+}$.

\paragraph{(5) Conclusion}
In summary, for any $t \in [\psi_m, \psi_{m+1})$, the results in (1)--(4) imply that for any $i \in [K]$:
\begin{equation}\label{eq: bai_N_upper}
    N_i(t) \leq (1+\eps_3) \frac{\uw_i}{\ug}G_m + \frac{\eps}{\ug}G_m.
\end{equation}

\subsection{Sample complexity}
From the upper bound on the number of plays for each arm in (\ref{eq: bai_N_upper}), for any $m \in \mathbb{N}$,
\begin{align*}
    \psi_m = \sum_{i=1}^K N_i(\psi_m) &\leq \sum_{i=1}^K (1+\eps_3) \frac{\uw_i}{\ug} G_m + \frac{\eps}{\ug} G_m \\
    &= (1+\eps_3) \frac{1}{\ug} G_m + \frac{K\eps}{\ug} G_m,
\end{align*}
which implies that
\begin{equation*}
    \frac{\ug \psi_m}{(1+\eps_3 + K\eps)} \leq G_m.
\end{equation*}
Therefore, the stopping time $\tau_\delta$ in (\ref{eq: bai_tau_delta}) can be written as
\begin{align*}
    \tau_\delta &\leq T_B + \inf\left\{ \psi_m : \frac{1}{1+\eps} G_m \geq \beta(\psi_m, \delta) \right\} \\
    &\leq T_B + \inf\left\{ \psi_m : \frac{1}{1+\eps}\frac{\ug \psi_m}{(1+\eps_3 + K\eps)} \geq \beta(\psi_m, \delta) \right\} \\
    &\leq T_B + \inf\left\{ \psi_m : \frac{\ug \psi_m}{(1+\eps_4)} \geq \log\left( \frac{Ct^\alpha}{\delta} \right) \right\},
\end{align*}
for some $\eps_4 > \eps_3$ satisfying $\eps_4 \to 0_{+}$ as $\eps \to 0_{+}$ and $\frac{G_{m+1}}{G_m} \to 1_{+}$ and constants $C$ and $\alpha \in [1, e/2]$ considered in Section~\ref{sec: bai_stopping}.
Then, by Lemma~\ref{lem: bai_garivier18}
\begin{equation*}
    \tau_\delta \leq T_B + \frac{\alpha}{\ug}(1+\eps_4)\left[\log\left((1+\eps_4)^{\alpha}\frac{Ce}{\delta \ug^{\alpha} } \right)  + \log\log\left((1+\eps_4)^{\alpha}\frac{C}{\delta \ug^{\alpha} } \right) \right].
\end{equation*}
Therefore, by taking expectations, we can obtain that
\begin{equation*}
    \limsup_{\delta \to 0} \frac{\mathbb{E}[\tau_\delta]}{\log(1/\delta)} \leq \frac{\alpha (1+\eps_4)}{\ug}
\end{equation*}
since $\mathbb{E}[T_B]$ is finite from Theorem~\ref{thm: bai_TB}.
Letting $\eps \to 0$ and setting $\frac{G_{m+1}}{G_m} \to 1$ conclude the proof.
\end{proof}

\subsection{Proof of Lemma~\ref{lem: bai_tf_nondecrease}: Non-decreasing objective function}\label{sec: bai_tf_nondecrease}

\begin{proof}[Proof of Lemma \ref{lem: bai_tf_nondecrease}]
    From the relation with $f_i$ and $h_i$ in (\ref{eq: bai_fkh_relation}), we can rewrite the function $tf_i(\bw^t;\bmu)$ as
\begin{equation*}
    tf_i(\bw^t;\bmu) = (N_1(t)+N_i(t)) h_i\left( \frac{N_i(t)}{N_1(t) + N_i(t)}; \bmu \right).
\end{equation*}
Recall that $h_i(z;\bmu)$ is a concave function with respect to $z \in [0,1]$ and $h_i(0;\bmu) = h_i(1;\bmu)=0$ for any $i \ne 1$.
For any $i \ne 1$, let us consider three possible cases (1) $i(t) = 1$, (2) $i(t) = i$, and (3) $i(t) \notin \{1, i\}$.

\paragraph{(1) When the optimal arm is played}
When $i(t)=1$ holds, for any $i \ne 1$
\begin{equation*}
    (t+1)f_{i}(\bw^{t+1}; \bmu) =  (N_1(t)+N_i(t)+1) h_i\left( \frac{N_i(t)}{N_1(t) + N_i(t)+1}; \bmu \right).
\end{equation*}
From the concavity of $h_i$, we obtain that
\begin{align*}
    h_i\left( \frac{N_i(t)}{N_1(t) + N_i(t)+1}; \bmu \right) &=  h_i\left( \frac{N_i(t)}{N_1(t) + N_i(t)} \frac{N_1(t) + N_i(t)}{N_1(t) + N_i(t)+1}  ; \bmu \right) \\
    &\geq \frac{N_1(t) + N_i(t)}{N_1(t) + N_i(t)+1} h_i\left( \frac{N_i(t)}{N_1(t) + N_i(t)}; \bmu \right) \\
    &\hspace{3em} + \frac{1}{N_1(t) + N_i(t)+1} h_i(0;\bmu),
\end{align*}
which implies
\begin{multline*}
    (N_1(t)+N_i(t)+1) h_i\left( \frac{N_i(t)}{N_1(t) + N_i(t)+1}; \bmu \right) \\\geq (N_1(t)+N_i(t)) h_i\left( \frac{N_i(t)}{N_1(t) + N_i(t)}; \bmu \right) = tf_i(\bw^t;\bmu).
\end{multline*}
This concludes the case when $i(t)=1$.

\paragraph{(2) When the suboptimal arm is played}
When $i(t) = i$ holds, 
\begin{equation*}
    (t+1)f_{i}(\bw^{t+1}; \bmu) =  (N_1(t)+N_i(t)+1) h_i\left( \frac{N_i(t)+1}{N_1(t) + N_i(t)+1}; \bmu \right).
\end{equation*}
By the concavity, again, we obtain that
\begin{align*}
    h_i\left( \frac{N_i(t)+1}{N_1(t) + N_i(t)+1}; \bmu \right) \hspace{-10em}& \\
    &=  h_i\left( \frac{N_i(t)}{N_1(t) + N_i(t)} \frac{N_1(t) + N_i(t)}{N_1(t) + N_i(t)+1} + \frac{1}{N_1(t) + N_i(t) + 1}  ; \bmu \right) \\
    &\geq \frac{N_1(t) + N_i(t)}{N_1(t) + N_i(t)+1} h_i\left( \frac{N_i(t)}{N_1(t) + N_i(t)}; \bmu \right) + \frac{1}{N_1(t) + N_i(t)+1} h_i(1;\bmu) \\
    &= \frac{N_1(t) + N_i(t)}{N_1(t) + N_i(t)+1} h_i\left( \frac{N_i(t)}{N_1(t) + N_i(t)}; \bmu \right),
\end{align*}
which concludes the case when $i(t)=i$.

\paragraph{(3) When the other suboptimal arms are played}
When $i(t) \notin \{ 1, i \}$, $N_1(t+1)=N_1(t)$ and $N_i(t+1)=N_i(t+1)$ holds.
Therefore, $(t+1)f_i(\bw^{t+1};\bmu) = tf_i(\bw^t; \bmu)$ holds, which concludes the case when $i(t) \ne 1, i$.
\end{proof}

\end{document}